\documentclass[11pt,oneside,a4paper]{article}
\usepackage[top=1 in, bottom=1 in, left=0.85 in, right=0.85 in]{geometry}
\usepackage[hidelinks]{hyperref}
\usepackage{multirow}
\usepackage[table,xcdraw]{xcolor}
\usepackage{amsmath,amssymb,graphicx,url}
\usepackage{algorithm, algorithmic}
\usepackage{ntheorem,thmtools,thm-restate,wrapfig,enumitem,mathabx}
\usepackage{float}
\usepackage{graphicx}
\usepackage{subcaption}
\usepackage{microtype}
\newtheorem{assumption}{Assumption}
\newtheorem{theorem}{Theorem}
\newtheorem{remark}{Remark}
\newtheorem{lemma}{Lemma}

\newtheorem{definition}{Definition}
\newtheorem*{proof}{Proof}

\DeclareMathOperator*{\argmin}{arg\,min}

\title{Online Nonstochastic Control with Adversarial and Static Constraints}

\author{Xin Liu \\ShanghaiTech University \\ liuxin7@shanghaitech.edu.cn
   \and Zixian Yang \\University of Michigan, Ann Arbor \\ zixian@umich.edu   
   \and Lei Ying \\University of Michigan, Ann Arbor \\ leiying@umich.edu}
\date{}

\begin{document}

\maketitle

\begin{abstract}
This paper studies online nonstochastic control problems with adversarial and static constraints. We propose online nonstochastic control algorithms that achieve both sublinear regret and sublinear adversarial constraint violation while keeping static constraint violation minimal against the optimal constrained linear control policy in hindsight. To establish the results, we introduce an online convex optimization with memory framework under adversarial and static constraints, which serves as a subroutine for the constrained online nonstochastic control algorithms. This subroutine also achieves the state-of-the-art regret and constraint violation bounds for constrained online convex optimization problems, which is of independent interest. Our experiments demonstrate the proposed control algorithms are adaptive to adversarial constraints and achieve smaller cumulative costs and violations. Moreover, our algorithms are less conservative and achieve significantly smaller cumulative costs than the state-of-the-art algorithm.
\end{abstract}

\section{Introduction}
Online nonstochastic control paradigm has been widely applied in practice \cite{HazSin_22,SuoGhaMin_21,OcoShiShi_22}. It has a topic of great interest in both the learning and control communities because online nonstochastic control algorithms are robust to time-varying and even adversarial environments \cite{AgaBulHaz_19,AgaHazSin_19,AloAviTom_18,DeaTuSte_19}. 
In an online nonstochastic control problem, the learner aims to learn a controller that minimizes the cumulative costs in a time-varying or even adversarial environment (e.g., the adversarial cost functions and disturbances).  
Practical control systems often operate under various constraints (e.g., safety, demand-supply, or energy  constraints), 
which could be unpredictable and adversarial as well.  
For example, robots need to navigate along a collision-free path by maintaining a minimum distance from each other with complicated surroundings (e.g., pedestrians or other robots) \cite{LukMelAda_22}; cloud computing platforms should guarantee low-latency service for users with time-varying traffic workload \cite{TirBarDen_20}. 
Motivated by these applications, we focus on online nonstochastic control problems with adversarial constraints and propose online nonstochastic control algorithms to achieve the minimum cost and the best constraint satisfaction. 

It is a challenging task to synthesize a safe (online) controller because of the conflicting objectives, i.e. minimizing the cumulative costs while satisfying all constraints. The traditional way to guarantee constraint satisfaction is to incorporate the constraints into Model Predictive Control (constrained MPC) \cite{MaySerRak_05,RawMayDie_17}. However, constrained MPC often introduces overly-conservative, even infeasible, actions in the presence of system disturbances. To address the issue, a sequence of works have relaxed or softened the constraints in MPC \cite{ZeiJonMor_10,WabKriZei_22,RakZhaSun_23}. For online nonstochatic control problems where  cost functions or disturbance could be arbitrary and time-varying, it is impossible to predict the model so the constrained MPC method is not applicable. 
Only a few works  \cite{NonMulMat_21, LiDasLi_21} have considered online non-stochastic control with constraints, and they only studied ``static'' affine constraints on the state $x_t$ and input $u_t$, i.e., $D_x x_t \leq 0$ and $D_u u_t \leq 0.$  
The work by \cite{NonMulMat_21} studied an online nonstochatic control problem with state and input constraints, but the system dynamics are noise/disturbance-free. 
The work most related to ours is \cite{LiDasLi_21}, which considers adversarial cost functions 
and adversarial disturbance but static affine constraints. The paper proposed a gradient descent-based control algorithm (called online gradient descent with buffer zones (OGD-BZ)) to achieve $\Tilde{O}(\sqrt{T})$ regret while the static affine constraints are satisfied. However, the work assumed the affine constraints and the knowledge of slackness of the constraints so that robust optimization methods can be used to construct safe/feasible regions for the affine constraints. The method also has the issue of being over-conservative as constrained MPC. Moreover, the method in \cite{LiDasLi_21} cannot be applied to the adversarial constraints because they are unknown to the controller before an action is taken. 

In this paper, we study an online nonstochastic control problem with adversarial constraints, where the cost and constraint functions are revealed after the control/action has been taken (our setting can also include static constraints).
Specifically, we consider a discrete-time, linear system as follows
$$x_{t+1} = A x_{t} + B u_t + w_t,$$
where $x_t$ is the state, $u_t$ is the control/action and $w_t$ is the adversarial noise/disturbance at time $t.$ The learner takes an action $u_t$ and observe the cost function $c_t(x_t, u_t)$ and constraint functions $d_t(x_t,u_t)$ and $l(x_t,u_t).$  The goal of the learner is to minimize the cumulative costs $\sum_{t=1}^T c_t(x_t, u_t)$ while satisfying the constraints, where the constraint violation will be measured by three different metrics: the soft cumulative constraint violation $\sum_{t=1}^T d_t(x_t, u_t),$ the hard cumulative violation $\sum_{t=1}^T d_t^{+}(x_t, u_t),$ and static anytime violation $l(x_t, u_t), \forall t.$ We propose a class of constrained online nonstochastic control algorithms (COCA) that guarantee sublinear regret and sublinear constraint violation against the {\em optimal linear} controllers for the constrained problems in hindsight, where the controller knows everything apriori. Our contributions are summarized below (in Table \ref{tab:main contribution}):
\begin{itemize}
    \item We propose COCA-Soft when adversarial constraints are measured using soft cumulative violation. The algorithm is based on the Lyapunov optimization method. COCA-Soft achieves $\Tilde{O}(\sqrt{T})$ regret, $O(1)$ cumulative soft violation for adversarial constraints, and $o(1)$ anytime violation for static constraints.  
    \item When considering hard cumulative violation as the metric, we propose COCA-Hard  based on proximal optimization methods. COCA-Hard achieves $\Tilde{O}(T^{2/3})$ regret, $\Tilde{O}(T^{2/3})$ hard cumulative violation for adversarial constraints, and $\Tilde{O}(T^{-1/3})$ anytime violation for static constraints. 
    \item When the cost functions are strongly-convex, we propose COCA-Best2Worlds that integrates proximal and Lyapunov optimization methods and provides performance guarantees in terms of both soft and hard violation metrics. COCA-Best2Worlds achieves $\Tilde{O}(\sqrt{T})$ regret, $O(1)$ and $\Tilde{O}(\sqrt{T})$ cumulative soft and hard violation for adversarial constraints, respectively, and $o(1)$ anytime violation for static constraints.  
\end{itemize}
To the best of our knowledge, all these results are new in the setting of online non-stochastic control with {\it adversarial} and {\it general static} constraints (not necessarily static affine constraints). 

\begin{table*}
\centering
\begin{tabular}{|c|c|c|c|c|}
\hline
Algorithms                            & Cost Function                                          & Regret                                                  & Soft/Hard Adversarial Vio. & Static Vio.  \\ \hline
OGD-BZ in \cite{LiDasLi_21} & Convex &$\Tilde O(\sqrt{T})$ & None/None & $\text{None}^*$ \\ \cline{1-1} \cline{2-3} \cline{3-4} \cline{4-5}
COCA-Soft & Convex &$\Tilde O(\sqrt{T})$ & $O(1)$/None & $o(1)$\\ \cline{1-1} \cline{2-3} \cline{3-4} \cline{4-5}  
 \cline{1-1} \cline{2-3} \cline{3-4} \cline{4-5} 
COCA-Hard &Convex &$\Tilde O(T^{2/3})$&$\Tilde O(T^{2/3})$/$\Tilde O(T^{2/3})$ & $\Tilde O(T^{-1/3})$
\\ 
 \cline{1-1} \cline{2-3} \cline{3-4} \cline{4-5}
COCA-Best2Worlds & Strongly Convex                                                       &  $\Tilde O(\sqrt{T})$  & $O(1)$/$\Tilde O(\sqrt T)$                   &  $o(1)$\\ \cline{1-1} \cline{2-3} \cline{3-4} \cline{4-5} 
\end{tabular}
\caption{Our contribution and related work. *\cite{LiDasLi_21} establishes  zero violation for static anytime affine constraints, which is a special case of static constraints studied in this paper.  Moreover, if we use projection-based method for static affine constraints, our algorithm can also achieve zero violation. 
}
\end{table*}\label{tab:main contribution}

\subsection{Related work}

{\bf Online Nonstochastic Control of Dynamic System: }
Online nonstochastic control leverages online learning or data-driven methods to design efficient and robust control algorithms in an  adversarial environment, where both cost functions and disturbances could be adversarial \cite{AgaBulHaz_19}. The main idea behind online nonstochastic control is to design a disturbance-action controller by carefully synthesizing the historical disturbance through the subroutine of online convex optimization (OCO) \cite{Haz_16}. The initial work \cite{AgaBulHaz_19} shows the disturbance-action controller can achieve $\Tilde{O}(\sqrt{T})$ regret w.r.t. the optimal linear controller in hindsight that knows all costs and disturbances beforehand. The results have been refined in \cite{AgaHazSin_19,FosSim_20} when the cost functions are strongly-convex and have been generalized to various settings \cite{SimSinHaz_20,MinGraSim_21}. 
The work \cite{SimSinHaz_20} studied online nonstochastic control problems for a block-box time-invariant system and it has been extended to the time-varying system in \cite{MinGraSim_21}. 
Further, a neural network has been used to parameterize the control policy in \cite{CheMinLee_22} and the regret performance is analyzed by combining OCO algorithm \cite{Haz_16} and neural-tangle kernel (NTK) theory \cite{JacGabHon_18}. 
However, these works do not consider any adversarial or static constraints. 

{\bf Online Learning with Constraints: }
Online learning with constraints has been widely studied in the literature \cite{MahJinYan_12,SunDeyKap_17, NeeYu_17, CaoZhaPoo_21, YiLiYan_21a, YiLiYan_21b, GuoLiuWei_22}. Existing results can be classified according to the types of constraints, e.g., static, stochastic, and adversarial constraints. We next only review the papers on adversarial constraints because they are the most related ones. 
The work \cite{SunDeyKap_17} studied OCO with adversarial constraints and established $O(\sqrt{T})$ regret and $O(T^{3/4})$ soft cumulative constraint violation. 
The work \cite{YiLiYan_21b, GuoLiuWei_22} considered the benchmark of hard cumulative violation and established $O(\sqrt{T})$ regret and $O(T^{3/4})$ hard violation. The performance has been further improved to be $O(\log{T})$ regret and $O(\sqrt{T\log T})$ hard violation when the objective is strongly-convex \cite{GuoLiuWei_22}. 

{\bf Safe Reinforcement Learning:}
Safe reinforcement learning (RL) refers to reinforcement learning with safety constraints and has received great interest as well
\cite{AniHumSha_13, FisAkaZei_19, JavFer_15, KolBerTur_18, WabZei_18, CheOroMur_19,TirBarDen_20,EfrManPir_20,DinZhaBas_20,DinWeiYan_21,LiuZhoKal_21,AmaThrYan_21,VasYanSze_22,CheJaiLuo_22,GhoZhoShr_22,WeiLiuYin_22}. In safe RL, The agent optimizes the policy by interacting with the environment without violating safety constraints. However, the line of safe RL requires either the knowledge of the initial safe policy or a stationary environment where the reward and cost distributions are time-invariant. 

\section{Online Nonstochastic Control with Constraints}
In this section, we introduce the online nonstochastic control problem with constraints and the performance metrics for evaluating the cost and constraint satisfaction. We consider the following linear  system:
$$x_{t+1} = A x_{t} + B u_t + w_t,$$
where $x_t \in \mathbb R^n$ is the state, $u_t \in \mathbb R^m$ is the control/action and $w_t \in \mathbb R^n$ is the noise or disturbance at time $t.$ Note $w_t$ could be even adversarial. The system parameters  $A \in \mathbb R^{n\times n}$ and $B \in \mathbb R^{n\times m}$ are assumed to be known.  A constrained online nonstochastic control system works as follows:
given the state $x_t, \forall t\in [T],$ the learner takes action $u_t$ and observes the cost function $c_t(x_t, u_t)$ and constraint functions $d_t(x_t,u_t)$ and $l(x_t,u_t).$ The system evolves to the next state $x_{t+1}$ according to the system equation. Our objective is to design an optimal control policy to minimize $\sum_{t=1}^T c_t(x_t, u_t)$ while satisfying the constraints. 
Next, we introduce our baseline control policy and define the performance metrics of regret and constraint satisfaction. 

{\noindent \bf Offline Control Problem:} Assuming the full knowledge of disturbance, cost functions, and constraint functions beforehand, the offline control problem is defined to be: 
\begin{align*}
    \min_{\{u_t\}}& ~ \sum_{t=1}^T c_t(x_t, u_t) \\
    \text{s.t.}& ~ x_{t+1} = Ax_t + Bu_t + w_t, \forall t \in [T], \\
    &~d_t(x_t, u_t) \leq 0, ~l(x_t, u_t) \leq 0, \forall t \in [T].
\end{align*}
We define $K^* \in \mathbb R^{m\times n}$ to be the optimal linear control $u_t^{K^*} = -K^* x_t^{K^*}$ which satisfies the constraints, i.e., $K^* \in \Omega$ such that $$\Omega =\left\{ \pi ~|~ d_t(x_t^{\pi}, u_t^{\pi}) \leq 0, l(x_t^{\pi}, u_t^{\pi}) \leq 0, \forall t \in [T]\right\}.$$

{\noindent \bf Regret:} Given the optimal linear policy $K^*$ as the baseline, the goal of the leaner is to design an online nonstochastic control policy $\pi$ that minimizes the following regret
\begin{align}
\mathcal R(T) = \sum_{t=1}^T c_t(x_t^\pi, u_t^\pi) - \sum_{t=1}^T c_t(x_t^{K^*}, u_t^{K^*})\nonumber.
\end{align}
{\noindent \bf Constraint Violation:} The control algorithm needs to obey the constraints. However, since the constraints are unknown and adversarial, some violation has to occur during learning and control. To evaluate the level of constraint satisfaction, we consider two different metrics for adversarial constraints: soft violation and hard violation:  
\begin{align*}
\mathcal V_d^{soft}(T) = \sum_{t=1}^T d_t(x_t^\pi, u_t^\pi), ~~
\mathcal V_d^{hard}(T) = \sum_{t=1}^T d_t^{+}(x_t^\pi, u_t^\pi),
\end{align*}
and the anytime violation for the static constraint $l$ is $$\mathcal V_l(t) = l(x_t^\pi,u_t^\pi).$$
Note soft and hard violation metrics for adversarial constraints are for different applications. For example, in cloud computing, the latency constraint is soft and the soft violation is a natural metric; however in drone control, the power constraint is hard and the hard violation is a better metric. We consider anytime violation for static constraint function $l$ because it is related to state and input constraints that needs to be satisfied anytime, resembling stability requirements. 
To present our algorithm, we first present several key concepts of online nonstochastic control from \cite{AgaBulHaz_19}.  

\subsection{Preliminary on Constrained Online Nonstochastic Control}
\begin{definition}[Strong Stability]
A linear controller $K \in \mathbb R^{m\times n}$ is $(\kappa, \rho)$-strongly stable if there exists matrices $A,B,U,L$ such that $$\tilde{A} := A-BK = ULU^{-1}$$ with $\max(\|U\|_2, \|H^{-1}\|_2, \|K\|_2) \leq \kappa$ and $\|L\|_2 \leq 1-\rho, \rho \in (0, 1].$ 
\end{definition}
Given the knowledge of system dynamics $A$ and $B,$ the $(\kappa, \rho)$-strongly stable controller $K$ can be computed with semi-definite programming (SDP) \cite{AloAviTom_18}. Note the stable controller $K$  might not satisfy the constraints in $\Omega$, i.e., $K\notin \Omega.$ 

\begin{definition}[Disturbance-Action Policy Class (DAC)]
A disturbance-action policy $\pi(K, \{\mathbf M_t\})$ with memory size $H$ is defined as follows $$u_t = -K x_t + \sum_{i=1}^H {\mathbf M}^{[i]}_t w_{t-i},$$
where $\mathbf M_t \in \mathcal M$ and $\mathcal M = \{{\mathbf M}^{[1]}_t, \cdots, {\mathbf M}^{[H]}_t ~|~ \|{\mathbf M}^{[i]}_t\| \leq a (1-\rho)^i, {\mathbf M}^{[i]}_t \in \mathbb R^{m\times n}, a>0, \forall i \in [H]\}.$ \label{def:dac}
\end{definition}
The disturbance-action policy consists of a linear combination of the disturbance (the memory size is $H=\Theta(\log T)$ in the paper). Given a stable controller $K$ and by carefully choosing $\{\mathbf M_t\},$ $\pi(K, \{\mathbf M_t\})$ aims to approximate a good linear stable controller that achieves small costs and satisfies the constraints in $\Omega.$ 
Further, we define DAC with fixed weights, which serves as an intermediate policy class and is frequently used in our analysis. 
\begin{definition}[DAC with Fixed Weight]
For a DAC $\pi(K, \{\mathbf M_t\}),$ the set of fixed weight DAC policies is 
\begin{align}
    \mathcal E = \{\pi(K, \{\mathbf M_t\}) ~|~ \mathbf M_t = \mathbf M, \forall t \in [T]\}. 
\end{align}
\end{definition}\label{def:approx cost and cons}
Let $\mathbf M_{s:t} := \{\mathbf M_{s}, \cdots, \mathbf M_{t}\}$ and $\tilde{A} = A - BK.$ Under a policy $\pi(K, \{\mathbf M_t\})$ in DAC,  $\Psi_{t,i}^\pi(\mathbf M_{t-H:t-1})$ is defined to be the disturbance-state transfer matrix:
\begin{align*}
&\Psi_{t,i}^\pi(\mathbf M_{t-H:t-1}) \\
=& \tilde{A}^{i-1} \mathbb I(i\leq H) + \sum_{j=1}^H \tilde{A}^{j-1} B \mathbf M^{[i-j]}_{t-j} \mathbb I_{i-j \in [1,H]}.
\end{align*}
We occasionally abbreviate $\Psi_{t,i}^\pi(\mathbf M_{t-H:t-1})$ to be $\Psi_{t,i}^\pi$ without causing any confusion.   
As shown in \cite{AgaBulHaz_19}, under a policy $\pi$ in DAC, the state is represented by
$x_{t}^\pi = \sum_{i=1}^t  \Psi_{t,i}^\pi w_{t-i},$    
which is equivalent to
\begin{align}
    x_{t}^\pi = \tilde{A}^H x_{t-H} + \sum_{i=1}^{2H}  \Psi_{t,i}^\pi w_{t-i}. \nonumber
\end{align}
By truncating the true states, we define the approximated states and actions
\begin{align}
   \tilde x_t^\pi = \sum_{i=1}^{2H}  \Psi_{t,i}^\pi w_{t-i}, ~~ 
   \tilde u_t^\pi =  -K \tilde x_t^\pi + \sum_{i=1}^{H}  {\mathbf M}^{[i]}_t w_{t-i}. \label{def: approx stateaction}
\end{align} 
Further, we have the approximated cost and constraint functions in the following
\begin{align}
    c_t(\tilde x_t^\pi, \tilde u_t^\pi),~~ d_t(\tilde x_t^\pi, \tilde u_t^\pi),~~  l(\tilde x_t^\pi, \tilde u_t^\pi). \label{def:approx cost and cons}
\end{align}
Based on the definition of approximated constraint functions, we define an approximated constraint set
$\tilde{\Omega}$ such that
$$\tilde{\Omega} = \{ \pi ~|~ d_t(\tilde x_t^\pi, \tilde u_t^\pi) \leq 0, l(\tilde x_t^\pi, \tilde u_t^\pi) \leq 0, \forall t \in [T]\}.$$ Note the states in both $\Omega$ and $\tilde{\Omega}$ are driven by the same underlying dynamics with the policy $\pi$. Intuitively, $\Omega$ and $\tilde{\Omega}$ are ``close'' if the approximated errors of states and actions are small.  
We introduce the following assumptions on cost and constraint functions. 
\begin{assumption}
The cost $c_t(x, u)$ and constraint functions $d_t(x, u)$ and $l(x, u)$ are convex and differentiable. Let $C_0$ and $C_1$ be positive constants. As long as $\|x-x'\| \leq D$ and $\|u-u'\| \leq D,$ we assume the gradients $\|\nabla_x c_t\|, \|\nabla_u c_t\|, \|\nabla_x d_t\|,  \|\nabla_u d_t\|$, $\|\nabla_x l\|, \|\nabla_u l\|$ are bounded by $C_0D;$ we assume the functions $c_t, d_t,$ and $l$ are bounded by $C_1 D.$
\label{assumption: fun bound}
\end{assumption}
Further, we introduce an assumption on the feasibility of the offline control problem. This assumption can be regarded as Slater's condition in the optimization literature. Note this assumption is necessary for establishing the cumulative soft violation in COCA-Soft and COCA-Best2Worlds, and COCA-Hard does not need this assumption.    
\begin{assumption} Let $\delta$ be a positive constant, there exists a policy $\pi \in \mathcal E$ such that 
\begin{align*}
    d_t(x_t^\pi, u_t^\pi) \leq -\delta, ~~
    l(x_t^\pi, u_t^\pi) \leq -\delta, ~\forall t\in [T].
\end{align*} \label{assumption: slater}
\end{assumption} 

\section{Constrained Online Nonstochastic Control Algorithm}\label{sec:coca}
\label{alg}
Given an (arbitrary) stable control policy $K,$ we develop a set of online nonstochastic control policy $\pi(K,\{\mathbf M_t\})$ to adjust the weights of disturbance/noise $\{\mathbf M_t\}$ such that it achieves small regret and constraint violation.
Specifically, we use constrained online learning algorithms as the subroutines of our online nonstochastic control policy $\pi(K,\{\mathbf M_t\})$ to optimize the weights $\{\mathbf M_t\}.$

According to the definition in \eqref{def: approx stateaction}, the approximated state $\tilde x_t^\pi$ and action $\tilde u_t^\pi$ are only related to the weights of the past $H$ steps $\mathbf M_{t-H:t}:=\{\mathbf M_{t-H}, \cdots, \mathbf M_t\}.$ Therefore, we denote 
$
    \tilde c_t(\mathbf M_{t-H:t}):= c_t(\tilde x_t^\pi, \tilde u_t^\pi), ~
    \tilde d_t(\mathbf M_{t-H:t}):= d_t(\tilde x_t^\pi, \tilde u_t^\pi), ~
    \tilde l(\mathbf M_{t-H:t}):= l(\tilde x_t^\pi, \tilde u_t^\pi). 
$
To simplify notation, we further define $\tilde c_t(\mathbf M):= \tilde c_t(\mathbf M, \cdots, \mathbf M)$ and similarly for $\tilde d_t(\mathbf M)$ and $\tilde l(\mathbf M).$ We are ready to present our constrained online nonstochastic control algorithm (COCA) by using the constrained online convex optimization solver (COCO-Solver) as the subroutine.
\begin{algorithm}[H]
\renewcommand{\thealgorithm}{}
\floatname{algorithm}{}
\caption{{\bf Constrained Online Nonstochastic Control Algorithm}}
\begin{algorithmic}
   \STATE {\bfseries Initialize:} a $(\kappa,\rho)$ stable controller $K$ 
   and the proper learning rates in COCO-Solver.
   \FOR{$t=1,\cdots, T,$}
   \STATE {\bf Observe} state $x_t$ and compute the disturbance $w_{t-1}$.
\STATE {\bf Apply} control $u_{t} = - Kx_t +\sum_{i=1}^H \mathbf M^{[i]}_t w_{t-i}.$
\STATE {\bf Receive} feedback including cost function $c_{t}(x_t,u_t)$ and constraint functions $d_{t}(x_t,u_t)$ and  $l(x_t,u_t).$
\STATE {\bf Compute} the approximated cost function $\tilde c_{t}(\cdot)$ and constraint functions $\tilde d_{t}(\cdot)$ and $\tilde l(\cdot)$.  
\STATE {\bf Invoke} the {\bf COCO-Solver}$(\mathbf M_t,$ $Q_t,$ $\tilde c_t(\cdot),$ $\tilde d_t(\cdot),$ $\tilde l(\cdot))$ to obtain $\mathbf M_{t+1}$ and $Q_{t+1}.$ 
   \ENDFOR
\end{algorithmic}
\end{algorithm}
For COCA at time $t$, we observe the state $x_t$ and infer the previous disturbance $w_{t-1} = x_{t} - A x_{t-1} - Bu_{t-1}.$ The information of state and the past disturbances are used in $\pi(K,\{\mathbf M_t\})$ to output a control/action $u_t.$ Then we observe the full information of the cost function $c_t(\cdot, \cdot)$ and constraint functions $d_t(\cdot, \cdot)$ and $l_t(\cdot, \cdot).$ Based on the feedback, we compute $\tilde c_t(\cdot),$ $\tilde d_t(\cdot),$ and $\tilde l(\cdot),$ and invoke COCO-Solver to optimize the weights of disturbance  $\{\mathbf M_{t}\}$ for the next control period $t+1.$ Note COCO-Solver has an input variable of $Q_t,$ which is designed to track the soft cumulative violation of $d_t(\cdot,\cdot)$ and is also a feedback signal to control the trade-off between the cost and soft constraint satisfaction of $d_t(\cdot,\cdot).$

As discussed, COCO-Solver is the key to optimizing the cumulative costs while minimizing (soft or hard) constraint violations. Depending on the types of constraint violation metrics we want to optimize, COCO-Solver will be instantiated with the COCO-Soft or COCO-Hard  solvers. Moreover, when the cost functions are strongly-convex, we design COCO-Best2Worlds solver that can optimize soft and hard cumulative violations simultaneously. It is worth mentioning that COCA with dedicated solvers is computationally efficient and less conservative compared to the existing robust optimization-based control approaches \cite{DeaTuSte_19,LiDasLi_21}. Next, we introduce these solvers and their corresponding theoretical performance, respectively.

\subsection{COCA with COCO-Soft Solver}
We instantiate COCO-Solver in COCA with the algorithm {\bf COCO-Soft}$(\mathbf M_t,$ $Q_t,$ $\tilde c_t(\cdot),$ $\tilde d_t(\cdot),$ $\tilde l(\cdot))$. The main idea behind COCO-soft is to carefully design a control surrogate function based on the Lyapunov optimization method such that the cumulative cost and soft violation are balanced. Specifically, for the loss function, we use $\tilde c_{t}(\mathbf M_t) + \langle \mathbf M-\mathbf M_t, \nabla \tilde c_t(\mathbf M_t) \rangle$ to approximate $\tilde c_{t+1}(\mathbf M).$ For the adversarial constraint, we use $\tilde d_{t}(\mathbf M_t) + \langle \mathbf M-\mathbf M_t, \nabla \tilde d_t(\mathbf M_t) \rangle$ to approximate $\tilde d_{t+1}(\mathbf M)$ and the virtual queue $Q_t$ 
indicates the degree of the constraint violation, i.e., a large/small $Q_t$ means a large/small violation of the adversarial constraints.
The product term $Q_t \langle \mathbf M-\mathbf M_t, \nabla \tilde d_t(\mathbf M_t) \rangle$ in the control surrogate function is a proxy term of $Q_t \tilde d_{t+1}(\mathbf M).$ Combining with the virtual queue update, minimizing the product term is equivalent to minimize the Lyapunov drift of $Q_{t+1}^2(\mathbf M) - Q^2_{t}.$ For the static constraint function $\tilde l(\mathbf M),$ we directly impose the penalty factor $\eta$ to prevent the violation.   
In summary, the control surrogate function carefully integrates the approximated cost and constraints in $V \tilde c_{t+1}(\mathbf M) + Q_{t+1}^2(\mathbf M) + \eta \tilde l^{+}(\mathbf M),$ so optimizing the surrogate function guarantees the best trade-off between the cumulative costs and constraint violations. 
\begin{algorithm}
\renewcommand{\thealgorithm}{}
\floatname{algorithm}{}
\caption{\bf{COCO-Soft}$(\mathbf M_t,$ $Q_t,$ $\tilde c_t(\cdot),$ $\tilde d_t(\cdot),$ $\tilde l(\cdot))$}
   \label{alg:example}
\begin{algorithmic}
\STATE {\bf Control Decision:} Choose $\mathbf M_{t+1}\in \mathcal M$ to minimize the control surrogate function
\begin{align}
&V \langle \mathbf M-\mathbf M_t, \nabla \tilde c_t(\mathbf M_t) \rangle + Q_t \langle \mathbf M-\mathbf M_t, \nabla \tilde d_t(\mathbf M_t) \rangle  + \eta \tilde l^{+}(\mathbf M) + \alpha \|\mathbf M - \mathbf M_{t}\|^2. \nonumber
\end{align}
\STATE {\bf Virtual Queue Update:} 
\begin{align}
  Q_{t+1} 
  = \left[Q_t + \tilde d_t(\mathbf M_t) + \langle \mathbf M_{t+1} - \mathbf M_{t}, \nabla \tilde d_t(\mathbf M_t) \rangle  + \epsilon \right]^{+} \nonumber
 \end{align}
\noindent {\bf Output:} $\mathbf M_{t+1}$ and $Q_{t+1}.$
\end{algorithmic}
\end{algorithm}\label{coco-soft}
Next, we present the theoretical results for COCA with the COCO-Soft solver. We only present order-wise results. The exact constants and detailed proof can be found in Appendix \ref{app:soft}.
\begin{theorem}
Given a stable linear controller $K,$ under Assumptions \ref{assumption: fun bound} and \ref{assumption: slater}, COCA with COCO-Soft solver achieves 
\begin{align*}
    \mathcal R(T) =& O\left(\sqrt{T}\log^3 T\right),\\
    \mathcal V_d^{soft}(T) = O\left(1\right),& ~\mathcal V_l(t) = O(1/\log T), \forall t \in [T],
\end{align*}
for a large $T$ with probability at least $1-1/T.$\label{thm:occa}
\end{theorem}
\begin{remark}
Theorem \ref{thm:occa} implies COCA with COCO-Soft achieves similar performance as the optimal offline linear controller $K^*$ when $T$ is large.  COCO-Soft  only needs to solve an almost unconstrained optimization problem, which is more computationally efficient than \cite{LiDasLi_21} that requires constructing a feasible region via robust optimization and use a projection-based method to guarantee the constraints in each control period. Moreover, if we project $\mathbf M_{t+1}$ into the set of $\{\mathbf M~|~ \tilde{l}(\mathbf M) \leq 0\}$ instead of the penalty-based design $\eta \tilde{l}^{+}(\mathbf M),$ we can also achieve zero anytime violation as in \cite{LiDasLi_21}, which is verified in Appendix \ref{app: good anytime}. Finally, we would like to mention that Lyapunov optimization with the pessimistic design in virtual queue allows COCO-Soft to achieve the best trade-off between regret and violation (please refer to Theorem \ref{thm:ocowm} and Remark \ref{rmk:bestcoco} for details). \label{remark: good anytime}
\end{remark}

\subsection{COCA with COCO-Hard Solver}
We instantiate COCO-Solver in COCA with the algorithm {\bf COCO-Hard}$(\mathbf M_t,$ $Q_t,$ $\tilde c_t(\cdot),$ $\tilde d_t(\cdot),$ $\tilde l(\cdot))$. The main idea behind COCO-hard is to capture the constraint directly in the control surrogate function with the proximal penalty-based method, which is different from the Lyapunov optimization method in COCO-Soft. Since the design for $\tilde c_t(\mathbf M)$ and $\tilde l(\mathbf M)$ is similar to that in COCO-Soft. We focus on the new design for taking care of the adversarial constraint. 
Specifically, 
we directly use $\tilde d_{t}^{+}(\mathbf M)$ as a proxy term of $\tilde d_{t+1}^{+}(\mathbf M)$ and impose a penalty factor $\gamma$ to prevent the violation.  
Therefore, the control surrogate function approximates $V \tilde c_{t+1}(\mathbf M) + \gamma \tilde d_{t+1}^{+}(\mathbf M) + \eta \tilde l^{+}(\mathbf M)$, which directly captures the cumulative costs and (hard) constraint violation. 
\begin{algorithm}
\renewcommand{\thealgorithm}{}
\floatname{algorithm}{}
\caption{\bf{COCO-Hard}$(\mathbf M_t,$ $\tilde c_t(\cdot),$ $\tilde d_t(\cdot),$ $\tilde l(\cdot))$}
   \label{alg:example}
\begin{algorithmic}
\STATE {\bf Control Decision:} Choose $\mathbf M_{t+1} \in \mathcal M$ to minimize the control surrogate function
\begin{align}
& V \langle \mathbf M-\mathbf M_t, \nabla \tilde c_t(\mathbf M_t) \rangle + \gamma \tilde d_t^{+}(\mathbf M) + \eta \tilde l^{+}(\mathbf M) + \alpha \|\mathbf M - \mathbf M_{t}\|^2 \nonumber
\end{align}
\STATE {\bf Output:} $\mathbf M_{t+1}.$ 
\end{algorithmic}
\end{algorithm}

Next, we present the theoretical results for COCA with the COCO-Hard solver. The detailed parameters and proof can be found in Appendix \ref{app:hard}.
\begin{theorem}
Given a stable linear controller $K,$ under Assumptions \ref{assumption: fun bound}, COCA with COCO-Hard solver achieves \begin{align*}
    \mathcal R(T) =& O(T^{\frac{2}{3}}\log^2 T),  ~~\mathcal V_l(t) = O(\log T/T^{\frac{1}{3}}), \forall t \in [T],\\
    &\mathcal V_d^{hard}(T) = O(T^{\frac{2}{3}}\log^2 T),
\end{align*}
for a large $T.$
\label{thm:occa-hard}
\end{theorem}
\begin{remark}
 COCO-Hard establishes Theorem \ref{thm:occa-hard} without a ``Slater-like'' Assumption \ref{assumption: slater}. Similar as in COCO-Soft, COCO-Hard is computationally efficient and avoids the complex projection operator. Moreover, by tuning learning rates $V, \gamma, \eta,$ and $\alpha$ in COCO-Hard, we are able to establish a trade-off $\mathcal R(T) = \Tilde{O}(T^{\max\{1-\frac{c}{2},c\}}),$ $\mathcal V_d^{hard}(T) = \Tilde{O}(T^{\max\{1-\frac{c}{2},0.5\}})$, and $\mathcal V_l(t) = \Tilde{O}(T^{-\frac{c}{2}}), \forall t \in [T],$ with $c\in [0.5,1)$ (please refer to Appendix \ref{app: lr tuning} for details). \label{remark: lr tuning}  
\end{remark}

\subsection{COCA with COCO-Best2Worlds Solver}
In this section, we show that when the cost functions are strongly-convex and the disturbances satisfy a mild condition, we are able to combine the Lyapunov optimization method and the proximal penalty-based method to guarantee small soft and hard violations simultaneously for adversarial constraints, which is called COCO-Best2Worlds. Specifically, in the control surrogate function, we minimize $Q_t \langle \mathbf M-\mathbf M_t, \nabla \tilde d_t(\mathbf M_t)\rangle$ and $\gamma \tilde d_t^{+}(\mathbf M)$ for soft and hard violation of adversarial constraints simultaneously. 
\begin{algorithm}
\renewcommand{\thealgorithm}{}
\floatname{algorithm}{}
\caption{\bf{COCO-Best2Worlds}$(\mathbf M_t,$ $Q_t,$ $\tilde c_t(\cdot),$ $\tilde d_t(\cdot),$ $\tilde l(\cdot))$}
\begin{algorithmic}
\STATE {\bf Control:} Choose $\mathbf M_{t+1} \in \mathcal M$ to minimize the control surrogate function
\begin{align}
&V_t \langle \mathbf M-\mathbf M_t, \nabla \tilde c_t(\mathbf M_t) \rangle + Q_t \langle \mathbf M-\mathbf M_t, \nabla \tilde d_t(\mathbf M_t) \rangle \nonumber\\
&~~+ \gamma_t \tilde d_t^{+}(\mathbf M) + \eta_t \tilde l^{+}(\mathbf M) + \alpha_t \|\mathbf M - \mathbf M_{t}\|^2 \nonumber
\end{align}
\STATE {\bf Virtual Queue Update:}
\begin{align}
  Q_{t+1} = \left[Q_t + \tilde d_t(\mathbf M_t) + \langle \mathbf M_{t+1} - \mathbf M_{t}, \nabla \tilde d_t(\mathbf M_t) \rangle  + \epsilon \right]^{+}\nonumber
 \end{align}
\end{algorithmic}
\end{algorithm}

We assume the cost functions are strongly-convex and disturbances satisfy a mild condition as follows. 
\begin{assumption}
The cost function $c_t(x, u),$ is $\alpha$-strongly convex, $\forall t\in[T].$ 
The disturbances introduced per time step are bounded, i.i.d, and zero-mean with a lower bounded covariance i.e., $\mathbb{E}[w_t w_t^{\dag}] \succeq I, \forall t\in[T].$ 
\label{assumption: strongly convex}
\end{assumption}

We are ready to present the theoretical results for COCA with the COCO-Best2Worlds solver. The detailed proof can be found in Appendix \ref{app:best}.
\begin{theorem}
Given a stable linear controller $K,$ under Assumptions \ref{assumption: fun bound}, \ref{assumption: slater}, and \ref{assumption: strongly convex}, COCA with COCO-Best2Worlds solver achieves
\begin{align*}
    \mathcal R(T) = O(\sqrt{T}\log^3 T),&  ~~\mathcal V_l(t) = O(1/\log T), \forall t \in [T],\\
    \mathcal V_d^{soft}(T) = O\left(1\right),& ~~ \mathcal V_d^{hard}(T) = O(\sqrt{T}\log^3 T),
\end{align*}
for a large $T$ with the probability at least $1-1/T.$ \label{thm:occa-best}
\end{theorem}

\subsection{A Roadmap to Prove Theorems \ref{thm:occa}, \ref{thm:occa-hard}, and \ref{thm:occa-best}} \label{sec:template}
We provide a general roadmap to prove the regret and constraint violation in Theorems \ref{thm:occa}, \ref{thm:occa-hard}, and \ref{thm:occa-best}. 

{\noindent \bf Regret analysis:} we have the following decomposition for the regret
\begin{align}
\mathcal R&(T) = \sum_{t=1}^T c_t(x_t^{\pi}, u_t^{\pi}) - \sum_{t=1}^T c_t(x_t^{K^*}, u_t^{K^*}) \nonumber\\
=& \sum_{t=1}^T \left[c_t(x_t^{\pi}, u_t^{\pi}) -  c_t(\tilde x_t^{\pi}, \tilde u_t^{\pi})\right] \label{term:A}\\
&+ \sum_{t=1}^T  c_t(\tilde x_t^{\pi}, \tilde u_t^{\pi}) - \min_{\pi \in \tilde{\Omega} \bigcap
 \mathcal E}\sum_{t=1}^T c_t(\tilde x_t^{\pi}, \tilde u_t^{\pi}) \label{term:B}\\
&+\min_{\pi \in \tilde{\Omega} \bigcap
 \mathcal E}\sum_{t=1}^T c_t(\tilde x_t^{\pi}, \tilde u_t^{\pi}) - \sum_{t=1}^T c_t(x^{K^*}_t, u^{K^*}_t). \label{term:C}
\end{align}
The term in \eqref{term:A} is on the approximation error of cost functions, related to the approximated errors of states and actions, and is bounded by $O(1)$ in Lemma \ref{lem:fun-diff} when choosing the memory size $H=\Theta(\log T)$ for a disturbance-action policy.
The term in \eqref{term:C} is on the representation ability of a disturbance-action policy with constraints, which can also be bounded by $O(1)$ in Lemma \ref{lem:dac-rep} because $K^*$ intuitively belongs to the class $\tilde{\Omega} \bigcap
 \mathcal E.$ 
The term in \eqref{term:B} is the key to the regret of COCA, which depends on the regret of COCO-Solver and will be established in Theorems \ref{thm:ocowm} and \ref{thm:ocowm-hard} in the next section, respectively. 

{\noindent \bf Cumulative soft/hard violation of $d_t$ function:} we have the following decomposition for the soft/hard violation of adversarial $d_t$:
\begin{align*}
\mathcal V_d^{soft}(T)
=& \sum_{t=1}^T [d_t(x_t^\pi, u_t^\pi) - d_t(\tilde x_t^\pi, \tilde u_t^\pi)] + \sum_{t=1}^T d_t(\tilde x_t^\pi, \tilde u_t^\pi) \\
\mathcal V_d^{hard}(T)
\leq& \sum_{t=1}^T [d_t(x_t^\pi, u_t^\pi) - d_t(\tilde x_t^\pi, \tilde u_t^\pi)]^{+} + \sum_{t=1}^T d_t^{+}(\tilde x_t^\pi, \tilde u_t^\pi) \nonumber
\end{align*}
The difference terms in $\mathcal V_d^{soft}(T)$ and $\mathcal V_d^{hard}(T)$ are on the approximation error of constraint functions, which are also related to the approximated errors of states and actions and are 
bounded by $O(1)$ in Lemma \ref{lem:fun-diff}; 
the terms $\sum_{t=1}^T d_t(\tilde x_t^\pi, \tilde u_t^\pi)$ or $\sum_{t=1}^T d_t^{+}(\tilde x_t^\pi, \tilde u_t^\pi)$ are on the soft or hard constraint violation of COCO-Solver, which are established in the next section.

{\noindent \bf Anytime violation of $l$ function:}
we have the following decomposition for the constraint $l$ function:
\begin{align*}
\mathcal V_l(t)
= l(x_t^\pi, u_t^\pi) - l(\tilde x_t^\pi, \tilde u_t^\pi) + l(\tilde x_t^\pi, \tilde u_t^\pi).
\end{align*}
Similarly, the difference term is on the anytime approximated error of constraint functions, which is bounded by $O(1/T)$ in Lemma \ref{lem:fun-diff}; the term of $l(\tilde x_t^\pi, \tilde u_t^\pi)$ depends on the anytime violation of COCO-Solver, which is established in the next section.

As discussed above, the key is to analyze the performance of $c_t(\tilde x_t^\pi, \tilde u_t^\pi), d_t(\tilde x_t^\pi, \tilde u_t^\pi)$ and $l(\tilde x_t^\pi, \tilde u_t^\pi)$ with COCO-Solver, where the approximated states $\tilde x_t^\pi$ and actions $\tilde u_t^\pi$ depend on the past states and actions up to the previous $H$ steps, i.e., $\mathbf M_{t-H:t}.$ COCO-Solver is naturally implemented by a constrained online convex optimization with memory framework (COCOwM). 
Since COCOwM is a plug-in component of COCA, we present the analysis in a separate section and any advance in COCOwM can be directly translated to that in COCA.    

\section{COCO-Solver via Constrained Online Convex Optimization \\with Memory}
In the standard constrained online convex optimization (COCO), the loss and constraint functions at time $t$ only depend on the current decision $\mathbf M_t \in \mathcal M.$   
In the constrained online convex optimization with memory (COCOwM), the loss function $f_t(\mathbf M_{t-H:t})$ and cost functions $g_t(\mathbf M_{t-H:t})$ and $h(\mathbf M_{t-H:t})$ at time $t$ depends on the historical decisions of $\{\mathbf M_{t-H:t}\}$ up to the previous $H$-steps. If we associate $f_t(\mathbf M_{t-H:t}),$ $g_t(\mathbf M_{t-H:t}),$ and $h(\mathbf M_{t-H:t})$ with $c_t(\tilde x_t^\pi, \tilde u_t^\pi), d_t(\tilde x_t^\pi, \tilde u_t^\pi),$ and $l(\tilde x_t^\pi, \tilde u_t^\pi),$ respectively, COCOwM is naturally used to optimize $\{\mathbf M_t\}$ and the performance of COCOwM (or COCO-Solver) can be translated to that of COCA. Similar to COCA, we define the metrics of regret and constraint
violation for COCOwM.  

{\noindent \bf Offline COCOwM:} Recall for a simple notation, we define $f_t(\mathbf M) = f_t(\mathbf M,\cdots,\mathbf M)$ and similarly for $g_t(\mathbf M)$ and $h(\mathbf M).$ We formulate the offline COCOwM as follows: 
\begin{align}
    \min_{\mathbf M \in \mathcal M}& ~ \sum_{t=1}^T f_t(\mathbf M) \label{eq:  coco-m baseline obj} \\
    \text{s.t.} 
    & ~ h(\mathbf M) \leq 0,~ g_t(\mathbf M) \leq 0, \forall t \in [T]. 
    \label{eq:  coco-m baseline cons}
\end{align}
Let the optimal solution to \eqref{eq:  coco-m baseline obj}-\eqref{eq:  coco-m baseline cons} be $\mathbf M^{*}.$ 
We define the regret and constraint violations of COCOwM
\begin{align}
    \mathcal R_f(T) = \sum_{t=1}^T f_t(\mathbf M_{t-H:t}&) -   \sum_{t=1}^T f_t(\mathbf M^*), \nonumber\\
    \mathcal V_g^{soft}(T) = \sum_{t=1}^T g_t(\mathbf M_{t-H:t}),&~  \mathcal V_g^{hard}(T) = \sum_{t=1}^T g_t^{+}(\mathbf M_{t-H:t}) \nonumber\\
    \mathcal V_h(t) = h(\mathbf M_{t-H:t}&),~ \forall t \in [T]. \nonumber
\end{align}

Before presenting the formal analysis of COCOwM (or COCO-Solver) algorithms, we introduce several necessary assumptions.
\begin{assumption}
The feasible set $\mathcal M$ is convex with diameter $D$ such that $\|\mathbf M - \mathbf M'\| \leq D, \forall \mathbf M, \mathbf M' \in \mathcal M$. \label{assumption:set}
\end{assumption}
\vspace{-20pt}
\begin{assumption}
The loss and constraint functions are convex and Lipschitz continuous with Lipschitz constant $L.$ 
Further, assume $h(\mathbf M) \leq E$ and $g_t(\mathbf M)\leq E, \forall t \in [T].$
\label{assumption:obj}
\end{assumption}
\begin{assumption}
There exists a positive constant $\xi>0$ and $\mathbf M\in \mathcal M$ such that $h(\mathbf M) \leq -\xi$ and $g_t(\mathbf M) \leq -\xi, \forall t\in[T].$  \label{assumption:slater}
\end{assumption}
We are ready to present the theoretical results of COCO-Soft, COCO-Hard, and COCO-Best2Worlds solvers introduced in Section \ref{sec:coca}.  
\subsection{Theoretical Analysis of COCO-Soft}

{\bf COCO-Soft}$(\mathbf M_t,$ $Q_t,$ $f_t(\cdot),$ $g_t(\cdot),$ $h(\cdot))$  optimizes $f_t(\mathbf M_t),$ $g_t(\mathbf M_t),$ and $h(\mathbf M_t),$ which are slightly off to the true targets $f_t(\mathbf M_{t-H, t}),$ $g_t(\mathbf M_{t-H, t}),$ and $h(\mathbf M_{t-H, t}).$ Therefore, we also need to quantify the mismatches by establishing the stability terms $\|\mathbf M_{t} - \mathbf M_{t+1}\|.$
The following theorem establishes the regret and constraint violation of COCO-Soft. 
\begin{theorem}
Under Assumptions \ref{assumption:set}-\ref{assumption:slater}, COCO-Soft algorithm achieves 
\begin{align*}
    \mathcal R_f(T) =& O\left(\sqrt{T}\log^3 T\right),  ~~\mathcal V_h(t) = O(1/\log T), \forall t \in [T],\\
    &\mathcal V_g^{soft}(T) = O\left(1\right),
\end{align*} for a large $T$ with the probability at least $1-1/T.$ \label{thm:ocowm}
\end{theorem}
We outline the key steps of the proof and leave the details in Appendix \ref{app:soft}.
We first study the regret 
\begin{align*}
\mathcal R_f(T) \leq&  \sum_{t=1}^T|f_t(\mathbf M_{t-H:t}) - f_t(\mathbf M_{t})| +
f_t(\mathbf M_{t}) - f_t(\mathbf M^*) \\
\leq& O(H^2 \sum_{t=1}^T\|\mathbf M_{t+1} - \mathbf M_{t}\|) + O(\sqrt{T}\log^3 T+T\epsilon), 
\end{align*}
which proves the regret bound by letting the pessimistic factor $\epsilon = \Theta(\log^3 T/\sqrt{T})$ in Lemma \ref{lemma:ocowm} and by Lemma \ref{lemma:diff}. Similarly, we establish the soft constraint violation 
\begin{align*}
\mathcal V_g^{soft}(T) \leq \sum_{t=1}^T|g_t(\mathbf M_{t-H:t}) - g_t(\mathbf M_{t})| + g_t(\mathbf M_{t}) \\
= O(H^2 \sum_{t=1}^T\|\mathbf M_{t+1} - \mathbf M_{t}\| + \sqrt{T}\log^3 T-T\epsilon)
\end{align*}
which is $O(1)$ with $\epsilon = \Theta(\log^3 T/\sqrt{T})$ for a relatively large $T.$ Finally, we have anytime violation
\begin{align*}
\mathcal V_h(t) \leq |h(\mathbf M_{t-H:t}) - h(\mathbf M_{t})| + h(\mathbf M_{t}) = O(1/\log T).
\end{align*}
\begin{remark}
Theorem \ref{thm:ocowm} achieves $\Tilde{O}(\sqrt{T})$ regret and $O(1)$ violation, which significantly improves the best existing results of $O(\sqrt{T})$ regret and $O(\sqrt{T})$ violation in \cite{NeeYu_17,YuNeeWei_17}. The key design for such improvement is to introduce the pessimistic factor $\epsilon$ in virtual queue update \cite{AkhAmrRaj_21,LiuLiShi_21} such that 
 we can trade regret (the amount of $T\epsilon$) to achieve the constant soft violation $\mathcal V_g^{soft}(T).$ Note a very recent work \cite{KimLee_23} studied OCO with stochastic constraints and established $O(\sqrt{T})$ regret and $O(1)$ soft cumulative violation by using a similar pessimistic technique as in this paper. However, the result in \cite{KimLee_23} can be regarded as a special case of ours because we considered COCOwM with adversarial constraints.  \label{rmk:bestcoco}      
\end{remark}

\subsection{Theoretical Analysis of COCO-Hard and COCO-Best2Worlds}
Similar to {\bf COCO-Soft}, we establish the regret and constraint violation 
under {\bf COCO-Hard} and {\bf COCO-Best2Worlds} algorithms, 
where we combine the stability term $\|\mathbf M_{t} - \mathbf M_{t+1}\|$ to have Theorems \ref{thm:ocowm-hard} and \ref{thm:ocowm-best}. 

\begin{theorem}
Under Assumptions \ref{assumption:set}-\ref{assumption:obj}, COCO-Hard algorithm achieves  \begin{align*}
    \mathcal R_f(T) =& O\left(T^{\frac{2}{3}}\log^2 T\right),  ~\mathcal V_h(t) = O(\log T/T^{\frac{1}{3}}), \forall t \in [T],\\
    &\mathcal V_g^{hard}(T) = O\left(T^{\frac{2}{3}}\log^2 T\right),
\end{align*} \label{thm:ocowm-hard}
for a large $T.$
\end{theorem}

\begin{theorem}
Under Assumptions \ref{assumption:set}-\ref{assumption:slater},  COCO-Best2Worlds algorithm achieves 
\begin{align*}
    \mathcal R_f(T) = O(\sqrt{T}\log^3 T),  ~~\mathcal V_h(t) = O(1/\log T), \forall t \in [T],\\
    \mathcal V_g^{soft}(T) = O\left(1\right), ~ \mathcal V_g^{hard}(T) = O(\sqrt{T}\log^3 T),
\end{align*}
for a large $T$ with the probability at least $1-1/T.$ \label{thm:ocowm-best}
\end{theorem}

Note once the regret and constraint violation of COCO solvers have been established in Theorems \ref{thm:ocowm}, \ref{thm:ocowm-hard}, and \ref{thm:ocowm-best},
we plug them into the roadmap in Section \ref{sec:template} and prove the performance of COCA in Theorems \ref{thm:occa}, \ref{thm:occa-hard}, and \ref{thm:occa-best}. 

\section{Experiments}
In this section, we test our algorithms on a quadrotor vertical flight (QVF) control under an adversarial environment, which is modified from \cite{LiZhaDas_23}. The experiment is designed to verify if our algorithms are adaptive to time-varying/adversarial constraints. 
We also test our algorithms on a Heating Ventilation and Air Conditioning (HVAC) control with static affine constraints in \cite{LiDasLi_21}. 
The experiment is to verify if our approach is less conservative in designing COCA algorithms. The omitted details and results (e.g. values of parameters in system equations or learning rates of COCA algorithms) are in Appendix~\ref{app-exp}. 
    
{\bf QVF Control}:
The system equation is $\ddot{x}_t = \frac{u_t}{m} - g - \frac{I^a {\dot x}_t}{m} + w_t$, where $x_t$ is the altitude of the quadrotor, $u_t$ is the motor thrust, $m$ is the mass of the quadrotor, $g$ is the gravitational acceleration, and $I^a$ is the drag coefficient of the air resistance. 
The system disturbances are i.i.d. drawn from a uniform distribution $w_t\sim U(-5.5,-4.5)$ simulating winds blowing down.
We impose time-varying constraints, $x_t \ge 0.3 + 0.3\sin(t/10)$, to emulate the adversarial and time-varying obstacles on the ground. The constraints are especially challenging when winds blow down.
The static affine constraints are $x_t\le 1.7$ and $0\le u_t \le 12$. We consider time-varying quadratic cost functions $0.1(x_t-0.7)^2 + 0.1\dot{x}_t^2 + \chi_t (u_t-9.8)^2$ with $\chi_t\sim U(0.1, 0.2).$ 

We compare COCA-Soft, COCA-Hard, and COCA-Best2Worlds with a strongly-stable linear controller $K$, which is obtained by solving an SDP problem \cite{AloAviTom_18}. 
Figure~\ref{fig:exp-results} (a)-(c) show that COCA algorithms achieve much better performance than the stable controller.
Specifically, our algorithms have much smaller cumulative costs, negative cumulative soft violations that decrease over time, and cumulative hard violations that remain constant shortly after the initial stages. These results verify our algorithms are very adaptive to the adversarial environment and achieve the minimum cumulative costs and the best constraint satisfaction. Moreover, we observe COCA algorithms have almost identical performance on cumulative costs and soft violations, but COCA-Hard and COCA-Best2Worlds have a better hard violation than COCA-Soft in Figure~\ref{fig:exp-results} (c). It justifies the penalty-based design is efficient in tackling the hard violation. Note we also test our algorithms with the disturbance $w_t\sim U(4.5,5.5)$ simulating winds blowing up, where we also observe similar results. The details can be found in Appendix~\ref{app-exp}.  

{\bf HVAC Control}:
The system equation is $\dot{x}_t = \frac{\theta^o - x_t}{v\zeta} - \frac{u_t}{v} + \frac{w_t+\iota}{v},$ where $x_t$ is the room temperature, $u_t$ is the airflow rate of the HVAC system as the control input, $\theta^o$ is the outdoor temperature, $w_t$ is the random disturbance, $\iota$ represents the impact of the external heat sources, $v$ and $\zeta$ denotes the environmental factors. 
Similar to \cite{LiDasLi_21}, the state and input constraints are $22.5 \le x_t \le 25.5$ and $0.5\le u_t\le 4.5,$ and the time-varying cost functions are $2(x_t-24)^2 + \chi_t(u_t-2.5)^2$ with $\chi_t \sim U(0.1, 4.0).$

We compare COCA with OGD-BZ algorithm (COCA-Soft, COCA-Hard, and COCA-Best2Worlds are exactly identical, called COCA, when only static constraints exist). Figure~\ref{fig:exp-results} (d)-(e) show the cumulative costs and the room temperature $x_t$. We observe that COCA has a significantly better cumulative cost than OGD-BZ algorithm with a near-zero constraint violation. The results verify our approach is effective in designing less-conservative COCA algorithms. 
\begin{figure*}
    \centering
    \tabcolsep=0.02\linewidth
    \divide\tabcolsep by 8
    \begin{tabular}{ccccc}
    \includegraphics[width=0.194\linewidth]{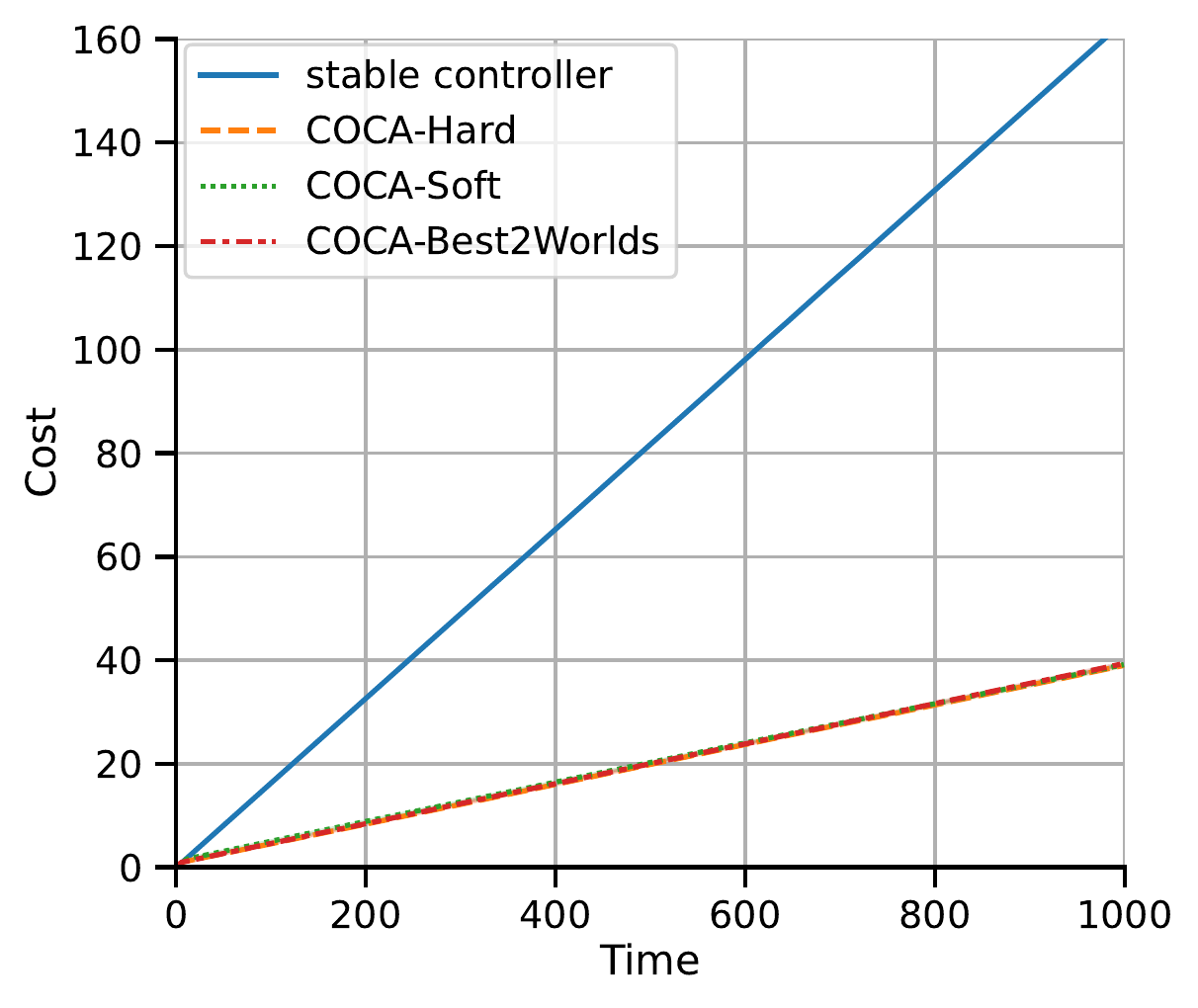} &
        \includegraphics[width=0.192\linewidth]{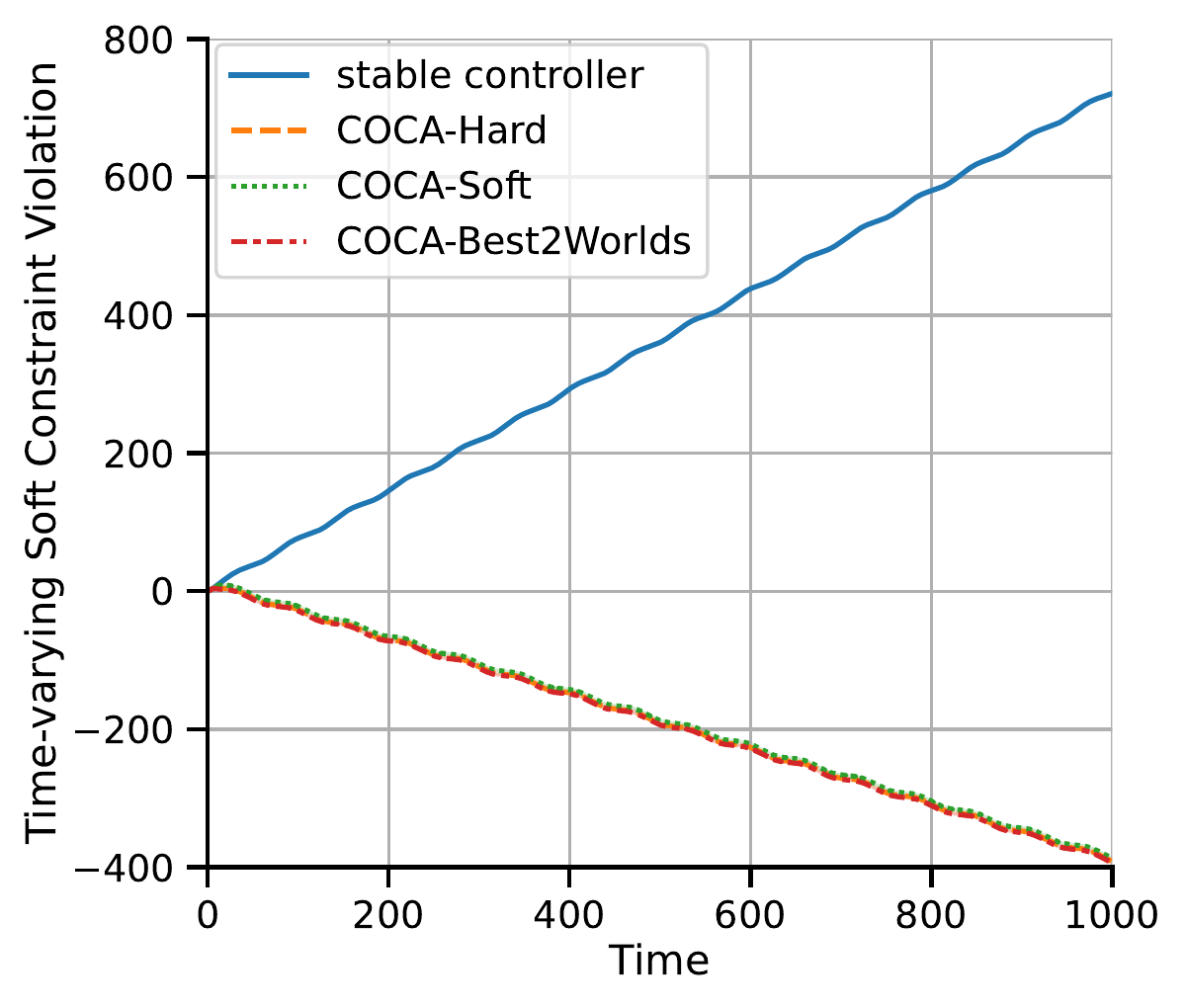} &
        \includegraphics[width=0.19\linewidth]{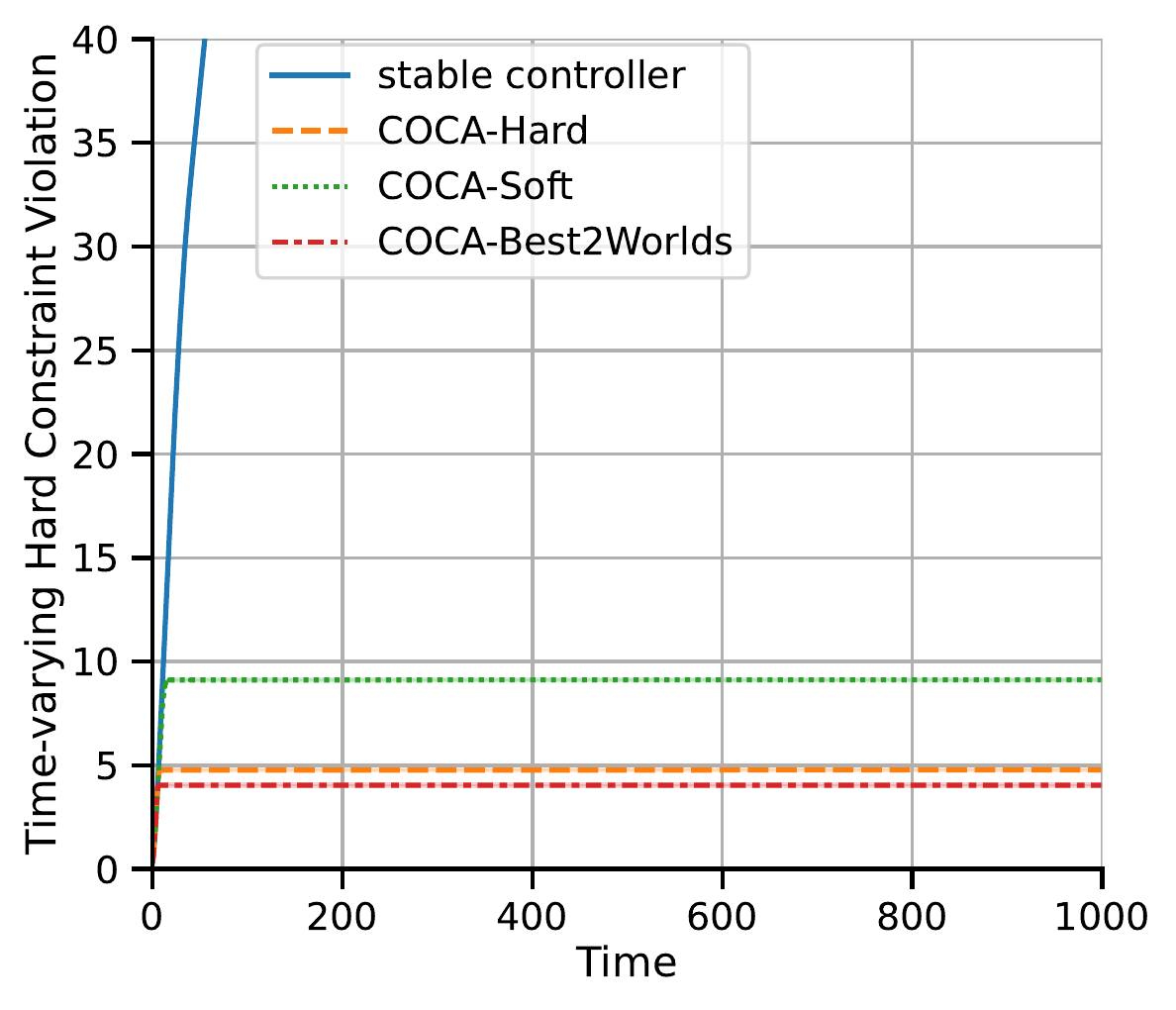} &
        \includegraphics[width=0.195\linewidth]{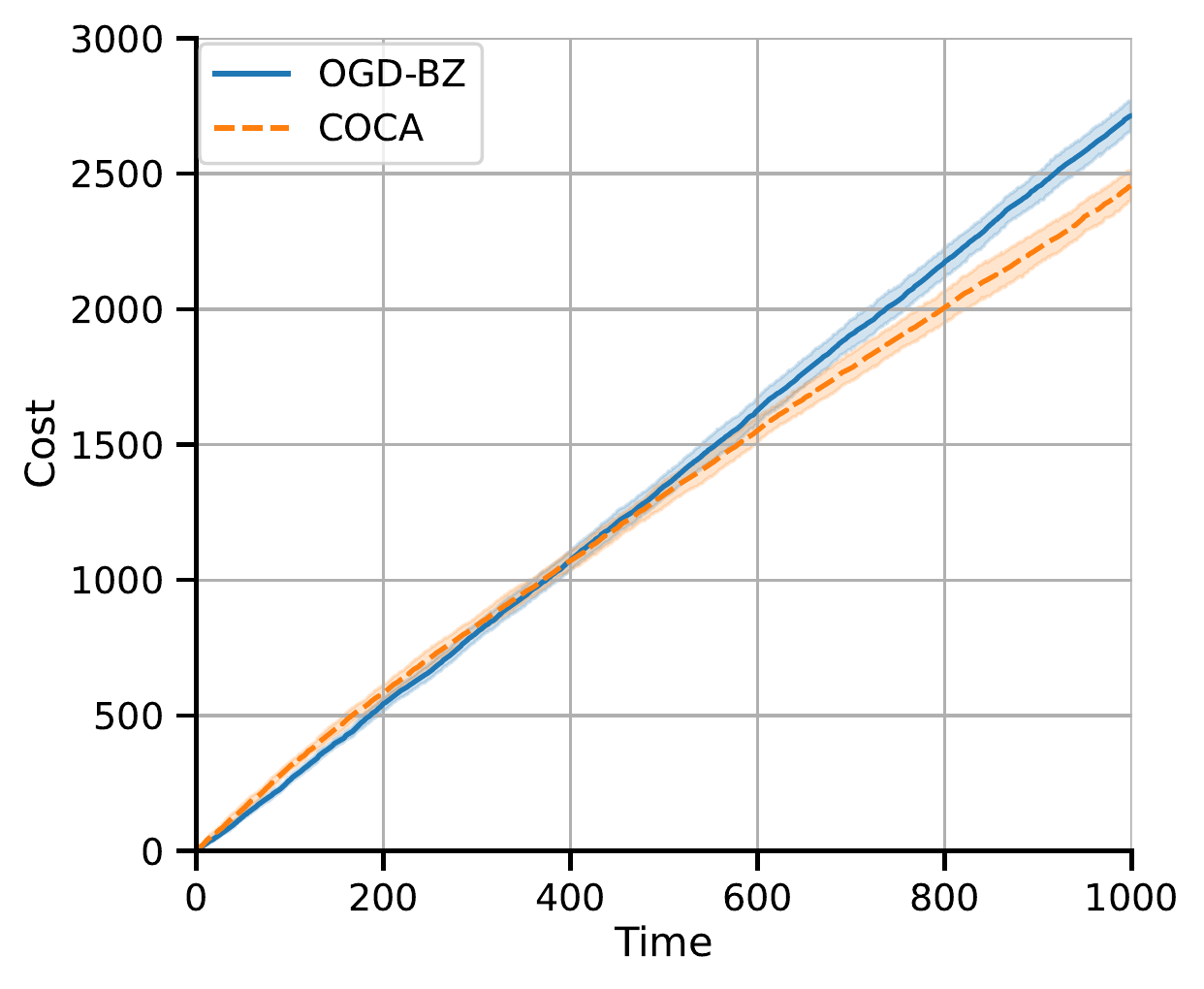} &
        \includegraphics[width=0.195\linewidth]{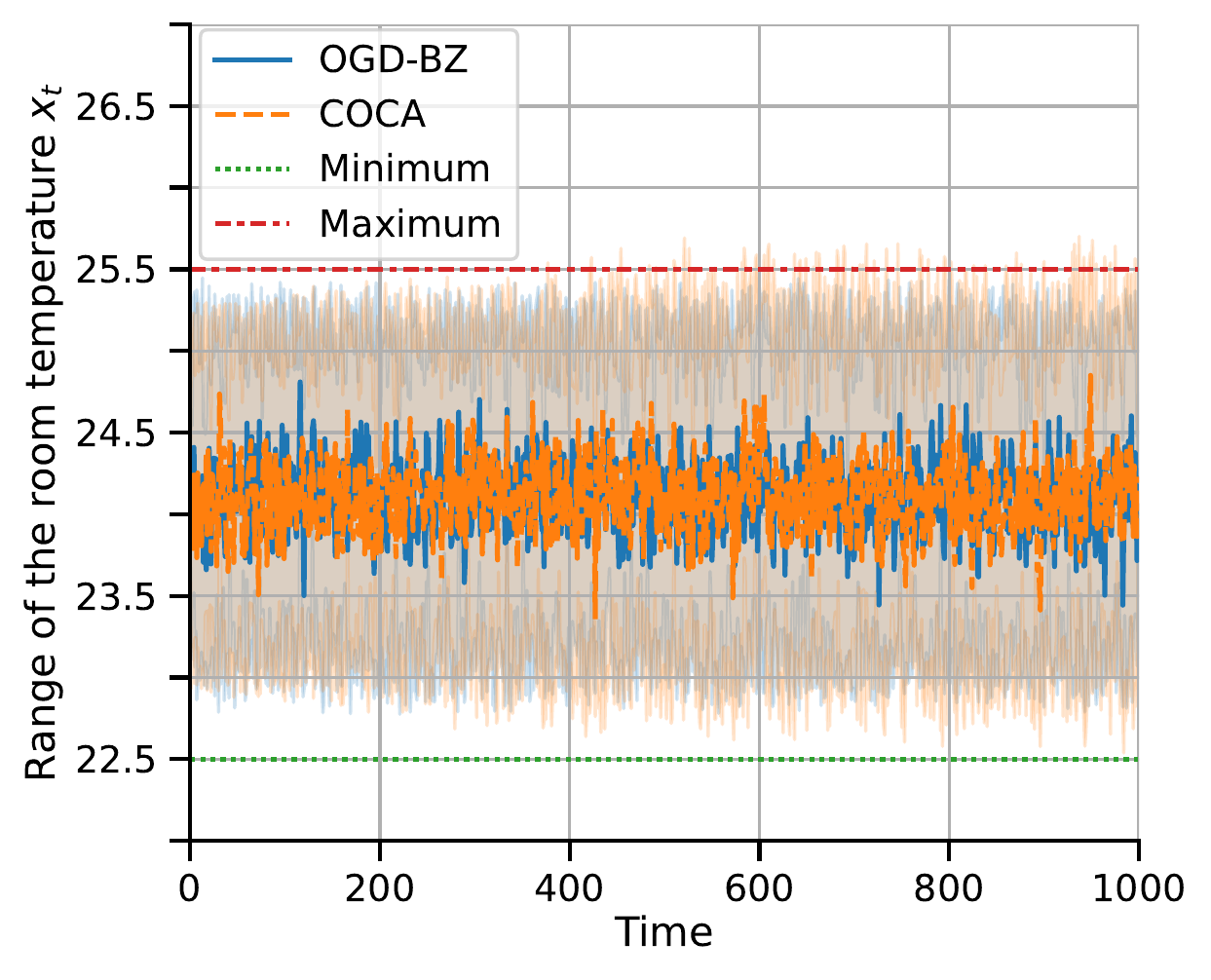}\\
        \small (a) Cumulative costs &
        \small (b) Soft violations &
        \small (c) Hard violations &
        \small (d) Cumulative costs &
        \small (e) Range of state $x_t$
    \end{tabular}
    \caption{Experiment results: Figures (a)-(c) are results for QVF control; Figures (d) (e) are results for HVAC control. 
    The lines are plotted by averaging over 10 independent runs. The shaded areas in Figures (a)-(d) are 95\% confidence intervals. The shaded areas in Figure (e) are the full ranges of the states.
    }
    \label{fig:exp-results}
\end{figure*}

\section{Conclusions}
In this paper, we studied online nonstochastic control problems with adversarial and static constraints. We developed COCA algorithms that minimize cumulative costs and soft or/and hard constraint violations based on Lyapunov optimization and proximal penalty-based methods. Our experiments showed the proposed algorithms are adaptive to time-varying environments and less conservative in achieving a better performance. 

\bibliographystyle{plain}
\bibliography{ref}

\begin{thebibliography}{10}

\bibitem{AgaBulHaz_19}
Naman Agarwal, Brian Bullins, Elad Hazan, Sham Kakade, and Karan Singh.
\newblock Online control with adversarial disturbances.
\newblock {\em Proceedings of the 36th International Conference on Machine
  Learning}, 2019.

\bibitem{AgaHazSin_19}
Naman Agarwal, Elad Hazan, and Karan Singh.
\newblock Logarithmic regret for online control.
\newblock {\em Advances Neural Information Processing Systems (NeurIPS)}, 2019.

\bibitem{AkhAmrRaj_21}
Zeeshan Akhtar, Amrit~Singh Bedi, and Ketan Rajawat.
\newblock Conservative stochastic optimization with expectation constraints.
\newblock {\em IEEE Transactions on Signal Processing}, 2021.

\bibitem{AmaThrYan_21}
Sanae Amani, Christos Thrampoulidis, and Lin Yang.
\newblock Safe reinforcement learning with linear function approximation.
\newblock In {\em Proceedings of the 38th International Conference on Machine
  Learning}, 2021.

\bibitem{AniHumSha_13}
Anil Aswani, Humberto Gonzalez, S.~Shankar Sastry, and Claire Tomlin.
\newblock Provably safe and robust learning-based model predictive control.
\newblock {\em Automatica}, 2013.

\bibitem{LukMelAda_22}
Lukas Brunke, Melissa Greeff, Adam~W. Hall, Zhaocong Yuan, Siqi Zhou, Jacopo
  Panerati, and Angela~P. Schoellig.
\newblock Safe learning in robotics: From learning-based control to safe
  reinforcement learning.
\newblock {\em Annual Review of Control, Robotics, and Autonomous Systems},
  2022.

\bibitem{CaoZhaPoo_21}
Xuanyu Cao, Junshan Zhang, and H.~Vincent Poor.
\newblock Online stochastic optimization with time-varying distributions.
\newblock {\em IEEE Transactions on Automatic Control}, 2021.

\bibitem{GonMar_93}
Gong Chen and Marc Teboulle.
\newblock Convergence analysis of a proximal-like minimization algorithm using
  bregman functions.
\newblock {\em SIAM Journal on Optimization}, 1993.

\bibitem{CheJaiLuo_22}
Liyu Chen, Rahul Jain, and Haipeng Luo.
\newblock Learning infinite-horizon average-reward {M}arkov decision process
  with constraints.
\newblock In {\em Proceedings of the 39th International Conference on Machine
  Learning}, 2022.

\bibitem{CheMinLee_22}
Xinyi Chen, Edgar Minasyan, Jason~D. Lee, and Elad Hazan.
\newblock Provable regret bounds for deep online learning and control.
\newblock {\em arXiv preprint arXiv:2110.07807}, 2022.

\bibitem{CheOroMur_19}
Richard Cheng, G\'{a}bor Orosz, Richard~M. Murray, and Joel~W. Burdick.
\newblock End-to-end safe reinforcement learning through barrier functions for
  safety-critical continuous control tasks.
\newblock In {\em Proceedings of the Thirty-Third AAAI Conference on Artificial
  Intelligence}, 2019.

\bibitem{AloAviTom_18}
Alon Cohen, Avinatan Hasidim, Tomer Koren, Nevena Lazic, Yishay Mansour, and
  Kunal Talwar.
\newblock Online linear quadratic control.
\newblock In {\em Proceedings of the 35th International Conference on Machine
  Learning}, Proceedings of Machine Learning Research, 2018.

\bibitem{DeaTuSte_19}
Sarah Dean, Stephen Tu, Nikolai Matni, and Benjamin Recht.
\newblock Safely learning to control the constrained linear quadratic
  regulator.
\newblock In {\em 2019 American Control Conference (ACC)}, 2019.

\bibitem{DinWeiYan_21}
Dongsheng Ding, Xiaohan Wei, Zhuoran Yang, Zhaoran Wang, and Mihailo Jovanovic.
\newblock Provably efficient safe exploration via primal-dual policy
  optimization.
\newblock In {\em Proceedings of The 24th International Conference on
  Artificial Intelligence and Statistics}, 2021.

\bibitem{DinZhaBas_20}
Dongsheng Ding, Kaiqing Zhang, Tamer Basar, and Mihailo Jovanovic.
\newblock Natural policy gradient primal-dual method for constrained markov
  decision processes.
\newblock In {\em Advances in Neural Information Processing Systems}. Curran
  Associates, Inc., 2020.

\bibitem{EfrManPir_20}
Yonathan Efroni, Shie Mannor, and Matteo Pirotta.
\newblock Exploration-exploitation in constrained mdps.
\newblock {\em arXiv preprint arXiv:2003.02189}, 2020.

\bibitem{FisAkaZei_19}
Jaime~F. Fisac, Anayo~K. Akametalu, Melanie~N. Zeilinger, Shahab Kaynama,
  Jeremy Gillula, and Claire~J. Tomlin.
\newblock A general safety framework for learning-based control in uncertain
  robotic systems.
\newblock {\em IEEE Transactions on Automatic Control}, 2019.

\bibitem{FosSim_20}
Dylan~J. Foster and Max Simchowitz.
\newblock Logarithmic regret for adversarial online control.
\newblock ICML'20. JMLR.org, 2020.

\bibitem{JavFer_15}
Javier Garc{{\'i}}a, Fern, and o~Fern{{\'a}}ndez.
\newblock A comprehensive survey on safe reinforcement learning.
\newblock {\em Journal of Machine Learning Research}, 2015.

\bibitem{GhoZhoShr_22}
Arnob Ghosh, Xingyu Zhou, and Ness Shroff.
\newblock Provably efficient model-free constrained {RL} with linear function
  approximation.
\newblock In {\em Advances in Neural Information Processing Systems}, 2022.

\bibitem{GuoLiuWei_22}
Hengquan Guo, Xin Liu, Honghao Wei, and Lei Ying.
\newblock Online convex optimization with hard constraints: Towards the best of
  two worlds and beyond.
\newblock In {\em Advances in Neural Information Processing Systems}, 2022.

\bibitem{Haj_82}
B.~Hajek.
\newblock Hitting-time and occupation-time bounds implied by drift analysis
  with applications.
\newblock {\em Ann. Appl. Prob.}, 1982.

\bibitem{Haz_16}
Elad Hazan.
\newblock Introduction to online convex optimization.
\newblock {\em Foundations and Trends{\textregistered} in Optimization}, 2016.

\bibitem{HazSin_22}
Elad Hazan and Karan Singh.
\newblock Introduction to online nonstochastic control.
\newblock {\em arXiv preprint arXiv:2111.09619}, 2022.

\bibitem{JacGabHon_18}
Arthur Jacot, Franck Gabriel, and Cl\'{e}ment Hongler.
\newblock Neural tangent kernel: Convergence and generalization in neural
  networks.
\newblock {\em Advances Neural Information Processing Systems (NeurIPS)}, 2018.

\bibitem{KolBerTur_18}
Torsten Koller, Felix Berkenkamp, Matteo Turchetta, and Andreas Krause.
\newblock Learning-based model predictive control for safe exploration.
\newblock In {\em 2018 IEEE Conference on Decision and Control (CDC)}, 2018.

\bibitem{LiDasLi_21}
Yingying Li, Subhro Das, and Na~Li.
\newblock Online optimal control with affine constraints.
\newblock {\em AAAI Conf. Artificial Intelligence}, 2021.

\bibitem{LiZhaDas_23}
Yingying Li, Tianpeng Zhang, Subhro Das, Jeff Shamma, and Na~Li.
\newblock Safe adaptive learning for linear quadratic regulators with
  constraints.
\newblock {\em Technical report}, 2023.

\bibitem{LiuZhoKal_21}
Tao Liu, Ruida Zhou, Dileep Kalathil, Panganamala Kumar, and Chao Tian.
\newblock Learning policies with zero or bounded constraint violation for
  constrained {MDP}s.
\newblock In {\em Advances in Neural Information Processing Systems}, 2021.

\bibitem{LiuLiShi_21}
Xin Liu, Bin Li, Pengyi Shi, and Lei Ying.
\newblock An efficient pessimistic-optimistic algorithm for stochastic linear
  bandits with general constraints.
\newblock In A.~Beygelzimer, Y.~Dauphin, P.~Liang, and J.~Wortman Vaughan,
  editors, {\em Advances in Neural Information Processing Systems}, 2021.

\bibitem{MahJinYan_12}
Mehrdad Mahdavi, Rong Jin, and Tianbao Yang.
\newblock Trading regret for efficiency: online convex optimization with long
  term constraints.
\newblock {\em The Journal of Machine Learning Research}, 2012.

\bibitem{MaySerRak_05}
D.Q. Mayne, M.M. Seron, and S.V. Raković.
\newblock Robust model predictive control of constrained linear systems with
  bounded disturbances.
\newblock {\em Automatica}, 2005.

\bibitem{MinGraSim_21}
Edgar Minasyan, Paula Gradu, Max Simchowitz, and Elad Hazan.
\newblock Online control of unknown time-varying dynamical systems.
\newblock In {\em Advances in Neural Information Processing Systems}, 2021.

\bibitem{Nee_10}
Michael~J. Neely.
\newblock Stochastic network optimization with application to communication and
  queueing systems.
\newblock {\em Synthesis Lectures on Communication Networks}, 2010.

\bibitem{NeeYu_17}
Michael~J Neely and Hao Yu.
\newblock Online convex optimization with time-varying constraints.
\newblock {\em arXiv preprint arXiv:1702.04783}, 2017.

\bibitem{NonMulMat_21}
Marko Nonhoff and Matthias~A. Müller.
\newblock An online convex optimization algorithm for controlling linear
  systems with state and input constraints.
\newblock In {\em 2021 American Control Conference (ACC)}, 2021.

\bibitem{OcoShiShi_22}
Michael O’Connell, Guanya Shi, Xichen Shi, Kamyar Azizzadenesheli, Anima
  Anandkumar, Yisong Yue, and Soon-Jo Chung.
\newblock Neural-fly enables rapid learning for agile flight in strong winds.
\newblock {\em Science Robotics}, 2022.

\bibitem{RakZhaSun_23}
Saša~V. Raković, Sixing Zhang, Haidi Sun, and Yuanqing Xia.
\newblock Model predictive control for linear systems under relaxed
  constraints.
\newblock {\em IEEE Transactions on Automatic Control}, 2023.

\bibitem{RawMayDie_17}
J.~Rawlings, D.Q. Mayne, and Moritz Diehl.
\newblock {\em Model Predictive Control: Theory, Computation, and Design}.
\newblock 2017.

\bibitem{SimSinHaz_20}
Max Simchowitz, Karan Singh, and Elad Hazan.
\newblock Improper learning for non-stochastic control.
\newblock In {\em Proceedings of Thirty Third Conference on Learning Theory},
  2020.

\bibitem{SunDeyKap_17}
Wen Sun, Debadeepta Dey, and Ashish Kapoor.
\newblock Safety-aware algorithms for adversarial contextual bandit.
\newblock In {\em International Conference on Machine Learning}, 2017.

\bibitem{SuoGhaMin_21}
Daniel Suo, Udaya Ghai, Edgar Minasyan, Paula Gradu, Xinyi Chen, Naman Agarwal,
  Cyril Zhang, Karan Singh, Julienne LaChance, Tom Zadjel, Manuel Schottdorf,
  Daniel Cohen, and Elad Hazan.
\newblock Machine learning for mechanical ventilation control.
\newblock 2021.

\bibitem{TirBarDen_20}
Muhammad Tirmazi, Adam Barker, Nan Deng, Md~E. Haque, Zhijing~Gene Qin, Steven
  Hand, Mor Harchol-Balter, and John Wilkes.
\newblock Borg: The next generation.
\newblock EuroSys, 2020.

\bibitem{VasYanSze_22}
Sharan Vaswani, Lin Yang, and Csaba Szepesvari.
\newblock Near-optimal sample complexity bounds for constrained {MDP}s.
\newblock In {\em Advances in Neural Information Processing Systems}, 2022.

\bibitem{WabKriZei_22}
Kim~P. Wabersich, Raamadaas Krishnadas, and Melanie~N. Zeilinger.
\newblock A soft constrained mpc formulation enabling learning from
  trajectories with constraint violations.
\newblock {\em IEEE Control Systems Letters}, 2022.

\bibitem{WabZei_18}
Kim~P. Wabersich and Melanie~N. Zeilinger.
\newblock Linear model predictive safety certification for learning-based
  control.
\newblock In {\em 2018 IEEE Conference on Decision and Control (CDC)}, 2018.

\bibitem{WeiLiuYin_22}
Honghao Wei, Xin Liu, and Lei Ying.
\newblock Triple-q: A model-free algorithm for constrained reinforcement
  learning with sublinear regret and zero constraint violation.
\newblock In {\em Proceedings of The 25th International Conference on
  Artificial Intelligence and Statistics}, 2022.

\bibitem{KimLee_23}
Kim Yeongjong and Lee Dabeen.
\newblock Online convex optimization with stochastic constraints: Zero
  constraint violation and bandit feedback.
\newblock {\em arXiv preprint arXiv:2301.11267}, 2023.

\bibitem{YiLiYan_21a}
Xinlei Yi, Xiuxian Li, Tao Yang, Lihua Xie, Tianyou Chai, and Karl Johansson.
\newblock Regret and cumulative constraint violation analysis for online convex
  optimization with long term constraints.
\newblock In {\em International Conference on Machine Learning}. PMLR, 2021.

\bibitem{YiLiYan_21b}
Xinlei Yi, Xiuxian Li, Tao Yang, Lihua Xie, Tianyou Chai, and Karl~H Johansson.
\newblock Regret and cumulative constraint violation analysis for distributed
  online constrained convex optimization.
\newblock {\em arXiv preprint arXiv:2105.00321}, 2021.

\bibitem{YuNeeWei_17}
Hao Yu, Michael Neely, and Xiaohan Wei.
\newblock Online convex optimization with stochastic constraints.
\newblock {\em Advances in Neural Information Processing Systems}, 30, 2017.

\bibitem{ZeiJonMor_10}
Melanie~N. Zeilinger, Colin~N. Jones, and Manfred Morari.
\newblock Robust stability properties of soft constrained mpc.
\newblock In {\em 49th IEEE Conference on Decision and Control (CDC)}, 2010.

\end{thebibliography}

\newpage
\appendix
\section{Auxiliary Lemmas for COCO-Solver}
The following lemma provides a useful upper bound on the optimal value of strongly convex function \cite{GonMar_93, YuNeeWei_17}.
\begin{lemma}
Let $\mathcal U$ be a convex set. Let $F: \mathcal U \to \mathcal R$ be $\alpha$-strongly convex function on $\mathcal U$ and $u_{opt}$ be an optimal solution of $F,$ i.e., $u_{opt} = \argmin_{u\in\mathcal U} F(u).$ Then, $F(u_{opt}) \leq F(u) - \frac{\alpha}{2}\|u - u_{opt}\|^2$ holds for any $u\in \mathcal U.$ 
\label{lemma:tool}
\end{lemma}
\begin{proof}
\normalfont
The proof of the lemma is based on the definition of $\alpha$-strongly convex functions and the first-order optimality condition. Define the subgradient of $F(u)$ to be $\nabla F(u),$ According to the definition of strong convexity, we have 
\begin{align}
F(u) \geq F(v) +  \langle \nabla F(v) , u - v\rangle + \frac{\alpha}{2} \|u - v\|^2. \label{eq:Fu}
\end{align}
Define $u_{opt} = \argmin_{u\in\mathcal U} F(u).$ Let $v=u_{opt}$ in \eqref{eq:Fu}, we have
\begin{align*}
F(u) \geq F(u_{opt}) +  \langle \nabla F(u_{opt}) , u - u_{opt}\rangle + \frac{\alpha}{2} \|u - u_{opt}\|^2.
\end{align*}
We then conclude the proof based on the first-order optimality condition that for any $u\in \mathcal U,$ $$\langle \nabla F(u_{opt}) , u - u_{opt} \rangle\geq 0.$$ 
\end{proof}

We introduce the following two lemmas to ``compare'' two optimization problems: an original problem with its relaxed version.
\begin{lemma}
Consider a convex optimization problem defined on $\mathcal U$ with the objective function $F(u)$ and constraints $G(u) \leq 0,$ that is,
\begin{align}
    \min_{u\in \mathcal U} ~ F(u) ~~ 
    \text{s.t.} ~ G(u) \leq 0. \label{eq:prob-original}
\end{align}
Assume Slater's condition holds, i.e., there exists $u \in \mathcal U$ and a positive constant $\delta$ such that $G(u) \leq -\delta.$ Given a ``loose'' convex optimization problem with $\delta> \epsilon > 0,$ we have
\begin{align}
    \min_{u\in \mathcal U} ~ F(u) ~~ 
    \text{s.t.} ~ G(u) \leq \epsilon. \label{eq:prob-loose}
\end{align}
Assume $\mathcal U$ is bounded with radius $R$ and $F(u)$ is Lipschitz continuous with constant $L.$ Let $u^{*}$ and $u^{*,\epsilon}$ be the optimal solution to \eqref{eq:prob-original} and \eqref{eq:prob-loose}, respectively, we have
\begin{align*}
    F(u^{*}) - F(u^{*,\epsilon}) \leq LR\epsilon/\delta. 
\end{align*}\label{lem:fun-loose}
\end{lemma}
\begin{proof}
\normalfont
Let $u^{\epsilon}$ be any feasible point to \eqref{eq:prob-original} and let $u^{\text{in}}$ be a point such that $G(u^{\text{in}}) \leq -\delta + \epsilon.$ Construct $u = (1-\frac{\epsilon}{\delta}) u^{\epsilon} + \frac{\epsilon}{\delta} u^{\text{in}},$ we have
\begin{align*}
    G(u) =& G((1-\frac{\epsilon}{\delta}) u^{\epsilon} + \frac{\epsilon}{\delta} u^{\text{in}}) \\
    \leq& (1-\frac{\epsilon}{\delta}) G(u^{\epsilon}) + \frac{\epsilon}{\delta} G(u^{\text{in}}) \\
    \leq& 0,
\end{align*}
which implies $u$ is a feasible solution to \eqref{eq:prob-loose}. Therefore, we have
\begin{align*}
    F(u^*) - F(u^\epsilon) \leq F(u) - F(u^\epsilon) \leq L \|u-u^\epsilon\| \leq LR\epsilon/\delta,
\end{align*}
which concludes the proof by letting $u^{\epsilon}=u^{*,\epsilon}.$ 
\end{proof}

Similar with Lemma \ref{lem:fun-loose} that compares the original problem with a relaxed version, we are able to compare the original problem with a tight version in Lemma \ref{lem:fun-tight}.
\begin{lemma}\label{lem:tightness}
Consider a convex optimization problem defined on $\mathcal U$ with the objective function $F(u)$ and constraints $G(u) \leq 0,$ that is,
\begin{align}
    \min_{u\in \mathcal U} ~ F(u) ~~ 
    \text{s.t.} ~ G(u) \leq 0. \label{eq:prob-original-1}
\end{align}
Assume Slater's condition holds, i.e., there exists $u \in \mathcal U$ and a positive constant $\delta$ such that $G(u) \leq -\delta.$ Given a ``tight'' convex optimization problem with $\delta> \epsilon > 0,$ we have
\begin{align}
    \min_{u\in \mathcal U} ~ F(u) ~~ 
    \text{s.t.} ~ G(u) \leq -\epsilon. \label{eq:prob-strict}
\end{align}
Assume $\mathcal U$ is bounded with radius $R$ and $F(u)$ is Lipschitz continuous with constant $L.$ Let $u^{*}$ and $u^{*,\epsilon}$ be the optimal solution to \eqref{eq:prob-original-1} and \eqref{eq:prob-strict}, respectively, we have 
\begin{align*}
    F(u^{*,\epsilon})-F(u^{*}) \leq LR\epsilon/\delta. 
\end{align*}\label{lem:fun-tight}
\end{lemma}
We present the lemma of Lyapunov drift analysis in \cite{YuNeeWei_17}, which is used to quantify the virtual queue $Q_t$ (a proxy of soft violation). The reader may refer to \cite{Haj_82, Nee_10, YuNeeWei_17} for more details on these techniques. 
\begin{lemma} \label{lemma:drift}
Let $Z_t$ be a random process with $Z_0=0$ and $\mathcal H_t$ be the filtration at time $t.$ Given $\theta,$ $\nu,$ and $\nu_{\max}$ with $0 < \nu \leq \nu_{\max},$ suppose the following conditions hold for $t_0:$
\begin{itemize}
    \item[(i)] There exists constants $\nu > 0$ and $\theta > 0$ such that  $\mathbb E[Z_{t+t_0}-Z_t|\mathcal H_t] \leq -\nu t_0$ when $z \geq \theta,$ 
\item[(ii)] $|Z_{t+1} - Z_t| \leq \nu_{\max}$ holds with probability one;
\end{itemize}
then we have
\begin{align}
    \mathbb E[Z_t]
    \leq \theta + t_0 \nu_{\max} + t_0\frac{4 \nu_{\max}^2}{\nu} \log\frac{8 \nu_{\max}^2}{\nu^2}, \label{eq: q exp bound}
\end{align}
and
\begin{align}
    \mathbb P[Z_t \geq z]
    \leq 1-p, \label{eq: q exp bound}
\end{align}
with $z = \theta + t_0 \nu_{\max} + t_0\frac{4 \nu_{\max}^2}{\nu}\left(\log\frac{8 \nu_{\max}^2}{\nu^2} +\log \frac{1}{p}\right).$
\end{lemma}

\section{Auxiliary Lemmas for COCA} \label{app:coc}
The following lemmas are from \cite{AgaBulHaz_19} and we include them for the sake of completeness. 
\begin{lemma}
Under a disturbance-action policy $\pi(K,\{\mathbf M_t\}),$ the disturbance-state transfer matrix is bounded as follows:
\begin{align*}
    \|\Psi_{t,i}^{\pi}(\mathbf M_{t-H}, \cdots, \mathbf M_{t-1})\|
    \leq \kappa^2(1-\rho)^{i-1}\mathbb I(i\leq H) + Ha\kappa^2\kappa_B(1-\rho)^{i-1}.
\end{align*}
\end{lemma}
\begin{proof}
\normalfont
We establish the disturbance-state transfer matrix is bounded according to its definition  $$\Psi_{t,i}^\pi(\mathbf M_{t-H}, \cdots, \mathbf M_{t-1}) = \tilde{A}^{i-1} \mathbb I(i\leq H) + \sum_{j=1}^H \tilde{A}^{j-1} B \mathbf M^{[i-j]}_{t-j} \mathbb I_{i-j \in [1,H]},$$ where $\tilde{A} = A - KB.$ We have
\begin{align*}
    \|\Psi_{t,i}^\pi(\mathbf M_{t-H}, \cdots, \mathbf M_{t-1})\| \leq&  \|\tilde{A}^{i-1}\|\mathbb I(i\leq H) + \sum_{j=1}^H \|\tilde{A}^{j-1} B \mathbf M^{[i-j]}_{t-j}\| \\
    \leq& \kappa^2(1-\rho)^{i-1}\mathbb I(i\leq H) + Ha\kappa^2\kappa_B(1-\rho)^{i-1}.
\end{align*}
\end{proof}

\begin{lemma}
Under a disturbance-action policy $\pi(K,\{\mathbf M_t\}),$ the disturbance-state and controller are bounded as follows:
\begin{align*}
\|x^\pi_{t}\| \leq D, ~~
\|u^\pi_{t}\| \leq (\kappa+1) D. \\
\end{align*}
where $D=\frac{W(\kappa^2 + Ha\kappa_B\kappa^2)}{\rho(1-\kappa^2(1-\rho)^{H+1})}.$
\end{lemma}
\begin{proof}
\normalfont
We establish the bound on the state $x_t,$ which corresponds to the system stability. According to the definition of the state update $x^\pi_t = \tilde{A}^H x^\pi_{t-H} + \sum_{i=1}^{2H}  \Psi_{t,i}^\pi w_{t-i},$ we have
\begin{align*}
    \|x^\pi_{t}\| =& \|\tilde{A}^H x^\pi_{t-H} + \sum_{i=1}^{2H}  \Psi_{t,i}^\pi w_{t-i}\| \\
    \leq& \|\tilde{A}^H x^\pi_{t-H}\| + \| \sum_{i=1}^{2H}  \Psi_{t,i}^\pi w_{t-i}\| \\
    \leq& \|\tilde{A}^H\| \|x^\pi_{t-H}\| +  \sum_{i=1}^{2H} \|\Psi_{t,i}^\pi\| \|w_{t-i}\|\\
    \leq& \kappa^2(1-\rho)^{H} \|x^\pi_{t-H}\| + \frac{W(\kappa^2 + Ha\kappa_B\kappa^2)}{\rho} 
\end{align*}
which implies 
$$\|x^\pi_{t}\| \leq \frac{W(\kappa^2 + Ha\kappa_B\kappa^2)}{\rho(1-\kappa^2(1-\rho)^{H+1})}=D.$$
For the controller $u_t^\pi,$ we have 
\begin{align*}
    \|u^\pi_{t}\| = \|-K x^\pi_t + \sum_{i=1}^H M^{[i]} w_{t-i}\| 
    \leq \kappa \|x^\pi_t\| + \|\sum_{i=1}^H M^{[i]} w_{t-i}\| 
    \leq (\kappa + 1)D. 
\end{align*}
\end{proof}

\begin{lemma}
Under a disturbance-action policy, the difference of a state and its approximation is bounded as follows:
\begin{align*}
\|x^\pi_t - \tilde x_t^\pi\| 
\leq (1-\rho)^H \frac{W\kappa^2(\kappa^2 + Ha\kappa_B\kappa^2)}{\rho(1-\kappa^2(1-\rho)^{H+1})},\\
\|u^\pi_t - \tilde u_t^\pi\| 
\leq \kappa(1-\rho)^H \frac{W\kappa^2(\kappa^2 + Ha\kappa_B\kappa^2)}{\rho(1-\kappa^2(1-\rho)^{H+1})}. 
\end{align*}\label{lem: dac approx error}
\end{lemma}
\begin{proof}
\normalfont
According to the definition of the approximated states, we have
\begin{align*}
    x_{t}^\pi =& \tilde{A}^H x^\pi_{t-H} + \sum_{i=1}^{2H}  \Psi_{t,i}^\pi w_{t-i}, \\
    \tilde x_t^\pi =& \sum_{i=1}^{2H}  \Psi_{t,i}^\pi w_{t-i}. 
\end{align*}
We study the difference $x^\pi_t$ and $\tilde x_t^\pi$ as follows
\begin{align*}
    \|x^\pi_t-\tilde x_t^\pi\| =& \|\tilde{A}^H x^\pi_{t-H}\| \\
    \leq& \kappa^2 (1-\rho)^H \|x^\pi_{t-H}\| \\
    \leq&  (1-\rho)^H \frac{W\kappa^2(\kappa^2 + Ha\kappa_B\kappa^2)}{\rho(1-\kappa^2(1-\rho)^{H+1})}.
\end{align*}
According to the definition of the (approximated) controller, we have  
\begin{align*}
    u^\pi_t =& -K x^\pi_t + \sum_{i=1}^H M^{[i]} w_{t-i}, \\
    \tilde u^\pi_t =&  -K \tilde x_t^\pi + \sum_{i=1}^H M^{[i]} w_{t-i}. 
\end{align*}
We study the difference $u^\pi_t$ and $\tilde u^\pi_t$ as follows
\begin{align*}
\|u^\pi_t - \tilde u^\pi_t\| 
\leq \|K\|\|x^\pi_t - \tilde x^\pi_t\|  
\leq \kappa(1-\rho)^H \frac{W\kappa^2(\kappa^2 + Ha\kappa_B\kappa^2)}{\rho(1-\kappa^2(1-\rho)^{H+1})}. 
\end{align*}
\end{proof}

\section{Proof of COCA with COCO-Soft Solver}\label{app:soft}
Let the learning rates be $V=\sqrt{T}\log^4 T, \eta = T^{1.5}\log^2 T, \alpha= T\log^7 T,$ and $\epsilon = (1+D^2)\log^3 T/\sqrt{T}.$ We define the parameters in this section as follows: $\nu_{\max}=E+LD+\epsilon,$ $\theta= 2\nu_{\max} + \frac{2VLD+4\nu_{\max}^2}{\xi} + \frac{2\alpha D^2}{\xi S},$ and $Q_{\max} =\theta + S \nu_{\max} + \frac{16 S \nu_{\max}^2}{\xi} \log\frac{128 \nu_{\max}^2T^2}{\xi^2}$ with $S=\sqrt{T}\log^5 T,$ where $L,D,E,$ and $\xi$ are parameters related to $\tilde{c}_t(\mathbf M), \tilde{d}_t(\mathbf M),$ and $\tilde{l}(\mathbf M)$ in Assumptions \ref{assumption:set}, \ref{assumption:obj}, and \ref{assumption:slater} when invoking COCO-Soft and will be specified at the end of the subsection \ref{sec: para soft}. 

\subsection{Performance of COCA with COCO-Soft Solver in Theorem \ref{thm:occa}}
According to the roadmap in Section \ref{sec:template}, we plug the regret and constraint violation of COCO-Soft solver in Theorem \ref{thm:ocowm} to justify Theorem \ref{thm:occa}. 

{\bf Regret analysis:} we have the following regret decomposition 
\begin{align}
\mathcal R(T) 
=& \sum_{t=1}^T \left[c_t(x_t^{\pi}, u_t^{\pi}) -  c_t(\tilde x_t^{\pi}, \tilde u_t^{\pi})\right] 
+ \sum_{t=1}^T  c_t(\tilde x_t^{\pi}, \tilde u_t^{\pi}) - \min_{\pi \in \tilde{\Omega} \bigcap
 \mathcal E}\sum_{t=1}^T c_t(\tilde x_t^{\pi}, \tilde u_t^{\pi}) \nonumber\\
&+\min_{\pi \in \tilde{\Omega} \bigcap
 \mathcal E}\sum_{t=1}^T c_t(\tilde x_t^{\pi}, \tilde u_t^{\pi}) - \sum_{t=1}^T c_t(x^{K^*}_t, u^{K^*}_t) \nonumber\\
\leq&  \frac{TL^2H^2(V+Q_{\max})}{2\alpha} + \frac{T V L^2}{4\alpha} +  \frac{2T\nu^2_{\max}}{V}+ \frac{\alpha D^2}{V} + \frac{\epsilon TLD}{\delta} + LH+2 \nonumber\\
=& O(\sqrt{T}\log^3 T) \nonumber
\end{align}
where the inequality holds because of the regret in Theorem \ref{thm:ocowm} and the approximated error in Lemma \ref{lem:fun-diff} and the representation ability of DAC policy in Lemma \ref{lem:dac-rep}. The order-wise result holds by substituting the learning rates and the parameters.

{\bf Cumulative soft violation of $d_t$ function:} we have the following decomposition for the constraint function $d_t$ 
\begin{align*}
\mathcal V_d^{soft}(T)
\leq& \sum_{t=1}^T [d_t(x_t^\pi, u_t^\pi) - d_t(\tilde x_t^\pi, \tilde u_t^\pi)] + \sum_{t=1}^T d_t(\tilde x_t^\pi, \tilde u_t^\pi)\\
\leq& 1  + \frac{\sqrt{T}L^2H^2(V+Q_{\max})}{2\alpha} + \frac{(\sqrt{T} + Q_{\max})LD + \alpha D^2}{\eta} -T\epsilon\\
=& O(1)
\end{align*}
where the inequality holds because of the soft violation in Theorem \ref{thm:ocowm} and the approximated error in Lemma \ref{lem:fun-diff}. The order-wise result holds by substituting the learning rates and the parameters for a large $T$ such that $\frac{\sqrt{T}\log^{10} T}{V+Q_{\max}} \geq L^2H^2 + LD + D^2.$ 

{\bf Anytime violation of $l$ function:} we have the following decomposition for the constraint function $l$
\begin{align*}
\mathcal V_l(t)
=& l(x_t^\pi, u_t^\pi) - l(\tilde x_t^\pi, \tilde u_t^\pi) + l(\tilde x_t^\pi, \tilde u_t^\pi)\\
\leq& \frac{1}{T}  + \frac{\sqrt{T}L^2H^2(V+Q_{\max})}{2\alpha} + \frac{(\sqrt{T} + Q_{\max})LD + \alpha D^2}{\eta}\\
=&(1/\log T)
\end{align*}
where the inequality holds because of the anytime violation in Theorem \ref{thm:ocowm} and the approximated error in Lemma \ref{lem:fun-diff}. The order-wise result holds by substituting the learning rates and the parameters for a large $T$ such that $\frac{\sqrt{T}\log^{10} T}{V+Q_{\max}} \geq L^2H^2 + LD + D^2.$ 

\subsection{Performance of COCO-Soft Solver in Theorem \ref{thm:ocowm}}
To prove Theorem \ref{thm:ocowm}, we first verify  Lemmas \ref{lemma:ocowm} and \ref{lemma:diff} 
in the following. Note the term of $H$-step difference can be bounded as follows
\begin{align}
\sum_{i=1}^{H} \|\mathbf M_{t-i} - \mathbf M_{t}\| =& \sum_{i=1}^{H} \sum_{j=1}^{i}  \|\mathbf M_{t-j} - \mathbf M_{t-j+1}\|. \nonumber
\end{align}
Based on Lemma \ref{lemma:ocowm}, we have 
\begin{align*}
\sum_{t=1}^T|f_t(\mathbf M_{t-H:t}) - f_t(\mathbf M_{t})|  
\leq& \frac{TL^2H^2(V+Q_{\max})}{2\alpha}, \\
\sum_{t=1}^T|g_t(\mathbf M_{t-H:t}) - g_t(\mathbf M_{t})|  \leq&  \frac{TL^2H^2(V+Q_{\max})}{2\alpha}, \\
|h(\mathbf M_{t-H:t}) - h(\mathbf M_{t})| \leq& \frac{\sqrt{T}L^2H^2(V+Q_{\max})}{2\alpha}.
\end{align*}
Finally, in conjugation with Lemma \ref{lemma:diff}, we prove Theorem \ref{thm:ocowm} 
as follows:
\begin{align*}
\mathcal R_f(T) \leq&  \sum_{t=1}^T|f_t(\mathbf M_{t-H:t}) - f_t(\mathbf M_{t})| +
\sum_{t=1}^T f_t(\mathbf M_{t}) - f_t(\mathbf M^*) \\
\leq& \frac{TL^2H^2(V+Q_{\max})}{2\alpha} + \frac{T V L^2}{4\alpha} +  \frac{2T\nu^2_{\max}}{V}+ \frac{\alpha D^2}{V} + \frac{LDT\epsilon}{\delta}, ~\\
\mathcal V_g^{soft}(T) \leq&  \sum_{t=1}^T|g_t(\mathbf M_{t-H:t}) - g_t(\mathbf M_{t})| +
\sum_{t=1}^T g_t(\mathbf M_{t}) \\
\leq& \frac{TL^2H^2(V+Q_{\max})}{2\alpha} + Q_{\max} + \frac{TL^2(V+Q_{\max})}{2\alpha} - T\epsilon,\\
\mathcal V_h(t) \leq& |h(\mathbf M_{t-H:t}) - h(\mathbf M_{t})| + h(\mathbf M_{t})\\
\leq&\frac{\sqrt{T}L^2H^2(V+Q_{\max})}{2\alpha} + \frac{(\sqrt{T} + Q_{\max})LD + \alpha D^2}{\eta}, \nonumber
\end{align*}
which completes the proof by substituting the learning rates. 

\begin{lemma}
Under COCO-Soft, we have \begin{align*}
    \sum_{t=1}^T f_t(\mathbf M_{t}) - f_t(\mathbf M^*) \leq&~ \frac{T V L^2}{4\alpha} +  \frac{2T\nu^2_{\max}}{V}+ \frac{\alpha D^2}{V} + \frac{LDT\epsilon}{\delta},\\
    \sum_{t=1}^T g_t(\mathbf M_{t}) \leq&~ Q_{\max} + \frac{TL^2(V+Q_{\max})}{2\alpha} - T\epsilon,\\
    h(\mathbf M_t) \leq&~ \frac{(\sqrt{T} + Q_{\max})LD + \alpha D^2}{\eta}, \forall t \in [T],
\end{align*} 
hold with probability at least $1-1/T.$ \label{lemma:ocowm}
\end{lemma}
Further, we establish the difference between $\mathbf M_{t}$ and $\mathbf M_{t+1}.$     
\begin{lemma}
Under COCO-Soft, we have
\begin{align*}
      \sum_{t=1}^T\|\mathbf M_{t+1} - \mathbf M_{t}\| \leq& \frac{TL(V+Q_{\max})}{2\alpha},\\
      \|\mathbf M_{t+1} - \mathbf M_{t}\| \leq& \frac{\sqrt{T}L(V+Q_{\max})}{2\alpha}, ~\forall t \in [T],
\end{align*}
hold with probability at least $1-1/T.$
\label{lemma:diff}
\end{lemma}

\subsubsection{Proof of Lemma \ref{lemma:ocowm}}
To prove Lemma \ref{lemma:ocowm}, we first introduce the following key lemma. 
\begin{lemma}
For any $\mathbf M \in \mathcal M,$ we have 
\begin{align*}
    &V\langle \mathbf M_{t+1} - \mathbf M_{t}, \nabla f_t(\mathbf M_t) \rangle + Q_t \langle \mathbf M_{t+1} - \mathbf M_{t}, \nabla g_t(\mathbf M_t) \rangle + \eta h^{+}(\mathbf M_{t+1}) + \alpha \|\mathbf M_{t+1} - \mathbf M_{t}\|^2 \\
    \leq& V\langle \mathbf M - \mathbf M_{t}, \nabla f_t(\mathbf M_t) \rangle + Q_t \langle \mathbf M-\mathbf M_t, \nabla g_t(\mathbf M_t) \rangle + \eta h^{+}(\mathbf M) + \alpha \|\mathbf M - \mathbf M_{t}\|^2 - \alpha \|\mathbf M - \mathbf M_{t+1}\|^2
\end{align*}\label{lem:pushback}
\end{lemma}
\begin{proof}
\normalfont
The proof is a direct application of Lemma \ref{lemma:tool}. 
Let $$F(\mathbf M) = V\langle \mathbf M - \mathbf M_{t}, \nabla f_t(\mathbf M) \rangle + Q_t \langle \mathbf M - \mathbf M_{t}, \nabla g_t(\mathbf M_t) \rangle + \eta h^{+}(\mathbf M).$$ Since $F(\mathbf M)+ \alpha \|\mathbf M - \mathbf M_{t}\|^2$ is $2\alpha$-strongly convex, the proof is completed according to Lemma \ref{lemma:tool}. 
\end{proof} 
By adding $Q_tg_t(\mathbf M_t)$ into the inequality in Lemma \ref{lem:pushback}, we have
for any $\mathbf M \in \mathcal M$ such that
\begin{align}
    &V\langle \mathbf M_{t+1} - \mathbf M_{t}, \nabla f_t(\mathbf M_t) \rangle + Q_t[g_t(\mathbf M_t) + \langle \mathbf M_{t+1} - \mathbf M_{t}, \nabla g_t(\mathbf M_t) \rangle] + \eta h^{+}(\mathbf M_{t+1}) + \alpha \|\mathbf M_{t+1} - \mathbf M_{t}\|^2 \nonumber\\
    \leq& V\langle \mathbf M - \mathbf M_{t}, \nabla f_t(\mathbf M_t) \rangle + Q_t[g_t(\mathbf M_t) + \langle \mathbf M-\mathbf M_t, \nabla g_t(\mathbf M_t) \rangle] + \eta h^{+}(\mathbf M) \nonumber\\
    &+ \alpha \|\mathbf M - \mathbf M_{t}\|^2 - \alpha \|\mathbf M - \mathbf M_{t+1}\|^2 \nonumber\\
    \leq& V\langle \mathbf M - \mathbf M_{t}, \nabla f_t(\mathbf M_t) \rangle + Q_tg_t(\mathbf M) + \eta h^{+}(\mathbf M) + \alpha \|\mathbf M - \mathbf M_{t}\|^2 - \alpha \|\mathbf M - \mathbf M_{t+1}\|^2, \label{eq:pushback}
\end{align}
where the last inequality holds because $g_t(\cdot)$ is convex. 
According to the virtual queue update as follows
\begin{align}
  Q_{t+1} = \left[Q_t + g_t(\mathbf M_t) + \langle \mathbf M_{t+1} - \mathbf M_{t}, \nabla g_t(\mathbf M_t) \rangle  + \epsilon \right]^{+}, \nonumber
\end{align}
we have
\begin{align}
  \frac{1}{2}Q^2_{t+1} - \frac{1}{2} Q^2_t \leq& \frac{1}{2}\left(Q_t + g_t(\mathbf M_t) + \langle \mathbf M_{t+1} - \mathbf M_{t}, \nabla g_t(\mathbf M_t) \rangle  + \epsilon \right)^{2} - \frac{1}{2} Q^2_t \nonumber\\
  \leq& Q_t(g_t(\mathbf M_t) + \langle \mathbf M_{t+1} - \mathbf M_{t}, \nabla g_t(\mathbf M_t) \rangle  + \epsilon) + 2\nu^2_{\max}, \label{eq:vq}
 \end{align}
where the first inequality holds because $(x^{+})^2 \leq x^2$ for any $x\in \mathbb R;$ the second inequality holds because of Assumptions \ref{assumption:set} and \ref{assumption:obj}. 
 
By combining the two inequalities in \eqref{eq:pushback} and \eqref{eq:vq}, we have 
\begin{align}
    &V\langle \mathbf M_{t+1} - \mathbf M_{t}, \nabla f_t(\mathbf M_t)\rangle +  \frac{1}{2}Q^2_{t+1} - \frac{1}{2} Q^2_t + \eta h^{+}(\mathbf M_{t+1}) + \alpha \|\mathbf M_{t+1} - \mathbf M_{t}\|^2 \nonumber\\
    \leq& V\langle \mathbf M - \mathbf M_{t}, \nabla f_t(\mathbf M_t) \rangle + Q_t(g_t(\mathbf M)+\epsilon) + 2\nu^2_{\max} + \eta h^{+}(\mathbf M) +\alpha \|\mathbf M - \mathbf M_{t}\|^2 \nonumber\\
    &- \alpha \|\mathbf M - \mathbf M_{t+1}\|^2. \label{eq:soft key}
\end{align}
Based on the key inequality in \eqref{eq:soft key}, we establish the regret and constraint violations in Lemma \ref{lemma:ocowm} in the following. 
\vspace{7pt}

{\noindent \bf Regret bound:}
We first define the following $\epsilon$-tightness problem of COCOwM as our baseline (with $\epsilon>0$): 
\begin{align*}
    \min_{\mathbf M\in \mathcal M}& ~ \sum_{t=1}^T f_t(\mathbf M) \\
    \text{s.t.} 
    & ~ h(\mathbf M) + \epsilon \leq 0,~ g_t(\mathbf M) + \epsilon \leq 0, \forall t \in [T].
\end{align*}
By adding $Vf_t(\mathbf M_t)$ on both sides of the inequality in \eqref{eq:soft key}, we have
\begin{align}
    &Vf_t(\mathbf M_t) + V\langle \mathbf M_{t+1} - \mathbf M_{t}, \nabla f_t(\mathbf M_t)\rangle  + \eta h^{+}(\mathbf M_{t+1}) +  \frac{1}{2}Q^2_{t+1} - \frac{1}{2} Q^2_t + \alpha \|\mathbf M_{t+1} - \mathbf M_{t}\|^2 \nonumber\\
    \leq& Vf_t(\mathbf M_t) + V\langle \mathbf M - \mathbf M_{t}, \nabla f_t(\mathbf M_t) \rangle + Q_t(g_t(\mathbf M)+\epsilon) + 2\nu^2_{\max} + \eta h^{+}(\mathbf M) \nonumber\\
    &+ \alpha \|\mathbf M - \mathbf M_{t}\|^2 - \alpha \|\mathbf M - \mathbf M_{t+1}\|^2 \nonumber \\
    \leq& Vf_t(\mathbf M) + Q_t(g_t(\mathbf M)+\epsilon) + 2\nu^2_{\max} + \alpha \|\mathbf M - \mathbf M_{t}\|^2 - \alpha \|\mathbf M - \mathbf M_{t+1}\|^2 \label{eq: soft selfbound per step}
\end{align}
where the last inequality holds because $f_t(\cdot)$ is convex $\forall t \in [T]$.
Let $\mathbf M = \mathbf M^{\epsilon}$ to be any feasible solution to the $\epsilon$-tightness problem such that $g_t(\mathbf M^{\epsilon}) + \epsilon \leq 0$. Taking the summation of \eqref{eq: soft selfbound per step} from time $t=1$ to $T,$ we have
\begin{align}
    &\sum_{t=1}^T f_t(\mathbf M_t) - f_t(\mathbf M^{\epsilon}) + \frac{\eta}{V} \sum_{t=1}^T h^{+}(\mathbf M_{t+1}) + \frac{Q^2_{T+1}}{2V}  - \frac{Q^2_1}{2V}  \nonumber\\
    \leq& \sum_{t=1}^T \left(\langle \mathbf M_{t} - \mathbf M_{t+1}, \nabla f_t(\mathbf M_t)\rangle - \frac{\alpha}{V} \|\mathbf M_{t+1} - \mathbf M_{t}\|^2\right) +  \frac{2T\nu^2_{\max}}{V}+ \frac{\alpha}{V} \|\mathbf M^{\epsilon} - \mathbf M_{1}\|^2 \nonumber\\
    \leq& \frac{T V L^2}{4\alpha} +  \frac{2T\nu^2_{\max}}{V}+ \frac{\alpha D^2}{V}, \nonumber
\end{align}
which, in conjugation with $Q_1=0,$ implies 
\begin{align}
    \sum_{t=1}^T f_t(\mathbf M_t) - f_t(\mathbf M^{\epsilon}) + \frac{\eta}{V} \sum_{t=1}^T h^{+}(\mathbf M_{t+1}) 
    \leq \frac{T V L^2}{4\alpha} +  \frac{2T\nu^2_{\max}}{V}+ \frac{\alpha D^2}{V}. \label{eq:regret}
\end{align}
Therefore, we have for any $\mathbf M \in \mathcal M$ such that
\begin{align*}
 \sum_{t=1}^T f_t(\mathbf M_t) - f_t(\mathbf M) =& \sum_{t=1}^T f_t(\mathbf M_t) - f_t(\mathbf M^{\epsilon}) + \sum_{t=1}^T f_t(\mathbf M^{\epsilon}) - f_t(\mathbf M)\\ 
\leq& \frac{T V L^2}{4\alpha} +  \frac{2T\nu^2_{\max}}{V}+ \frac{\alpha D^2}{V} + \sum_{t=1}^T f_t(\mathbf M^{\epsilon}) - f_t(\mathbf M) \\
\leq& \frac{T V L^2}{4\alpha} +  \frac{2T\nu^2_{\max}}{V}+ \frac{\alpha D^2}{V} + \frac{LDT\epsilon}{\delta}, 
\end{align*}
where the first inequality holds by \eqref{eq:regret}; the last inequality holds by Lemma \ref{lem:tightness}. The regret bound in Lemma \ref{lemma:ocowm} is completed by letting $\mathbf M = \mathbf M^*.$ 

\vspace{7pt}
{\noindent \bf Violation bound of $\mathcal V_g^{soft}(T)$:}
According to the virtual queue update, we have
\begin{align}
  Q_{t+1} = \left[Q_t + g_t(\mathbf M_t) + \langle \mathbf M_{t+1} - \mathbf M_{t}, \nabla g_t(\mathbf M_t) \rangle  + \epsilon \right]^{+}, \nonumber
\end{align}
we have
\begin{align}
  Q_{t+1} \geq Q_t + g_t(\mathbf M_t) + \langle \mathbf M_{t+1} - \mathbf M_{t}, \nabla g_t(\mathbf M_t) \rangle  + \epsilon. \nonumber
\end{align}
Take summation from time $t=1$ to $T,$ it implies
\begin{align}
    \mathcal V_g(T) := \sum_{t=1}^T g_t(\mathbf M_t) \leq& Q_{T+1} + L\sum_{t=1}^T \|\mathbf M_t-\mathbf M_{t+1}\|- T\epsilon.\nonumber
\end{align}
To establish the violation $\mathcal V_g(T),$ we need to bound the virtual queue $Q_{T+1}$ and the difference $\sum_{t=1}^T \|\mathbf M_t-\mathbf M_{t+1}\|.$ 

By invoking the multi-step Lyapunov drift analysis ($S=\sqrt{T}\log^5 T$ steps) Lemma 5 in \cite{YuNeeWei_17}, we establish the bound of the virtual queue in the following lemma.
\begin{lemma}
Let $\nu_{\max}=E+LD+\epsilon,$ $\theta= 2\nu_{\max} + \frac{2VLD+4\nu^2_{\max}}{\delta} + \frac{2\alpha D^2}{\delta S},$ and $Q_{\max} =\theta + S \nu_{\max} + \frac{16 S \nu_{\max}^2}{\delta} \log\frac{128 \nu_{\max}^2T^2}{\delta^2},$ with $S=\sqrt{T}\log^5 T.$ Under COCO-Soft, we have
\begin{align}
\mathbb E[Q_t] \leq Q_{\max}, ~\forall t \in [T],\nonumber
\end{align}
and 
\begin{align}
\mathbb P\left(Q_t \leq Q_{\max}\right) \geq 1-{1}/{T^2}, ~\forall t \in [T]. \nonumber
\end{align}\label{lemma:VQ}
\end{lemma}
\begin{proof}
\normalfont
Given the history $\mathcal H_t$ and according to the Lyapunov drift analysis, we have for any time $t$ such that
\begin{align}
    &\frac{1}{2}Q^2_{t+1} - \frac{1}{2} Q^2_t\nonumber\\
    \leq& VLD + Q_t(g_t(\mathbf M)+\epsilon) + 2\nu^2_{\max} + \alpha \|\mathbf M - \mathbf M_{t}\|^2 - \alpha \|\mathbf M - \mathbf M_{t+1}\|^2 \nonumber\\ 
    \leq& VLD - \frac{\delta}{2}Q_t + 2\nu^2_{\max} + \alpha \|\mathbf M - \mathbf M_{t}\|^2 - \alpha \|\mathbf M - \mathbf M_{t+1}\|^2 \nonumber
\end{align}
We study the multi-step drift as follows
\begin{align}
    &\mathbb E\left[\frac{1}{2}Q^2_{t+S} - \frac{1}{2} Q^2_t|\mathcal H_t\right]\nonumber\\
    =&\sum_{s=1}^S\mathbb E\left[\mathbb E\left[\frac{1}{2}Q^2_{t+s} - \frac{1}{2} Q^2_{t+s-1}|\mathcal H_{t+s-1}\right]|\mathcal H_t\right]\nonumber\\
    \leq&  - \frac{\delta}{2}\sum_{s=1}^S \mathbb E\left[Q_{t+s}|\mathcal H_t\right] + (VLD + 2\nu^2_{\max})S + \alpha \|\mathbf M - \mathbf M_{t}\|^2 
    \nonumber \\
    \leq& - \frac{\delta S}{2}Q_t + \frac{\delta S }{2}\nu_{\max} + (VLD+2\nu^2_{\max})S + \alpha D^2 \nonumber
\end{align}
Recall $Q_t \geq \theta := 2\nu_{\max} + \frac{2VLD+ 4\nu^2_{\max}}{\delta} + \frac{2\alpha D^2}{\delta S},$ we have
\begin{align}
    \mathbb E\left[\frac{1}{2}Q^2_{t+S} - \frac{1}{2} Q^2_t|\mathcal H_t\right]
    \leq& - \frac{\delta S}{4}Q_t, \nonumber
\end{align}
which implies conditional on $\mathcal H_t$
\begin{align}
    \mathbb E\left[Q^2_{t+S}|\mathcal H_t\right]
    \leq& Q^2_t - \frac{\delta S}{2}Q_t \nonumber \\
    \leq & (Q_t - \frac{\delta S}{4})^2. \nonumber
\end{align}
According to the Jenssn's inequality $(\mathbb E\left[Q_{t+S}|\mathcal H_t\right])^2 \leq \mathbb E\left[Q^2_{t+S}|\mathcal H_t\right],$ we have
\begin{align}
    \mathbb E\left[Q_{t+S} - Q_t|\mathcal H_t\right]
    \leq&  - \frac{\delta S}{4}, \nonumber
\end{align}
which, in conjugation with Lemma \ref{lemma:drift}, implies that $$\mathbb E[Q_t] \leq \theta + S \nu_{\max} + \frac{16 S \nu_{\max}^2}{\delta} \log\frac{128 \nu_{\max}^2}{\delta^2}.$$
and 
$$\mathbb P\left(Q_t > \theta + S \nu_{\max} + \frac{16 S \nu_{\max}^2}{\delta} \left(\log\frac{128 \nu_{\max}^2}{\delta^2}+\log \frac{1}{p}\right)\right) < p.$$
Note the definition of $Q_{\max}$ and let $p=1/T^2$ prove the lemma. 
\end{proof}
Therefore, we have 
\begin{align}
    \sum_{t=1}^T g_t(\mathbf M_t) \leq& Q_{T+1} + L\sum_{t=1}^T \|\mathbf M_t-\mathbf M_{t+1}\|- T\epsilon\\
    \leq& Q_{\max} + \frac{TL^2(V+Q_{\max})}{2\alpha} - T\epsilon,
\end{align}
with the high probability $1-1/T,$ where the second inequality holds because of Lemma \ref{lemma:diff}.
\vspace{7pt}

{\noindent \bf Violation bound of $V_h(t)$}:
Let $\mathbf M$ be any feasible point such that $h^{+}(\mathbf M) = 0$ in Lemma \ref{lem:pushback}, we have
\begin{align*}
    &V\langle \mathbf M_{t+1} - \mathbf M_{t}, \nabla f_t(\mathbf M_t) \rangle + Q_t \langle \mathbf M_{t+1} - \mathbf M_{t}, \nabla g_t(\mathbf M_t) \rangle + \eta h^{+}(\mathbf M_{t+1}) + \alpha \|\mathbf M_{t+1} - \mathbf M_{t}\|^2 \\
    \leq& V\langle \mathbf M - \mathbf M_{t}, \nabla f_t(\mathbf M_t) \rangle + Q_t \langle \mathbf M-\mathbf M_t, \nabla g_t(\mathbf M_t) \rangle + \alpha \|\mathbf M - \mathbf M_{t}\|^2 - \alpha \|\mathbf M - \mathbf M_{t+1}\|^2
\end{align*}
which implies that
\begin{align*}
    \eta h^{+}(\mathbf M_{t+1})  
    \leq V\langle \mathbf M - \mathbf M_{t+1}, \nabla f_t(\mathbf M_t) \rangle + Q_t \langle \mathbf M-\mathbf M_{t+1}, \nabla g_t(\mathbf M_t) \rangle + \alpha \|\mathbf M - \mathbf M_{t}\|^2. 
\end{align*}
Therefore, we have 
\begin{align*}
     h^{+}(\mathbf M_{t+1})  
    \leq& \frac{(V+Q_{\max})LD+\alpha D^2}{\eta}.
\end{align*}

\subsubsection{Proof of Lemma \ref{lemma:diff}}
Recall $Q_t \leq Q_{\max}, \forall t \in [T]$ hold with the high probability in Lemma \ref{lemma:VQ}.
\begin{itemize}
    \item We study the term of the total difference $\sum_{t=1}^T\|\mathbf M_{t+1}-\mathbf M_t\|.$ Let $\mathbf M=\mathbf M_t$ in Lemma \ref{lem:pushback}, we have 
\begin{align*}
    &V\langle \mathbf M_{t+1} - \mathbf M_{t}, \nabla f_t(\mathbf M_t) \rangle + Q_t \langle \mathbf M_{t+1} - \mathbf M_{t}, \nabla g_t(\mathbf M_t) \rangle + 2\alpha \|\mathbf M_{t+1} - \mathbf M_{t}\|^2 \\
    \leq&  \eta h^{+}(\mathbf M_t) - \eta h^{+}(\mathbf M_{t+1})
\end{align*}
which implies that
\begin{align}
    &-L(V+Q_{\max})\|\mathbf M_{t+1} - \mathbf M_{t}\| + 2\alpha \|\mathbf M_{t+1} - \mathbf M_{t}\|^2 \nonumber\\
    &\leq  \eta h^{+}(\mathbf M_t) - \eta h^{+}(\mathbf M_{t+1}) \label{term: big eta}
\end{align}
holds with the high probability at least $1-1/T^2$ by Lemma \ref{lemma:VQ} ($Q_t \leq Q_{\max}$). Since $h^{+}(\mathbf M_1) = 0$ and $\sum_{t=1}^T  [h^{+}(\mathbf M_t) -  h^{+}(\mathbf M_{t+1})] \leq 0,$ we have 
\begin{align*}
   2\alpha \sum_{t=1}^T \|\mathbf M_{t+1} - \mathbf M_{t}\|^2
    \leq L(V+Q_{\max}) \sum_{t=1}^T \| \mathbf M_{t+1} - \mathbf M_{t}\|,
\end{align*} with the high probability at least $1-1/T$ according to the union bound. It implies that 
\begin{align*}
   \frac{2\alpha}{T} \left(\sum_{t=1}^T \|\mathbf M_{t+1} - \mathbf M_{t}\| \right)^2
    \leq L(V+Q_{\max}) \sum_{t=1}^T \|\mathbf M_{t+1} - \mathbf M_{t}\|,
\end{align*}
and we have 
\begin{align*}
   \sum_{t=1}^T \|\mathbf M_{t+1} - \mathbf M_{t}\| 
    \leq \frac{TL(V+Q_{\max})}{2\alpha}
\end{align*}
hold with the probability at least $1-1/T.$
\item We study the term of the individual difference $\|\mathbf M_{t} - \mathbf M_{t-1}\|,$ and we have 
\begin{align*}
    -L(V+Q_{\max})\|\mathbf M_{t+1} - \mathbf M_{t}\| + 2\alpha \|\mathbf M_{t+1} - \mathbf M_{t}\|^2 
    \leq  \eta h^{+}(\mathbf M_t) - \eta h^{+}(\mathbf M_{t+1})
\end{align*}
where $h^{+}(u_1) = 0.$ For any $\tau \in [T],$ we have
\begin{align*}
    \frac{2\alpha}{L(V+Q_{\max})} \sum_{t=1}^{\tau}\|\mathbf M_{t+1} - \mathbf M_{t}\|^2 -\sum_{t=1}^{\tau}\|\mathbf M_{t+1} - \mathbf M_{t}\|
    \leq  0
\end{align*}
hold with the probability at least $1-1/T.$
Therefore, we have
\begin{align*}
      \frac{2\alpha}{L(V+Q_{\max})} \|\mathbf M_{\tau+1} - \mathbf M_{\tau}\|^2 -\|\mathbf M_{\tau+1} - u_{\tau}\|
    \leq   \frac{\tau L(V+Q_{\max})}{8\alpha}, ~\forall \tau \in [T] 
\end{align*}
and 
\begin{align}
      \|\mathbf M_{\tau+1} - \mathbf M_{\tau}\| \leq   \frac{\sqrt{T}L(V+Q_{\max})}{2\alpha}, ~\forall \tau \in [T]\label{term: big stable}
\end{align}
hold with the probability at least $1-1/T.$
\end{itemize}

\subsection{Specifying Parameters of COCO-Soft with $\tilde{c}_t(\mathbf M),\tilde{d}_t(\mathbf M),$ and $\tilde{l}(\mathbf M).$}\label{sec: para soft}
We verify the assumptions in Theorem \ref{thm:ocowm} and specify $D, L, E$ and $\xi$ in these assumptions as follows
\begin{align*}
    D:=& \frac{2a}{\rho} \\
    L:=& \frac{ 2C_0W(1+\kappa)(\kappa^2 + Ha\kappa_B\kappa^2)}{\rho} + \frac{2aC_0W}{\rho} \\
    E:=& C_0 D + C_1D \\
    \xi:=& \frac{\delta}{2}
\end{align*}

\begin{itemize}
    \item For Assumption \ref{assumption:set}, we have \begin{align*}
        \|\mathbf M - \mathbf M'\| \leq \|\mathbf M\| + \|\mathbf M'\| 
        \leq 2 a \sum_{i=1}^H (1-\rho)^i
        \leq \frac{2a}{\rho} = D.
    \end{align*} 
    \item For Assumption \ref{assumption:obj}, given any policies $\pi$ and $\pi'$ associated with $\mathbf M$ and $\mathbf M'$, respectively, we have \begin{align*}
        |\tilde c_t(\mathbf M) - \tilde c_t(\mathbf M')| \leq& C_0\left(\|\tilde x^{\pi}_t - \tilde x^{\pi'}_t\| + \|\tilde u^{\pi}_t - \tilde u^{\pi'}_t\|\right) \\
        \leq& C_0\left(\|\tilde x^{\pi}_t \| + \|\tilde u^{\pi}_t\| + \|\tilde x^{\pi'}_t \| + \|\tilde u^{\pi'}_t\|\right)\\
        \leq& 2C_0\left(\|\tilde x^{\pi}_t \| + \kappa\|\tilde x^{\pi}_t\| + \|\sum_{i=1}^H M^{[i]} w_{t-i}\| \right)\\
        \leq& 2C_0(1+\kappa) \sum_{i=1}^{2H} \| \Psi_{t,i}^\pi w_{t-i}\| + \frac{2aC_0W}{\rho}\\
        \leq& \frac{ 2C_0W(1+\kappa)(\kappa^2 + Ha\kappa_B\kappa^2)}{\rho} + \frac{2aC_0W}{\rho} = L.
    \end{align*} 
Similarly, we have
    \begin{align*}
        |\tilde d_t(\mathbf M) - \tilde d_t(\mathbf M')|
        \leq L, ~~ \|\tilde l(\mathbf M) - \tilde l(\mathbf M')\|
        \leq L. 
    \end{align*} 
    \item For Assumption \ref{assumption:obj}, we have for a policy $\pi(K, \mathbf M)$ such that 
    \begin{align*}
        |\tilde d_t(\mathbf M)|
        &\leq |d_t(x_t^\pi, u_t^\pi)| + |\tilde d_t(\mathbf M)- d_t(x_t^\pi, u_t^\pi)| \leq C_0 D + C_1D,\\
        |\tilde l(\mathbf M)|
        &\leq |l(x_t^\pi, u_t^\pi)| + |\tilde l(\mathbf M)- l(x_t^\pi, u_t^\pi)| \leq C_0 D + C_1D. 
    \end{align*} 
    
    \item For Assumption \ref{assumption:slater}, we have for a policy $\pi(K, \mathbf M)$ such that  
\begin{align*}
     d_t(\tilde x^{\pi}_t, \tilde u^{\pi}_t) =&  d_t(\tilde x^{\pi}_t, \tilde u^{\pi}_t) - d_t(x^{\pi}_t, u^{\pi}_t) + d_t(x^{\pi}_t, u^{\pi}_t) \\
    \leq&  d_t(\tilde x^{\pi}_t, \tilde u^{\pi}_t) - d_t(x^{\pi}_t, u^{\pi}_t)  - \delta\\
    \leq& C_0 (\|\tilde x_{t}^{\pi} - x_{t}^{\pi}\| + \|\tilde u_{t}^{\pi} - u_{t}^{\pi}\|) - \delta\\
    \leq& 2C_0 \kappa(1-\rho)^H \frac{W\kappa^2(\kappa^2 + Ha\kappa_B\kappa^2)}{\rho(1-\kappa^2(1-\rho)^{H+1})} - \delta\\
    \leq& \bar \epsilon - \delta,
\end{align*}
which is smaller than $-\frac{\delta}{2}$ when $\bar \epsilon = \frac{1}{T} \leq \frac{\delta}{2}.$
Similarly, we have $l(\tilde x^{\pi}_t, \tilde u^{\pi}_t) 
    \leq - \frac{\delta}{2}.$
\end{itemize}

\subsection{Zero Anytime Violation with Projection-based Method} \label{app: good anytime}
We present how to achieve zero anytime violation as in \cite{LiDasLi_21} with the projection-based method in Remark \ref{remark: good anytime}. We illustrate the key changes for anytime violation because the regret and cumulative violation follow the exact analysis above. 

For \eqref{term: big eta}, we have the following inequality hold under the projection-based method 
\begin{align*}
    -L(V+Q_{\max})\|\mathbf M_{t+1} - \mathbf M_{t}\| + 2\alpha \|\mathbf M_{t+1} - \mathbf M_{t}\|^2 \leq 0
\end{align*}
which implies 
\begin{align}
      \|\mathbf M_{\tau+1} - \mathbf M_{\tau}\| \leq   \frac{L(V+Q_{\max})}{2\alpha}, ~\forall \tau \in [T], \label{term refined stable}
\end{align}
hold with the probability at least $1-1/T.$ Compared to \eqref{term: big stable}, we have a refined stability term, which is the key to reducing the anytime violation. 
To achieve zero anytime violation, we impose a slight pessimistic constraint $\Tilde{h}$ where $h(\mathbf M)-\Tilde{h}(\mathbf M) \leq \beta/\sqrt{T}, \forall \mathbf M \in \mathcal M$ such that 
\begin{align}
\mathcal V_h(t) \leq& |h(\mathbf M_{t-H:t}) - h(\mathbf M_{t})| + \Tilde{h}(\mathbf M_{t}) + h(\mathbf M_{t}) - \Tilde{h}(\mathbf M_{t}) \nonumber\\
\leq&\frac{L^2H^2(V+Q_{\max})}{2\alpha} + \frac{(\sqrt{T} + Q_{\max})LD + \alpha D^2}{\eta}-\frac{\beta}{\sqrt{T}} \nonumber\\
\leq& 0 \nonumber
\end{align}
where the second inequality holds by using \eqref{term refined stable}; the last inequality holds by choosing a proper constant $\beta.$ Note we impose a tight constraint only induces additional $O(\sqrt{T})$ regret according to Lemma \ref{lem:tightness}. Therefore, we have verified the zero anytime violation is achievable with projection-based method in Remark \ref{remark: good anytime}.

\section{Proof of COCA with COCO-Hard Solver}\label{app:hard}
Let the learning rates be $V = 1, \gamma = T^{2/3}, \eta = T^{3/2},$ and $\alpha = T^{2/3}.$
\subsection{Performance of COCA with COCO-Hard Solver in Theorem \ref{thm:occa-hard}}
According to the roadmap in Section \ref{sec:template}, we plug the regret and constraint violation of COCO-Hard solver in Theorem \ref{thm:ocowm-hard} to justify Theorem \ref{thm:occa-hard}. 

{\bf Regret analysis:} we have the following regret decomposition 
\begin{align}
\mathcal R(T) = & \sum_{t=1}^T c_t(x_t^{\pi}, u_t^{\pi}) - \sum_{t=1}^T c_t(x_t^{K^*}, u_t^{K^*}) \nonumber\\
=& \sum_{t=1}^T \left[c_t(x_t^{\pi}, u_t^{\pi}) -  c_t(\tilde x_t^{\pi}, \tilde u_t^{\pi})\right] 
+ \sum_{t=1}^T  c_t(\tilde x_t^{\pi}, \tilde u_t^{\pi}) - \min_{\pi \in \tilde{\Omega} \bigcap
 \mathcal E}\sum_{t=1}^T c_t(\tilde x_t^{\pi}, \tilde u_t^{\pi}) \nonumber\\
&+\min_{\pi \in \tilde{\Omega} \bigcap
 \mathcal E}\sum_{t=1}^T c_t(\tilde x_t^{\pi}, \tilde u_t^{\pi}) - \sum_{t=1}^T c_t(x^{K^*}_t, u^{K^*}_t) \nonumber\\
 \leq & 2(L\sqrt{LD}H^2+L^2+D^2+D) T^{\frac{2}{3}} + LH+2  \nonumber
\end{align} 
where the inequality holds because of the regret in Theorem \ref{thm:ocowm-hard} and the approximated error in Lemma \ref{lem:fun-diff} and the representation ability of DAC policy in Lemma \ref{lem:dac-rep}. 

{\bf Anytime violation of $l$ function:} we have the following decomposition for the constraint function $l$
\begin{align*}
\mathcal V_l(t)
=& l(x_t^\pi, u_t^\pi) - l(\tilde x_t^\pi, \tilde u_t^\pi) + l(\tilde x_t^\pi, \tilde u_t^\pi)\\
\leq& 1/T + (L^2DH^2\log^2 T+2LD+D^2)/T^{\frac{1}{3}} 
\end{align*}
where the inequality holds because of the anytime violation in Theorem \ref{thm:ocowm-hard} and the approximated error in Lemma \ref{lem:fun-diff}.

{\bf Cumulative hard violation of $d_t$ function:} we have the following decomposition for $d_t$ function
\begin{align*}
\mathcal V_d^{hard}(T)
\leq& \sum_{t=1}^T [d_t^{+}(x_t^\pi, u_t^\pi) - d_t^{+}(\tilde x_t^\pi, \tilde u_t^\pi)] + \sum_{t=1}^T d_t^{+}(\tilde x_t^\pi, \tilde u_t^\pi)\\
\leq&  1 + 2L\sqrt{LD}H^2+L^2+D^2+D) T^{\frac{2}{3}}
\end{align*}
where the inequality holds because of the soft violation in Theorem \ref{thm:ocowm-hard} and the approximated error in Lemma \ref{lem:fun-diff}. 

\subsection{Performance of COCO-Hard Solver in Theorem \ref{thm:ocowm-hard}}
Note the term of $H$-step difference can be bounded as follows
\begin{align}
\sum_{i=1}^{H} \|\mathbf M_{t-i} - \mathbf M_{t}\| =& \sum_{i=1}^{H} \sum_{j=1}^{i}  \|\mathbf M_{t-j} - \mathbf M_{t-j+1}\|. \nonumber
\end{align}
Based on Lemma \ref{lemma:diff-hard}, we have
\begin{align*}
\sum_{t=1}^T|f_t(\mathbf M_{t-H:t}) - f_t(\mathbf M_{t})|  
\leq&  LH^2(2\sqrt{LD}T^{\frac{2}{3}} + D\sqrt{T}), \\
\sum_{t=1}^T|g_t^{+}(\mathbf M_{t-H:t}) - g_t^{+}(\mathbf M_{t})| \leq& LH^2(2\sqrt{LD}T^{\frac{2}{3}} + D\sqrt{T}),\\
|h(\mathbf M_{t-H:t}) - h(\mathbf M_{t})| \leq& L^2DH^2\log^2 T/T^{\frac{1}{3}}.
\end{align*}
Finally, in conjugation with Lemma \ref{lemma:ocowm-hard}, we prove Theorem \ref{thm:ocowm-hard} 
as follows:
\begin{align*}
\mathcal R_f(T) \leq&  \sum_{t=1}^T|f_t(\mathbf M_{t-H:t}) - f_t(\mathbf M_{t})| +
\sum_{t=1}^T f_t(\mathbf M_{t}) - f_t(\mathbf M^*)\\ 
\leq& (2L\sqrt{LD}H^2+L^2+D^2+D) T^{\frac{2}{3}},\\
\mathcal V_g(T) \leq&  \sum_{t=1}^T|g_t^{+}(\mathbf M_{t-H:t}) - g_t^{+}(\mathbf M_{t})| +
\sum_{t=1}^T g_t^{+}(\mathbf M_{t}) \\
\leq& (4L\sqrt{LD}H^2+3LD+D) T^{\frac{2}{3}}+D^2,\\
\mathcal V_h(T) \leq& |h(\mathbf M_{t-H:t}) - h(\mathbf M_{t})| + h(\mathbf M_{t})\\
\leq& (L^2DH^2\log^2 T+2LD+D^2)/T^{\frac{1}{3}}.
\end{align*}

\begin{lemma}
Under COCO-Hard Solver, we have \begin{align*}
    \sum_{t=1}^T f_t(\mathbf M_{t}) - f_t(\mathbf M^*) \leq& T^{\frac{2}{3}}\left(L^2 + D^2\right),\\
    \sum_{t=1}^T g_t^{+}(\mathbf M_{t}) \leq (2L^{\frac{3}{2}}D^{\frac{1}{2}} + 3LD)T^{\frac{2}{3}} + D^2,& ~h^{+}(\mathbf M_t) \leq \frac{2LD+D^2}{\sqrt{T}}, \forall t \in [T],
\end{align*}
hold with probability at least $1-1/T.$ \label{lemma:ocowm-hard}
\end{lemma}

Further, we establish the difference between $\mathbf M_{t+1}$ and $\mathbf M_{t}.$ 

\begin{lemma}
Under COCO-Hard Solver, we have \begin{align*}
      \sum_{t=1}^T\|\mathbf M_{t+1} - \mathbf M_{t}\| \leq&  2\sqrt{LD}T^{\frac{2}{3}} + D\sqrt{T},\\
      \|\mathbf M_{t+1} - \mathbf M_{t}\| \leq& \frac{LD}{T^\frac{1}{3}}, ~\forall t \in [T].
\end{align*}\label{lemma:diff-hard}
\end{lemma}

\subsubsection{Proof of Lemma \ref{lemma:ocowm-hard}}
To prove Lemma \ref{lemma:ocowm-hard}, we first introduce the following key lemma.
\begin{lemma}
For any $\mathbf M \in \mathcal M,$ we have 
\begin{align}
    &V\langle \mathbf M_{t+1} - \mathbf M_{t}, \nabla f_t(\mathbf M_t) \rangle + \gamma g_t^{+}(\mathbf M)+ \eta h^{+}(\mathbf M_{t+1}) + \alpha \|\mathbf M_{t+1} - \mathbf M_{t}\|^2 \nonumber\\
    \leq& V\langle \mathbf M - \mathbf M_{t}, \nabla f_t(\mathbf M_t) \rangle  + \gamma g_t^{+}(\mathbf M) + \eta h^{+}(\mathbf M) + \alpha \|\mathbf M - \mathbf M_{t}\|^2 - \alpha \|\mathbf M - \mathbf M_{t+1}\|^2 \label{eq:pushback-hard}
\end{align}\label{lem:pushback-hard}
\end{lemma}
\begin{proof}
\normalfont
The proof is a direct application of Lemma \ref{lemma:tool}. 
Let $$F(\mathbf M) = V\langle \mathbf M - \mathbf M_{t}, \nabla f_t(\mathbf M) \rangle + \gamma g^{+}_t(\mathbf M) + \eta h^{+}(\mathbf M).$$ Since $F(\mathbf M)+ \alpha \|\mathbf M - \mathbf M_{t}\|^2$ is $2\alpha$-strongly convex, the proof is completed by Lemma \ref{lemma:tool}. 
\end{proof} 
By adding $Vf_t(\mathbf M_t)$ on both sides of the inequality in \eqref{eq:pushback-hard}, we have
\begin{align}
    &Vf_t(\mathbf M_t) + V\langle \mathbf M_{t+1} - \mathbf M_{t}, \nabla f_t(\mathbf M_t)\rangle  + \gamma g^{+}_t(\mathbf M_{t+1})  + \eta h^{+}(\mathbf M_{t+1}) + \alpha \|\mathbf M_{t+1} - \mathbf M_{t}\|^2 \nonumber\\
    \leq& V f_t(\mathbf M_t) + V \langle \mathbf M - \mathbf M_{t}, \nabla f_t(\mathbf M_t) \rangle  + \gamma g^{+}_t(\mathbf M) + \eta_t h^{+}(\mathbf M)+ \alpha \|\mathbf M - \mathbf M_{t}\|^2 - \alpha \|\mathbf M - \mathbf M_{t+1}\|^2 \nonumber \\
    \leq& V f_t(\mathbf M) + \gamma g^{+}_t(\mathbf M) + \eta h^{+}(\mathbf M)+ \alpha \|\mathbf M - \mathbf M_{t}\|^2 - \alpha \|\mathbf M - \mathbf M_{t+1}\|^2 \label{eq: hard key}
\end{align}
where the last inequality holds because $f_t(\cdot)$ is convex $\forall t \in [T]$.

Based on the key inequality in \eqref{eq: hard key}, we establish the regret and constraint violations in Lemma \ref{lemma:ocowm-hard} in the following. 

{\noindent \bf Regret Analysis:}
Recall our baseline is the following offline COCOwM: 
\begin{align*}
    \min_{\mathbf M\in \mathcal M}& ~ \sum_{t=1}^T f_t(\mathbf M) \\
    \text{s.t.} 
    & ~ h(\mathbf M)  \leq 0,~ g_t(\mathbf M)  \leq 0, \forall t \in [T].
\end{align*}
Let $\mathbf M$ be any feasible solution to the offline problem such that $h(\mathbf M) \leq 0$ and $g_t(\mathbf M) \leq 0, \forall t\in[T]$. Taking the summation of \eqref{eq: hard key} from time $t=1$ to $T,$ we have
\begin{align}
    &\sum_{t=1}^T f_t(\mathbf M_t) - f_t(\mathbf M) + \sum_{t=1}^T \frac{\gamma g^{+}_t(\mathbf M_{t+1})}{V}  + \sum_{t=1}^T\frac{\eta h^{+}(\mathbf M_{t+1})}{V} \leq  \frac{T V L^2}{4\alpha} + \frac{\alpha D^2}{V}, \nonumber
\end{align}
which implies 
\begin{align}
    \sum_{t=1}^T f_t(\mathbf M_t) - f_t(\mathbf M) 
    \leq  T^{\frac{2}{3}}\left(L^2 + D^2\right). \nonumber
\end{align}

\vspace{7pt}
{\noindent \bf Hard Violation of $\mathcal V_g^{hard}(T)$:}
From \eqref{eq: hard key}, we have
\begin{align}
    & \gamma g^{+}_t(\mathbf M_{t+1}) + \eta h^{+}(\mathbf M_{t+1}) + \alpha \|\mathbf M_{t+1} - \mathbf M_{t}\|^2 \nonumber\\
    \leq& V \left(f_t(\mathbf M)-f_t(\mathbf M_t\right)-\langle \mathbf M_{t+1} - \mathbf M_{t}, \nabla f_t(\mathbf M_t)\rangle)+ \alpha \|\mathbf M - \mathbf M_{t}\|^2 - \alpha \|\mathbf M - \mathbf M_{t+1}\|^2 \nonumber\\
    \leq& 2VLD + \alpha \|\mathbf M - \mathbf M_{t}\|^2 - \alpha \|\mathbf M - \mathbf M_{t+1}\|^2 \label{eq: key-hard}
\end{align}
which implies 
\begin{align*}
    \sum_{t=1}^T g^{+}_t(\mathbf M_{t+1})
    \leq&  \frac{2TVLD}{\gamma} + \frac{\alpha D^2}{\gamma}  \leq 2LDT^{\frac{1}{3}}+ D^2, \nonumber\\
    \sum_{t=1}^T \|\mathbf M_{t+1} - \mathbf M_{t}\|^2 \leq&  \frac{2TVLD}{\alpha} + D^2
    \leq 2LDT^{\frac{1}{3}}+ D^2.
\end{align*}
Therefore, we have
\begin{align*}
    \sum_{t=1}^T g^{+}_t(\mathbf M_{t}) - g^{+}_t(\mathbf M_{t+1}) \leq& \sum_{t=1}^T \langle \nabla g_t^{+}(\mathbf M_t), \mathbf M_{t+1} - \mathbf M_t\rangle \\
    \leq& L\sum_{t=1}^T \|\mathbf M_{t+1} - \mathbf M_t\| \leq 2L^{\frac{3}{2}}D^{\frac{1}{2}}T^{\frac{2}{3}} + LD\sqrt{T}.
\end{align*}
The hard violation is 
\begin{align*}
    \sum_{t=1}^T g^{+}_t(\mathbf M_{t}) =& \sum_{t=1}^T g^{+}_t(\mathbf M_{t}) - g^{+}_t(\mathbf M_{t+1}) + \sum_{t=1}^Tg^{+}_t(\mathbf M_{t+1})\leq (2L^{\frac{3}{2}}D^{\frac{1}{2}} + 3LD)T^{\frac{2}{3}} + D^2. 
\end{align*}

{\noindent \bf Anytime Violation of $V_h(t)$}:
From \eqref{eq: hard key}, we have
\begin{align*}
    \sum_{t=1}^T h^{+}(\mathbf M_{t+1})
    \leq& \frac{2TVLD}{\eta} + \frac{\alpha D^2}{\eta}   \leq \frac{2LD+D^2}{\sqrt{T}}, 
\end{align*}
which implies 
\begin{align*}
     h^{+}(\mathbf M_{t})
    \leq \frac{2LD+D^2}{\sqrt{T}}.
\end{align*}

\subsubsection{Proof of Lemma \ref{lemma:diff-hard}}
The first part of Lemma \ref{lemma:diff-hard} has been proved above. Let's restate it for the purpose of completeness. From \eqref{eq: key-hard} in the proof of Lemma \ref{lemma:ocowm-hard}, we have \begin{align*}
    \sum_{t=1}^T \|\mathbf M_{t+1} - \mathbf M_{t}\|^2 \leq 2LDT^{\frac{1}{3}} + D^2,
\end{align*}
which implies by Cauchy-Schwarz inequality
\begin{align*}
      \sum_{t=1}^T\|\mathbf M_{t+1} - \mathbf M_{t}\| \leq&  2L^{\frac{1}{2}}D^{\frac{1}{2}}T^{\frac{2}{3}} + DT^{\frac{1}{2}}.
\end{align*}
Let $\mathbf M = \mathbf M_t$ in \eqref{eq: key-hard}, we have
\begin{align}
    \alpha \|\mathbf M_{t+1} - \mathbf M_{t}\|^2 
    \leq VLD, \nonumber
\end{align}
which implies $$\|\mathbf M_{t+1} - \mathbf M_{t}\| \leq \frac{LD}{T^\frac{1}{3}}.$$

\subsection{Specifying Parameters of COCO-Hard with $\tilde{c}_t(\mathbf M),\tilde{d}_t(\mathbf M),$ and $\tilde{l}(\mathbf M).$}
The parameters $D, L, E,$ and $\xi$ in Assumptions \ref{assumption:set}-\ref{assumption:slater} are exactly the same as in COCO-Soft as follows
\begin{align*}
    D:=& \frac{2a}{\rho},~
    L:= \frac{ 2CW(1+\kappa)(\kappa^2 + Ha\kappa_B\kappa^2)}{\rho} + \frac{2aCW}{\rho} \\
    E:=& C_0 D + C_1D.
\end{align*}

\subsection{Achievable Trade-off Via Learning Rates Tuning} \label{app: lr tuning}
We illustrate how to tune the learning rates in COCO-Hard to achieve the trade-off between regret and violation for COCA with COCO-Hard in Remark \ref{remark: lr tuning}. 

We summarize the regret and constraint violation for COCA with COCO-Hard 
\begin{align*}
    \mathcal R(T) =& \Tilde{O}(\frac{TV}{\alpha} + \frac{\alpha}{V} + T\sqrt{\frac{V}{\alpha}} + \sqrt{T}),\\
    \mathcal V^{hard}_d(T) =& \Tilde{O}(\frac{TV}{\gamma} + \frac{\alpha}{\gamma} + T\sqrt{\frac{V}{\alpha}} + \sqrt{T}),\\
    \mathcal V_l(t) =& \Tilde{O}(\frac{TV}{\eta}+\frac{\alpha}{\gamma}+\sqrt{\frac{V}{\alpha}}).
\end{align*}
Let the learning rates be $V = 1, \gamma = T^{2c}, \eta = T^{3/2},$ and $\alpha = T^{c}$ where $c \in [0.5, 1),$ we establish the trade-off in Remark \ref{remark: lr tuning} such that
\begin{align*}
    \mathcal R(T) = \Tilde{O}(T^{c} + T^{1-\frac{c}{2}} + \sqrt{T}),~
    \mathcal V^{hard}_d(T) = \Tilde{O}(T^{1-\frac{c}{2}} + \sqrt{T}),~
    \mathcal V_l(t) = \Tilde{O}(T^{-\frac{c}{2}}).
\end{align*}

\section{Proof of COCA with COCO-Best2Worlds Solver}\label{app:best}
Similar to COCA with COCO-Soft and with COCO-Hard, we prove COCA with COCO-Best2Worlds. According to the roadmap in Section \ref{sec:template}, we plug the regret and constraint violation of COCO-Best2Worlds in Theorem \ref{thm:ocowm-best} to justify Theorem \ref{thm:occa-best}.  

Let the learning rate be $V_t = \sqrt{t}, \gamma_t = t^{1.5}, \eta_t = t^{1.5}, \alpha_t = \lambda t^{1.5}/2$ and $\epsilon = \log^3 T/\sqrt{T}.$ We again define the parameters in this section as follows: $\nu_{\max} := E+LD+\epsilon, \theta := \frac{2(LD+\nu_{\max}^2)}{\delta},$ and
$Q_{\max}:=\theta + \nu_{\max} + \frac{16 \nu_{\max}^2}{\delta} \left(\log\frac{128 \nu_{\max}^2}{\delta^2}+2\log T\right).$

\subsection{Performance of COCA with COCO-Best2Worlds Solver in Theorem \ref{thm:occa-best}}

{\bf Regret analysis:} we have the following regret decomposition 
\begin{align}
\mathcal R(T) 
=& \sum_{t=1}^T \left[c_t(x_t^{\pi}, u_t^{\pi}) -  c_t(\tilde x_t^{\pi}, \tilde u_t^{\pi})\right] 
+ \sum_{t=1}^T  c_t(\tilde x_t^{\pi}, \tilde u_t^{\pi}) - \min_{\pi \in \tilde{\Omega} \bigcap
 \mathcal E}\sum_{t=1}^T c_t(\tilde x_t^{\pi}, \tilde u_t^{\pi}) \nonumber\\
&+\min_{\pi \in \tilde{\Omega} \bigcap
 \mathcal E}\sum_{t=1}^T c_t(\tilde x_t^{\pi}, \tilde u_t^{\pi}) - \sum_{t=1}^T c_t(x^{K^*}_t, u^{K^*}_t) \nonumber\\
 \leq& 1 + LH^2\sqrt{\frac{(2LD+\nu_{\max}^2)T\log T}{\lambda}} + \frac{L^2\log T}{\lambda} + 2\nu_{\max}^2 \sqrt{T} + \frac{\epsilon TLD}{\delta} + LD \nonumber\\
 =& O(\sqrt{T}\log^3 T)\nonumber
\end{align}
where the inequality holds because of the regret in Theorem \ref{thm:ocowm-best} and the approximated error in Lemma \ref{lem:fun-diff} and the representation ability of DAC policy in Lemma \ref{lem:dac-rep}. The order-wise result is established by substituting the learning rates.

{\bf Cumulative soft violation of $d_t$ function:} we have the following decomposition for the constraint function $d_t$ 
\begin{align*}
\mathcal V_d^{soft}(T)
\leq& \sum_{t=1}^T [d_t(x_t^\pi, u_t^\pi) - d_t(\tilde x_t^\pi, \tilde u_t^\pi)] + \sum_{t=1}^T d_t(\tilde x_t^\pi, \tilde u_t^\pi)\\
\leq& 1 + 3LH^2\sqrt{\frac{(2LD+\nu_{\max}^2)T\log T}{\lambda}} -T\epsilon \\
=& O(1)
\end{align*}
where the inequality holds because of the soft violation in Theorem \ref{thm:ocowm-best} and the approximated error in Lemma \ref{lem:fun-diff}. The order-wise result is established by substituting the learning rates and when $\log T \geq \frac{18L(LD+\nu_{\max}^2)}{\lambda}.$ 

{\bf Cumulative hard violation of $d_t$ function:} we have the following decomposition for $d_t$ function
\begin{align*}
\mathcal V_d^{hard}(T)
\leq& \sum_{t=1}^T [d_t^{+}(x_t^\pi, u_t^\pi) - d_t^{+}(\tilde x_t^\pi, \tilde u_t^\pi)] + \sum_{t=1}^T d_t^{+}(\tilde x_t^\pi, \tilde u_t^\pi)\\
\leq& 1 + 2LH^2\sqrt{\frac{(2LD+\nu_{\max}^2)T\log T}{\lambda}} + (2LD+\nu_{\max}^2)\log T\\
=& O(\sqrt{T}\log^3 T)
\end{align*}
where the inequality holds because of the hard violation in Theorem \ref{thm:ocowm-best} and the approximated error in Lemma \ref{lem:fun-diff}. 

{\bf Anytime violation of $l$ function:} we have the following decomposition for the constraint function $l$
\begin{align*}
\mathcal V_l(t)
=& l(x_t^\pi, u_t^\pi) - l(\tilde x_t^\pi, \tilde u_t^\pi) + l(\tilde x_t^\pi, \tilde u_t^\pi)\\
\leq& \frac{1}{T}  + 2H^2\sqrt{\frac{(4\sqrt{t}LD + 2\nu_{\max}^2 +  Q^2_{\max})}{2\lambda t^{1.5}}}\\
 =& O(1/\log T)
\end{align*}
where the inequality holds because of the anytime violation in Theorem \ref{thm:ocowm-best} and the approximated error in Lemma \ref{lem:fun-diff}. The order-wise result is established when $t \geq \frac{10LD\log^6 T}{\lambda}$ and $\log^3 T \geq \frac{(2\lambda\nu_{\max}^2 + Q_{\max}^2)}{(LD)^{1.5}}.$

\subsection{Performance of COCO-Best2Worlds Solver in Theorem \ref{thm:ocowm-best}}
Note the term of $H$-step difference can be bounded as follows
\begin{align}
\sum_{i=1}^{H} \|\mathbf M_{t-i} - \mathbf M_{t}\| =& \sum_{i=1}^{H} \sum_{j=1}^{i}  \|\mathbf M_{t-j} - \mathbf M_{t-j+1}\|. \nonumber
\end{align}
Based on Lemma \ref{lemma:diff-best}, we have 
\begin{align*}
\sum_{t=1}^T|f_t(\mathbf M_{t-H:t}) - f_t(\mathbf M_{t})|  
\leq&  L\sqrt{\frac{(2LD+\nu_{\max}^2)T\log T}{\lambda}}, \\
\sum_{t=1}^T|g_t(\mathbf M_{t-H:t}) - g_t(\mathbf M_{t})|  \leq&  LH^2\sqrt{\frac{(2LD+\nu_{\max}^2)T\log T}{\lambda}}, \\
\sum_{t=1}^T|g_t^{+}(\mathbf M_{t-H:t}) - g_t^{+}(\mathbf M_{t})|  \leq&  LH^2\sqrt{\frac{(2LD+\nu_{\max}^2)T\log T}{\lambda}}, \\
|h(\mathbf M_{t-H:t}) - h(\mathbf M_{t})| \leq& LH^2\sqrt{\frac{(4\sqrt{t}LD + 2\nu_{\max}^2 +  Q^2_{\max})}{2\lambda t^{1.5}}}.
\end{align*}
Finally, in conjugation with Lemma \ref{lemma:ocowm-best}, we prove Theorem \ref{thm:ocowm-best} 
as follows:
\begin{align*}
\mathcal R_f(T) \leq&  \sum_{t=1}^T|f_t(\mathbf M_{t-H:t}) - f_t(\mathbf M_{t})| +
\sum_{t=1}^T f_t(\mathbf M_{t}) - f_t(\mathbf M^*) \\
\leq& LH^2\sqrt{\frac{(2LD+\nu_{\max}^2)\log T}{\lambda}} + \frac{L^2\log T}{\lambda} + 2\nu_{\max}^2 \sqrt{T} + \frac{\epsilon TLD}{\delta}, ~\\
\mathcal V_g^{hard}(T) \leq&  \sum_{t=1}^T|g_t(\mathbf M_{t-H:t}) - g_t(\mathbf M_{t})| +
\sum_{t=1}^T g_t(\mathbf M_{t}) \\
\leq& 2LH^2\sqrt{\frac{(2LD+\nu_{\max}^2)\log T}{\lambda}} + (2LD+\nu_{\max}^2)\log T ,\\
\mathcal V_g^{soft}(T) \leq&  \sum_{t=1}^T|g_t(\mathbf M_{t-H:t}) - g_t(\mathbf M_{t})| +
\sum_{t=1}^T g_t(\mathbf M_{t}) \\
\leq& 3LH^2\sqrt{\frac{(2LD+\nu_{\max}^2)\log T}{\lambda}} -T\epsilon ,\\
\mathcal V_h(t) \leq& |h(\mathbf M_{t-H:t}) - h(\mathbf M_{t})| + h(\mathbf M_{t})\\
\leq&2LH^2\sqrt{\frac{(4\sqrt{t}LD + 2\nu_{\max}^2 +  Q^2_{\max})}{2\lambda t^{1.5}}}.
\end{align*}

\begin{lemma}
Under COCO-Best2Worlds solver, we have \begin{align*}
    \sum_{t=1}^T f_t(\mathbf M_{t}) - f_t(\mathbf M^*) \leq& \frac{L^2\log T}{\lambda} + 2\nu_{\max}^2 \sqrt{T} + \frac{\epsilon TLD}{\delta}, ~\\
    \sum_{t=1}^T g_t^{+}(\mathbf M_{t}) \leq&  L\sqrt{\frac{(2LD+\nu_{\max}^2)T\log T}{\lambda}} + (2LD+\nu_{\max}^2)\log T, \\
    ~\sum_{t=1}^T g_t(\mathbf M_{t}) \leq& 2L\sqrt{\frac{(2LD+\nu_{\max}^2)T\log T}{\lambda}}-T\epsilon,\\
    h(\mathbf M_t) \leq& \frac{2LD\sqrt{t} +\nu_{\max}^2 + Q_{\max}^2}{t^{1.5}}, ~\forall t \in [T].\\
\end{align*}
with the probability at least $1-1/T.$
\label{lemma:ocowm-best}
\end{lemma}
Further, we establish the difference between $\mathbf M_{t+1}$ and $\mathbf M_{t}.$    
\begin{lemma}
Under COCO-Best2Worlds , we have \begin{align*}
      \sum_{t=1}^T\|\mathbf M_{t+1} - \mathbf M_{t}\| \leq&  \sqrt{\frac{(2LD+\nu_{\max}^2)T\log T}{\lambda}},\\
      \|\mathbf M_{t+1} - \mathbf M_{t}\| \leq&  \sqrt{\frac{4\sqrt{t}LD + 2\nu_{\max}^2 +  Q^2_{\max}}{2\lambda t^{1.5}}}, ~\forall t \in [T],
\end{align*}\label{lemma:diff-best}
\end{lemma}

\subsubsection{Proof of Lemma \ref{lemma:ocowm-best}}
To prove Lemma \ref{lemma:ocowm-best}, we first introduce the following key lemma.
\begin{lemma}
For any $\mathbf M \in \mathcal M,$ we have 
\begin{align*}
    &V_t\langle \mathbf M_{t+1} - \mathbf M_{t}, \nabla f_t(\mathbf M_t) \rangle + Q_t \langle \mathbf M_{t+1} - \mathbf M_{t}, \nabla g_t(\mathbf M_t) \rangle + \gamma_t g^{+}_t(\mathbf M_{t+1}) + \eta_t h^{+}(\mathbf M_{t+1}) \nonumber \\ 
    &~+ \alpha_t \|\mathbf M_{t+1} - \mathbf M_{t}\|^2 \\
    \leq& V_t\langle \mathbf M - \mathbf M_{t}, \nabla f_t(\mathbf M_t) \rangle + Q_t \langle \mathbf M-\mathbf M_t, \nabla g_t(\mathbf M_t) \rangle + \gamma_t g_t^{+}(\mathbf M) + \eta_t h^{+}(\mathbf M) \nonumber \\
    & ~+ \alpha_t \|\mathbf M - \mathbf M_{t}\|^2 - \alpha_t \|\mathbf M - \mathbf M_{t+1}\|^2
\end{align*}\label{lem:pushback-best}
\end{lemma}
\begin{proof}
\normalfont
The proof is a direct application of Lemma \ref{lemma:tool}. 
Let $$F(\mathbf M) = V_t\langle \mathbf M - \mathbf M_{t}, \nabla f_t(\mathbf M) \rangle + Q_t \langle \mathbf M - \mathbf M_{t}, \nabla g_t(\mathbf M_t) \rangle + \gamma_t g^{+}(\mathbf M) + \eta_t h^{+}(\mathbf M).$$ Since $F(\mathbf M)+ \alpha_t \|\mathbf M - \mathbf M_{t}\|^2$ is $2\alpha_t$-strongly convex, the proof is completed by Lemma \ref{lemma:tool}. 
\end{proof} 
By adding $g_t(\mathbf M_t)$ into the inequality in Lemma \ref{lem:pushback-best}, we have
for any $\mathbf M \in \mathcal M$ such that
\begin{align}
    &V_t\langle \mathbf M_{t+1} - \mathbf M_{t}, \nabla f_t(\mathbf M_t) \rangle + Q_t[g_t(\mathbf M_t) + \langle \mathbf M_{t+1} - \mathbf M_{t}, \nabla g_t(\mathbf M_t) \rangle] + \gamma_t g^{+}_t(\mathbf M_{t+1})+ \eta_t h^{+}(\mathbf M_{t+1}) \nonumber\\
    &~~+ \alpha_t \|\mathbf M_{t+1} - \mathbf M_{t}\|^2 \nonumber\\
    \leq& V_t\langle \mathbf M - \mathbf M_{t}, \nabla f_t(\mathbf M_t) \rangle + Q_t[g_t(\mathbf M_t) + \langle \mathbf M-\mathbf M_t, \nabla g_t(\mathbf M_t) \rangle] + \gamma_t g^{+}_t(\mathbf M) + \eta_t h^{+}(\mathbf M) \nonumber\\
    &+ \alpha_t \|\mathbf M - \mathbf M_{t}\|^2 - \alpha_t \|\mathbf M - \mathbf M_{t+1}\|^2 \nonumber\\
    \leq& V_t\langle \mathbf M - \mathbf M_{t}, \nabla f_t(\mathbf M_t) \rangle + Q_tg_t(\mathbf M) + \eta_t h^{+}(\mathbf M) + \gamma_t g^{+}_t(\mathbf M) + \alpha_t \|\mathbf M - \mathbf M_{t}\|^2 - \alpha_t \|\mathbf M - \mathbf M_{t+1}\|^2, \label{eq:pushback-best}
\end{align}
where the last inequality holds because $g_t(\cdot)$ is convex. 
According to the virtual queue update as follows
\begin{align}
  Q_{t+1} = \left[Q_t + g_t(\mathbf M_t) + \langle \mathbf M_{t+1} - \mathbf M_{t}, \nabla g_t(\mathbf M_t) \rangle + \epsilon \right]^{+}, \nonumber
\end{align}
we have
\begin{align}
  \frac{1}{2}Q^2_{t+1} - \frac{1}{2} Q^2_{t} \leq& \frac{1}{2}\left(Q_t + g_t(\mathbf M_t) + \langle \mathbf M_{t+1} - \mathbf M_{t}, \nabla g_t(\mathbf M_t) \rangle  \right)^{2} - \frac{1}{2} Q^2_t \nonumber\\
  \leq& Q_t(g_t(\mathbf M_t) + \langle \mathbf M_{t+1} - \mathbf M_{t}, \nabla g_t(\mathbf M_t) \rangle  ) + \nu_{\max}^2, \label{eq:vq-best}
 \end{align}
where the first inequality holds because $(x^{+})^2 \leq x^2$ for any $x\in \mathbb R;$ the second inequality holds because of Assumptions \ref{assumption:set} and \ref{assumption:obj}; and recall $\nu_{\max} = E+LD+\epsilon$. 
 
By combining the two inequalities in \eqref{eq:pushback-best} and \eqref{eq:vq-best}, we have 
\begin{align}
    &V_t\langle \mathbf M_{t+1} - \mathbf M_{t}, \nabla f_t(\mathbf M_t)\rangle +  \frac{1}{2}Q^2_{t+1} - \frac{1}{2} Q^2_t + \gamma_t g^{+}_t(\mathbf M_{t+1}) + \eta_t h^{+}(\mathbf M_{t+1}) \nonumber\\
    &+ \alpha_t \|\mathbf M_{t+1} - \mathbf M_{t}\|^2 \label{eq:key}\\
    \leq& V_t\langle \mathbf M - \mathbf M_{t}, \nabla f_t(\mathbf M_t) \rangle + Q_t g_t(\mathbf M) + \nu_{\max}^2 + \gamma_t g^{+}_t(\mathbf M)+ \eta_t h^{+}(\mathbf M) +\alpha_t \|\mathbf M - \mathbf M_{t}\|^2 \nonumber\\
    &- \alpha_t \|\mathbf M - \mathbf M_{t+1}\|^2. \nonumber
\end{align}
By adding $V_tf_t(\mathbf M_t)$ on both sides of the inequality in \eqref{eq:key}, we have
\begin{align}
    &V_tf_t(\mathbf M_t) + V_t\langle \mathbf M_{t+1} - \mathbf M_{t}, \nabla f_t(\mathbf M_t)\rangle +  \frac{1}{2}Q^2_{t+1} - \frac{1}{2} Q^2_{t} + \gamma_t g^{+}_t(\mathbf M_{t+1})  + \eta_t h^{+}(\mathbf M_{t+1}) \nonumber \\
    &+ \alpha_t \|\mathbf M_{t+1} - \mathbf M_{t}\|^2 \nonumber\\
    \leq& V_tf_t(\mathbf M_t) + V_t\langle \mathbf M - \mathbf M_{t}, \nabla f_t(\mathbf M_t) \rangle + Q_t(g_t(\mathbf M)+\epsilon) + \nu_{\max}^2 + \gamma_t g^{+}_t(\mathbf M) + \eta_t h^{+}(\mathbf M) \nonumber\\
    &+ \alpha_t \|\mathbf M - \mathbf M_{t}\|^2 - \alpha_t \|\mathbf M - \mathbf M_{t+1}\|^2 \nonumber \\
    \leq& V_t f_t(\mathbf M) + Q_tg_t(\mathbf M) + \nu_{\max}^2 + \gamma_t g^{+}_t(\mathbf M) + \eta_t h^{+}(\mathbf M) \nonumber \\
    & + (\alpha_t - \lambda V_t/2) \|\mathbf M - \mathbf M_{t}\|^2 - \alpha_t \|\mathbf M - \mathbf M_{t+1}\|^2 \label{eq: strongly key}
\end{align}
where the last inequality holds because $f_t(\cdot)$ is $\alpha_t-$strongly convex $\forall t \in [T]$.

Based on the key inequality in \eqref{eq:key}, we establish the regret and constraint violations in Lemma \ref{lemma:ocowm-best} in the following. Recall the learning rate $V_t = \sqrt{t}, \gamma_t = t^{1.5}, \eta_t = t^{1.5},$ $\alpha_t = \lambda t^{1.5}/2,$ and $\epsilon = \log^3 T/\sqrt{T}.$
\vspace{7pt}

{\noindent \bf Regret Analysis:}
recall the following $\epsilon$-tightness problem of COCOwM as our baseline (with $\epsilon>0$): 
\begin{align*}
    \min_{\mathbf M\in \mathcal M}& ~ \sum_{t=1}^T f_t(\mathbf M)\\
    \text{s.t.}&~ 
     h(\mathbf M) + \epsilon \leq 0, ~ g_t(\mathbf M) + \epsilon \leq 0, \forall t \in [T].
\end{align*}
Let $\mathbf M = \mathbf M^{\epsilon}$ be any feasible solution to the offline problem such that $g_t(\mathbf M^{\epsilon}) + \epsilon \leq 0$. Taking the summation of \eqref{eq: strongly key} from time $t=1$ to $T,$ we have
\begin{align}
    &\sum_{t=1}^T f_t(\mathbf M_t) - f_t(\mathbf M^{\epsilon}) + \sum_{t=1}^T\left(\frac{Q^2_{t+1}}{2V_t}  - \frac{Q^2_t}{2V_t}\right)  \nonumber\\
    \leq& \sum_{t=1}^T \frac{V_t L^2}{4\alpha_t} + \sum_{t=1}^T \frac{\nu_{\max}^2}{V_t} + \sum_{t=1}^T \left(\frac{\alpha_t}{V_t} - \frac{\alpha_{t-1}}{V_{t-1}} - \frac{\lambda}{2}\right) \|\mathbf M^{\epsilon} - \mathbf M_{t}\|^2  , \nonumber
\end{align}
which, in conjugation with $V_t \geq V_{t-1},$ $Q_1=0,$ and $\frac{\alpha_t}{V_t} - \frac{\alpha_{t-1}}{V_{t-1}} - \frac{\lambda}{2} = 0,$ implies 
\begin{align}
    \sum_{t=1}^T f_t(\mathbf M_t) - f_t(\mathbf M^{\epsilon}) 
    \leq \sum_{t=1}^T \frac{V_t L^2}{4\alpha_t} + \sum_{t=1}^T \frac{\nu_{\max}^2}{V_t}. \label{eq:regret-best}
\end{align}
Therefore, we have for any $\mathbf M^{\epsilon} \in \mathcal M$ such that
\begin{align*}
 \sum_{t=1}^T f_t(\mathbf M_t) - f_t(\mathbf M) =& \sum_{t=1}^T f_t(\mathbf M_t) - f_t(\mathbf M^{\epsilon}) + \sum_{t=1}^T f_t(\mathbf M^{\epsilon}) - f_t(\mathbf M)\\
\leq& \sum_{t=1}^T \frac{V_t L^2}{4\alpha_t} + \sum_{t=1}^T \frac{\nu_{\max}^2}{V_t} + \frac{\epsilon TLD}{\delta} = \frac{L^2\log T}{\lambda} + 2\nu_{\max}^2 \sqrt{T} + \frac{\epsilon TLD}{\delta}
\end{align*}
where the first inequality holds by \eqref{eq:regret-best}; the second inequality holds by Lemma \ref{lem:tightness}. 

\vspace{7pt}
{\noindent \bf Hard Violation of $\mathcal V_g^{hard}(T)$:}
From \eqref{eq: strongly key}, we have
\begin{align}
    & \gamma_t g^{+}_t(\mathbf M_{t+1}) +  \frac{1}{2}Q^2_{t+1} - \frac{1}{2} Q^2_t + \alpha_t \|\mathbf M_{t+1} - \mathbf M_{t}\|^2 \nonumber\\
    \leq& V_t \left(f_t(\mathbf M)-f_t(\mathbf M_t\right)-\langle \mathbf M_{t+1} - \mathbf M_{t}, \nabla f_t(\mathbf M_t)\rangle) + \nu_{\max}^2 + (\alpha_t - \lambda V_t/2) \|\mathbf M^{\epsilon} - \mathbf M_{t}\|^2 - \alpha_t \|\mathbf M^{\epsilon} - \mathbf M_{t+1}\|^2 \nonumber\\
    \leq& 2V_tLD + \nu_{\max}^2 + (\alpha_t - \lambda V_t/2) \|\mathbf M^{\epsilon} - \mathbf M_{t}\|^2 - \alpha_t \|\mathbf M^{\epsilon} - \mathbf M_{t+1}\|^2 \label{eq: key-best}
\end{align}
which implies 
\begin{align*}
    \sum_{t=1}^T g^{+}_t(\mathbf M_{t+1})
    \leq \sum_{t=1}^T \frac{2V_tLD + \nu_{\max}^2}{\gamma_t} \leq (2LD+\nu_{\max}^2)\log T, \nonumber\\
    \sum_{t=1}^T \|\mathbf M_{t+1} - \mathbf M_{t}\|^2 
    \leq \sum_{t=1}^T \frac{2V_tLD + \nu_{\max}^2}{\alpha_t} \leq \frac{(2LD+\nu_{\max}^2)\log T}{\lambda}.
\end{align*}
Therefore, we have
\begin{align*}
    \sum_{t=1}^T g^{+}_t(\mathbf M_{t}) - g^{+}_t(\mathbf M_{t+1}) \leq& \sum_{t=1}^T \langle \nabla g_t^{+}(\mathbf M_t), \mathbf M_{t+1} - \mathbf M_t\rangle \\
    \leq& L\sum_{t=1}^T \|\mathbf M_{t+1} - \mathbf M_{t}\| \leq L\sqrt{\frac{(2LD+\nu_{\max}^2)T\log T}{\lambda}}.
\end{align*}
Finally, the cumulative hard violation is 
\begin{align*}
    \sum_{t=1}^T g^{+}_t(\mathbf M_{t}) =& \sum_{t=1}^T g^{+}_t(\mathbf M_{t}) - g^{+}_t(\mathbf M_{t+1}) + \sum_{t=1}^Tg^{+}_t(\mathbf M_{t+1})\\
    \leq& L\sqrt{\frac{(2LD+\nu_{\max}^2)T\log T}{\lambda}} + (2LD+\nu_{\max}^2)\log T.
\end{align*}

\vspace{7pt}
{\noindent \bf Soft Violation of $\mathcal V_g^{soft}(T)$:}
From \eqref{eq: strongly key}, we have
\begin{align*}
    &\frac{1}{2}Q^2_{t+1} - \frac{1}{2} Q^2_t \\
    \leq& V_t \left(f_t(\mathbf M)-f_t(\mathbf M_t\right)-\langle \mathbf M_{t+1} - \mathbf M_{t}, \nabla f_t(\mathbf M_t)\rangle) - \alpha_t \|\mathbf M_{t+1} - \mathbf M_{t}\|^2 + Q_t g_t(\mathbf M^{\epsilon})
 + \nu_{\max}^2  \nonumber \\
    & + \gamma_t g^{+}_t(\mathbf M^{\epsilon}) + \eta_t h^{+}(\mathbf M^{\epsilon}) + (\alpha_t - \lambda V_t/2) \|\mathbf M^{\epsilon} - \mathbf M_{t}\|^2 - \alpha_t \|\mathbf M^{\epsilon} - \mathbf M_{t+1}\|^2 \\
    \leq& \frac{V_t^2 LD}{2\alpha_t} + \nu_{\max}^2 + Q_t g_t(\mathbf M^{\epsilon}) + \gamma_t g^{+}_t(\mathbf M^{\epsilon}) + \eta_t h^{+}(\mathbf M^{\epsilon}) + (\alpha_t - \lambda V_t/2) \|\mathbf M^{\epsilon} - \mathbf M_{t}\|^2 - \alpha_t \|\mathbf M^{\epsilon} - \mathbf M_{t+1}\|^2.
\end{align*}
Therefore, our Lyapunov drift is established as follows
\begin{align}
    \mathbb E\left[\frac{1}{2}Q^2_{t+1} - \frac{1}{2} Q^2_{t}|\mathcal H_t\right]
    \leq - \delta Q_t + (LD + \nu_{\max}^2). \nonumber
\end{align}
Recall $Q_t \geq \theta := \frac{2(LD+\nu_{\max}^2)}{\delta},$ we have
\begin{align}
    \mathbb E\left[\frac{1}{2}Q^2_{t+1} - \frac{1}{2} Q^2_t|\mathcal H_t\right]
    \leq& - \frac{\delta Q_t}{2}, \nonumber
\end{align}
which implies conditional on $\mathcal H_t$
\begin{align}
    \mathbb E\left[Q^2_{t+1}|\mathcal H_t\right]
    \leq& Q^2_t - \delta Q_t\nonumber \\
    \leq & (Q_t - \frac{\delta}{2})^2. \nonumber
\end{align}
According to the Jensen's inequality $(\mathbb E\left[Q_{t+1}|\mathcal H_t\right])^2 \leq \mathbb E\left[Q^2_{t+1}|\mathcal H_t\right],$ we have
\begin{align}
    \mathbb E\left[Q_{t+1} - Q_t|\mathcal H_t\right]
    \leq&  - \frac{\delta}{2}, \nonumber
\end{align}
which, in conjugation with Lemma \ref{lemma:drift}, implies that $$\mathbb E[Q_t] \leq \theta + \nu_{\max} + \frac{16\nu_{\max}^2}{\delta} \log\frac{128 \nu_{\max}^2}{\delta^2}.$$
and 
$$\mathbb P\left(Q_t > \theta + \nu_{\max} + \frac{16 \nu_{\max}^2}{\delta} \left(\log\frac{128 \nu_{\max}^2}{\delta^2}+2\log T\right)\right) < \frac{1}{T^2}.$$

Let $Q_{\max}:= \theta + \nu_{\max} + \frac{16 \nu_{\max}^2}{\delta} \left(\log\frac{128 \nu_{\max}^2}{\delta^2}+2\log T\right).$ According to the virtual queue update, we have 
\begin{align*}
    \sum_{t=1}^T g_t(\mathbf M_t) \leq& Q_{T+1} + L\sum_{t=1}^T \|\mathbf M_t-\mathbf M_{t+1}\|- T\epsilon\\
    \leq& Q_{\max} + L\sum_{t=1}^T \|\mathbf M_t-\mathbf M_{t+1}\| - T\epsilon,
\end{align*}
hold with the probability at least $1-1/T.$ Finally, the proof is completed because of Lemma \ref{lemma:diff-best}.

\vspace{7pt}
{\noindent \bf Anytime Violation of $V_h(t)$}:
From \eqref{eq: strongly key}, we have
\begin{align*}
    & \eta_t h^{+}(\mathbf M_{t+1}) +  \frac{1}{2}Q^2_{t+1} - \frac{1}{2} Q^2_t + \alpha_t \|\mathbf M_{t+1} - \mathbf M_{t}\|^2 \\
    \leq& V_t \left(f_t(\mathbf M)-f_t(\mathbf M_t\right)-\langle \mathbf M_{t+1} - \mathbf M_{t}, \nabla f_t(\mathbf M_t)\rangle) + \nu_{\max}^2 + (\alpha_t - \lambda V_t/2) \|\mathbf M - \mathbf M_{t}\|^2 - \alpha_t \|\mathbf M - \mathbf M_{t+1}\|^2\\
    \leq& 2V_tLD + \nu_{\max}^2 + (\alpha_t - \lambda V_t) \|\mathbf M^{\epsilon} - \mathbf M_{t}\|^2 - \alpha_t \|\mathbf M^{\epsilon} - \mathbf M_{t+1}\|^2
\end{align*}
which implies 
\begin{align*}
     h^{+}(\mathbf M_{t})
    \leq \frac{2V_tLD +\nu_{\max}^2 + Q_{\max}^2}{\eta_t}=\frac{2LD\sqrt{t} +\nu_{\max}^2 + Q_{\max}^2}{t^{1.5}}, ~\forall t \in [T],
\end{align*}
holds with the probability at least $1-1/T.$ 

\subsection{Proof of Lemma \ref{lemma:diff-best}}
From \eqref{eq: key-best} in the proof of Lemma \ref{lemma:ocowm-best}, we have \begin{align*}
    \sum_{t=1}^T \|\mathbf M_{t+1} - \mathbf M_{t}\|^2 \leq \frac{(2LD+\nu_{\max}^2)\log T}{\lambda},
\end{align*}
which implies by Cauchy-Schwarz inequality
\begin{align*}
      \sum_{t=1}^T\|\mathbf M_{t+1} - \mathbf M_{t}\| \leq& \sqrt{\frac{(2LD+\nu_{\max}^2)T\log T}{\lambda}}.
\end{align*}
Let $\mathbf M^{\epsilon} = \mathbf M_t$ in \eqref{lemma:ocowm-best}, we have
\begin{align}
    2\alpha_t \|\mathbf M_{t+1} - \mathbf M_{t}\|^2 
    \leq 2V_tLD + \nu_{\max}^2 +  \frac{Q^2_{\max}}{2}  \nonumber
\end{align}
which implies $$\|\mathbf M_{t+1} - \mathbf M_{t}\| \leq \sqrt{\frac{4V_tLD + 2\nu_{\max}^2 +  Q^2_{\max}}{4\alpha_t}}.$$

\subsection{Specifying Parameters of COCO-Best2Worlds with $\tilde{c}_t(\mathbf M),\tilde{d}_t(\mathbf M), \tilde{l}(\mathbf M).$}
The parameters $D, L, E,$ and $\xi$ in Assumptions \ref{assumption:set}-\ref{assumption:slater}  are exactly the same as in COCO-Soft as follows
\begin{align*}
    D:=& \frac{2a}{\rho},~
    L:= \frac{ 2C_0W(1+\kappa)(\kappa^2 + Ha\kappa_B\kappa^2)}{\rho} + \frac{2aC_0W}{\rho}, \\
    E:=& C_0 D + C_1D, ~ 
    \xi:= \frac{\delta}{2}
\end{align*}
Moreover, according to Lemma 4.2 in \cite{AgaHazSin_19}, the assumption of $\Tilde{\lambda}$-strongly-convex cost function $c_t(\tilde x_t^{\pi}, \ u_t^{\pi})$  implies $c_t(\tilde x_t^{\pi}, \tilde u_t^{\pi})$ is also $
\lambda=\frac{\Tilde{\lambda}\sigma^2\rho^2}{36\kappa^{10}}$- strongly-convex under Assumption \ref{assumption: strongly convex}. Therefore, we specify $\lambda = \frac{\Tilde{\lambda}\sigma^2\rho^2}{36\kappa^{10}}$.

\section{Approximated Errors of Cost and Constraint Functions}
We quantify the approximated errors of cost and constraint functions in the following lemma, which are used to prove the regret and constraint violation in Theorems \ref{thm:occa}, \ref{thm:occa-hard}, and \ref{thm:occa-best} as suggested in the roadmap. 
\begin{lemma}
Under a policy in constrained DAC with the memory size $H=\Theta(\log T)$, we have
\begin{align*}
    \sum_{t=1}^T [c_t(x_t^\pi, u_t^\pi) - c_t(\tilde x_t^\pi, \tilde u_t^\pi)] =& O(1), \\
    \sum_{t=1}^T [d_t(x_t^\pi, u_t^\pi) - d_t(\tilde x_t^\pi, \tilde u_t^\pi)] =& O(1),\\
    l(x_t^\pi, u_t^\pi) - l(\tilde x_t^\pi, \tilde u_t^\pi) =& O(1/T). 
\end{align*}\label{lem:fun-diff}
\end{lemma}
\subsection{Proof of Lemma \ref{lem:fun-diff}}
We first study the following difference between $c_t(x^\pi_t, u^\pi_t)$ and $c_t(\tilde x^\pi_t, \tilde u^\pi_t)$ 
\begin{align}
&\sum_{t=1}^T c_t(x^\pi_t, u^\pi_t) - c_t(\tilde x^\pi_t, \tilde u^\pi_t)\nonumber\\
\leq & C_0 \sum_{t=1}^T \left(\|x^\pi_t-\tilde x^\pi_t\| + \|u^\pi_t - \tilde u^\pi_t\|\right) \nonumber\\
\leq& 2TC_0 \kappa(1-\rho)^H \frac{W\kappa^2(\kappa^2 + Ha\kappa_B\kappa^2)}{\rho(1-\kappa^2(1-\rho)^{H+1})}. \label{eq: c diff proof}
\end{align}
where the first inequality holds by Assumption \ref{assumption: fun bound} and the second inequality holds by Lemma \ref{lem: dac approx error}.

By following the same argument, we have
\begin{align}
&\sum_{t=1}^T d_t(x^\pi_t, u^\pi_t) - d_t(\tilde x^\pi_t, \tilde u^\pi_t)
\leq 2TC_0 \kappa(1-\rho)^H \frac{W\kappa^2(\kappa^2 + Ha\kappa_B\kappa^2)}{\rho(1-\kappa^2(1-\rho)^{H+1})}, \label{eq: d diff proof}
\end{align}
and 
\begin{align}
l(x^\pi_t, u^\pi_t) - l(\tilde x^\pi_t, \tilde u^\pi_t)
\leq 2C_0 \kappa(1-\rho)^H \frac{W\kappa^2(\kappa^2 + Ha\kappa_B\kappa^2)}{\rho(1-\kappa^2(1-\rho)^{H+1})}. \label{eq: l diff proof}
\end{align}
Choosing $H= 3\log T/\rho,$ we have 
\begin{align}
2C_0\kappa(1-\rho)^H \frac{W\kappa^2(\kappa^2 + 2Ha\kappa_B\kappa^2)}{\rho(1-\kappa^2(1-\rho)^{H+1})} \leq& \frac{2C_0W(1+\kappa)^5(1 + 2Ha\kappa_B)}{\rho(T^3-\kappa^2)}\nonumber\\
\leq& \frac{2C_0W\kappa^5(1 + 2Ha\kappa_B)}{T^2} \nonumber\\
\leq& \frac{1}{T}, \label{eq: large T}
\end{align}
where the first inequality holds because $e^{-x} \geq 1- x$ for any $0\leq x\leq 1;$ the second inequality holds $T\geq 2/\rho + \kappa^2$; the last inequality holds  because $T \geq 2C_0W\kappa^5(1 + 2Ha\kappa_B).$ The proof is completed by using \eqref{eq: large T} in \eqref{eq: c diff proof}, \eqref{eq: d diff proof}, and \eqref{eq: l diff proof}. 

\section{Representation Ability of Constrained  DAC}
We quantify the representation ability
of disturbance-action policy with constraints in the following lemma, which are used to prove the regret in Theorems \ref{thm:occa}, \ref{thm:occa-hard}, and \ref{thm:occa-best} as suggested in the roadmap. 
\begin{lemma}
Under a policy in constrained DAC with the memory size $H=\Theta(\log T)$, we have
\begin{align*}
    \min_{\pi \in \tilde{\Omega} \bigcap
 \mathcal E}\sum_{t=1}^T c_t(\tilde x_t^{\pi}, \tilde u_t^{\pi}) - \sum_{t=1}^T c_t(x^{K^*}_t, u^{K^*}_t) =& O(1).
\end{align*}\label{lem:dac-rep}
\end{lemma}
We define a disturbance-action policy $\pi(K, \{\mathbf M_t\}) \in \mathcal E$ such that $$\mathbf M_t = \mathbf{\bar M}, \forall t \in [T] ~~\text{and}~~ \mathbf{\bar M}^{[i]} = (K-K^*) \tilde{A}_{K^*}^{i-1}.$$ 
and we denote it to be $\bar \pi.$ Let $\bar \epsilon = 1/T.$ In the following, we show that
\begin{align}
d_t(\tilde x^{\bar{\pi}}_t, \tilde u^{\bar{\pi}}_t) \leq \bar \epsilon, ~ l(\tilde x^{\bar{\pi}}_t, \tilde u^{\bar{\pi}}_t) \leq \bar \epsilon, \forall t \in [T].
\nonumber
\end{align}
For the constraint $d_t,$ we have 
\begin{align*}
     d_t(\tilde x^{\bar{\pi}}_t, \tilde u^{\bar{\pi}}_t) =&  d_t(\tilde x^{\bar{\pi}}_t, \tilde u^{\bar{\pi}}_t) - d_t(x^{\bar{\pi}}_t, u^{\bar{\pi}}_t) + d_t(x^{\bar{\pi}}_t, u^{\bar{\pi}}_t) - d_t(x^{K^*}_t, u^{K^*}_t) + d_t(x^{K^*}_t, u^{K^*}_t)\\
    \leq&  d_t(\tilde x^{\bar{\pi}}_t, \tilde u^{\bar{\pi}}_t) - d_t(x^{\bar{\pi}}_t, u^{\bar{\pi}}_t) + d_t(x^{\bar{\pi}}_t, u^{\bar{\pi}}_t) - d_t(x^{K^*}_t, u^{K^*}_t)\\
    \leq& C_0 (\|\tilde x_{t}^{\bar{\pi}} - x_{t}^{\bar{\pi}}\| + \|\tilde u_{t}^{\bar{\pi}} - u_{t}^{\bar{\pi}}\|) + C_0 (\|x_{t}^{\bar{\pi}} - x_{t}^{K^*}\| + \|u_{t}^{\bar{\pi}} - u_{t}^{K^*}\|)\\
    \leq& 2C_0 \kappa(1-\rho)^H \frac{W\kappa^2(\kappa^2 + Ha\kappa_B\kappa^2)}{\rho(1-\kappa^2(1-\rho)^{H+1})} + \frac{2C_0HW (\kappa+1)^2 \kappa_B a(1-\rho)^{H+1}}{\rho}\\
    \leq& \frac{2C_0W(1+\kappa)^5(1 + 2Ha\kappa_B)}{\rho(T^3-\kappa^2)}\\
    \leq& \bar \epsilon
\end{align*}
where the first inequality holds because $K^* \in \Omega$ and $d_t(x^{K^*}_t, u^{K^*}_t) \leq 0;$ the second inequality holds because of Assumption \ref{assumption: fun bound}; the third inequality holds because of Lemma \ref{lem: dac approx error} and Lemma \ref{lem:xu diff}. The last inequality holds for a large $T$ by letting $H = 3\log T/\rho$ according to \eqref{eq: large T}. 
For the constraint $l,$ by following the similar steps, we have $$l(\tilde x^{\bar{\pi}}_t, \tilde u^{\bar{\pi}}_t) \leq \bar{\epsilon}, \forall t \in [T].$$ 
Define the set $\bar{\Omega}$
$$\bar{\Omega} = \{\pi ~|~ d_t(\tilde x_t^{\pi}, \tilde u_t^{\pi}) \leq \bar{\epsilon},~  l(\tilde x_t^{\pi}, \tilde u_t^{\pi}) \leq \bar{\epsilon}, \forall t \in [T]\}.$$
Therefore, we have established $$\bar \pi \in \bar{\Omega}.$$
Next, we use $\bar \pi$ to bridge $\pi$ and $K^*,$ and prove Lemma \ref{lem:dac-rep} in the following. 
\begin{align*}
    &\min_{\pi \in \tilde{\Omega} \bigcap
 \mathcal E}\sum_{t=1}^T c_t(\tilde x_t^{\pi}, \tilde u_t^{\pi}) - \sum_{t=1}^T c_t(x^{K^*}_t, u^{K^*}_t)  \\
    =&\min_{\pi \in \tilde{\Omega} \bigcap
 \mathcal E}\sum_{t=1}^T c_t(\tilde x_t^{\pi}, \tilde u_t^{\pi}) - \sum_{t=1}^T c_t(\tilde x_t^{\bar \pi}, \tilde u_t^{\bar \pi}) + \sum_{t=1}^T [c_t(\tilde x_t^{\bar \pi}, \tilde u_t^{\bar \pi}) - c_t(x_t^{\bar \pi}, u_t^{\bar \pi})]\\
    & + \sum_{t=1}^T [c_t(x_t^{\bar \pi}, u_t^{\bar \pi}) - c_t(x^{K^*}_t, u^{K^*}_t)]
\end{align*}
\begin{itemize}
    \item The first term can be bounded as follows 
\begin{align*}
    &\min_{\pi \in \tilde{\Omega} \bigcap
 \mathcal E}\sum_{t=1}^T c_t(\tilde x_t^{\pi}, \tilde u_t^{\pi}) - \sum_{t=1}^T c_t(\tilde x_t^{\bar \pi}, \tilde u_t^{\bar \pi})\\
    \leq&\min_{\pi \in \tilde{\Omega} \bigcap
 \mathcal E}\sum_{t=1}^T c_t(\tilde x_t^{\pi}, \tilde u_t^{\pi}) - \min_{\pi \in \bar{\Omega} \bigcap
 \mathcal E} \sum_{t=1}^T c_t(\tilde x_t^{\pi}, \tilde u_t^{\pi})
\end{align*}
To quantify the difference term, we need to invoke Lemma \ref{lem:fun-loose} by comparing two optimal solutions in $\tilde{\Omega}$ and $\bar{\Omega}.$  
By Lemma \ref{lem:fun-loose}, we show 
\begin{align*}
    \min_{\pi \in \tilde{\Omega} \bigcap
 \mathcal E}\sum_{t=1}^T c_t(\tilde x_t^{\pi}, \tilde u_t^{\pi}) - \min_{\pi \in \bar{\Omega} \bigcap
 \mathcal E} \sum_{t=1}^T c_t(\tilde x_t^{\bar \pi}, \tilde u_t^{\bar \pi}) \leq TLH \bar{\epsilon}.
\end{align*}

\item The second term can be bounded as follows
\begin{align*}
\sum_{t=1}^T c_t(\tilde x_t^{\bar \pi}, \tilde u_t^{\bar \pi}) - c_t(x_t^{\bar \pi}, u_t^{\bar \pi})
\leq& C_0 (\|\tilde x_{t}^{\bar{\pi}} - x_{t}^{\bar{\pi}}\| + \|\tilde u_{t}^{\bar{\pi}} - u_{t}^{\bar{\pi}}\|) \\
\leq& 2TC_0 \kappa(1-\rho)^H \frac{W\kappa^2(\kappa^2 + Ha\kappa_B\kappa^2)}{\rho(1-\kappa^2(1-\rho)^{H+1})}.
\end{align*}

\item The third term is also bounded by $O(\log T)$ according to Lemma \ref{lem:xu diff}.
\begin{align*}
\sum_{t=1}^T c_t(x^{\bar{\pi}}_t, u^{\bar{\pi}}_t) - c_t(x^{K^*}_t, u^{K^*}_t) 
\leq& TC_0 (\|x_{t}^{\bar{\pi}} - x_{t}^{K^*}\| + \|u_{t}^{\bar{\pi}} - u_{t}^{K^*}\|)\\
\leq& \frac{2THC_0W (\kappa+1)^2 \kappa_B a(1-\rho)^{H+1}}{\rho}.
\end{align*}
\end{itemize}
We combine these three terms and complete the proof as follows: 
\begin{align*}
    &\min_{\pi \in \tilde{\Omega} \bigcap
 \mathcal E}\sum_{t=1}^T c_t(\tilde x_t^{\pi}, \tilde u_t^{\pi}) - \sum_{t=1}^T c_t(x^{K^*}_t, u^{K^*}_t)  \\
    \leq& TLH \bar{\epsilon} + 2TC_0 \kappa(1-\rho)^H \frac{W\kappa^2(\kappa^2 + Ha\kappa_B\kappa^2)}{\rho(1-\kappa^2(1-\rho)^{H+1})} + \frac{2THC_0W (\kappa+1)^2 \kappa_B a(1-\rho)^{H+1}}{\rho} \\
    \leq& TLH \bar{\epsilon} + \frac{2TC_0W(1+\kappa)^5(1 + 2Ha\kappa_B)}{\rho(T^3-\kappa^2)} \\
    \leq& LH + 1
\end{align*}
where the last inequality holds for a large $T$ by letting $H = 3\log T/\rho$ according to \eqref{eq: large T}. 

Lastly, we present the following key lemma in proving Lemma \ref{lem:dac-rep}, which is related to the differences of $x^{\bar{\pi}}_t$ and $x^{K^*}_t$ and $u^{\bar{\pi}}_t$ and $u^{K^*}_t.$ 
\begin{lemma}
Under a disturbance-action policy $\pi(K, \{\mathbf M_t\}) \in \mathcal E$ that satisfies $$\mathbf M_t = \mathbf{\bar M}, \forall t \in [T] ~~\text{and}~~ \mathbf{\bar M}^{[i]} = (K-K^*) \tilde{A}_{K^*}^{i-1},$$ we have
\begin{align}
\|x_{t}^{\bar\pi} - x_{t}^{K^*}\| 
\leq& \frac{HW \kappa^2 \kappa_B a(1-\rho)^{H+1}}{\rho}, \nonumber\\
\|u_{t}^{\bar\pi} - u_{t}^{K^*}\| 
\leq& \frac{2HW \kappa^3 \kappa_B a(1-\rho)^{H+1}}{\rho}. \nonumber
\end{align}\label{lem:xu diff}
\end{lemma}
\begin{proof}
\normalfont
Under a disturbance-action policy, we have
\begin{align*}
    x_{t}^{\bar{\pi}} = \sum_{i=1}^t  \Psi_{t,i}^{\bar{\pi}} w_{t-i},
\end{align*}
where $$\Psi_{t,i}^{\bar{\pi}}:= \Psi_{t,i}^{\bar{\pi}}(\mathbf{\bar M}, \cdots, \mathbf{\bar M}) = \tilde{A}_K^{i-1} \mathbb I(i\leq H) + \sum_{j=1}^H \tilde{A}_K^{j-1} B \mathbf{\bar M}^{[i-j]} \mathbb I_{i-j \in [1,H]}.$$
Therefore, we compute
\begin{align*}
    x_{t}^{\bar{\pi}} - x_{t}^{K^*} = \sum_{i=1}^t \left[(\tilde{A}_K^{i-1} - \tilde{A}_{K^*}^{i-1})\mathbb I(i\leq H) w_{t-i} + \sum_{j=1}^H \tilde{A}_K^{j-1} B \mathbf{\bar M}^{[i-j]}\mathbb I(1\leq i-j\leq H) w_{t-i}\right].
\end{align*}
According to the definition of $\mathbf{\bar M}^{[i-j]} = (K-K^*) \tilde{A}_{K^*}^{i-j-1},$ we have for any $i\leq H$ such that 
\begin{align*}
    &\sum_{j=1}^H \tilde{A}_K^{j-1} B \mathbf{\bar M}^{[i-j]}\mathbb I(1\leq i-j\leq H) \\
    =& \sum_{j=1}^H \tilde{A}_K^{j-1} B (K-K^*) \tilde{A}_{K^*}^{i-j-1}\mathbb I(1\leq i-j\leq H) \\
    =& \sum_{j=1}^H \tilde{A}_K^{j-1} (\tilde{A}_{K^*}-\tilde{A}_{K}) \tilde{A}_{K^*}^{i-j-1}\mathbb I(1\leq i-j\leq H) \\
    =& \sum_{j=1}^{i-1} \tilde{A}_K^{j-1} (\tilde{A}_{K^*}-\tilde{A}_{K})  \tilde{A}_{K^*}^{i-j-1} \\
    =& \tilde{A}_{K^*}^{i-1} - \tilde{A}_K^{i-1}
\end{align*}
which implies
\begin{align*}
    \|x_{t}^{\bar\pi} - x_{t}^{K^*}\| =& \| \sum_{i=H+1}^t \sum_{j=1}^H \tilde{A}_K^{j-1} B \mathbf{\bar M}^{[i-j]}\mathbb I(1\leq i-j\leq H) w_{t-i} \| \\
    \leq& \sum_{i=H+1}^t HW \kappa^2 \kappa_B a (1-\rho)^i \\
    \leq& \frac{HW \kappa^2 \kappa_B a(1-\rho)^{H+1}}{\rho}.
\end{align*}
Recall the definition of $u_t^{\bar \pi} = -K x_t^{\bar \pi} + \sum_{i=1}^H \mathbf{\bar M}^{[i]} w_{t-i}$ and we compute 
\begin{align*}
    \|u_{t}^{\bar\pi} - u_{t}^{K^*}\| =& \|K^* x^{K^*}_t -K x^{\bar\pi}_t + \sum_{i=1}^H \mathbf{\bar M}^{[i]} w_{t-i}\| \\
    =& \|(K^* - K) x^{K^*}_t + K (x^{K^*}_t - x^{\bar\pi}_t) + \sum_{i=1}^H (K-K^*) \tilde{A}_{K^*}^{i-1} w_{t-i}\| \\
    =& \|K (x^{K^*}_t - x^{\bar\pi}_t)\| + \|\sum_{i=H+1}^t (K-K^*) \tilde{A}_{K^*}^{i-1} w_{t-i}\| \\
\leq& \frac{2HW \kappa^3 \kappa_B a(1-\rho)^{H+1}}{\rho}.
\end{align*}
\end{proof}

\newpage
\section{Addition Details of the Experiments}
\label{app-exp}
\subsection{QVT Control}
The system equation is $\ddot{x}_t = \frac{u_t}{m} - g - \frac{I^a {\dot x}_t}{m} + w_t$, where $x_t$ is the altitude of the quadrotor, $u_t$ is the motor thrust, $m$ is the mass of the quadrotor, $g$ is the gravitational acceleration, and $I^a$ is the drag coefficient of the air resistance. 
Let $m=1 \mathrm{kg}$, $g=9.8 \mathrm{m/s}^2$, and $I^a=0.25\mathrm{kg/s}$. The system is discretized with $\Delta_t=1s$. 
We impose time-varying constraints, $z_t \ge 0.3 + 0.3\sin(t/10)$, to emulate the complicated time-varying obstacles on the ground. 
The static affine constraints are $z_t\le 1.7$ and $0\le v_t \le 12$. We consider a time-varying quadratic cost function $0.1(z_t-0.7)^2 + 0.1\dot{z}_t^2 + \chi_t (v_t-9.8)^2$, where $\chi_t\sim U(0.1, 0.2).$ 
We simulate two different wind conditions $w_t\sim U(-5.5,-4.5)$ (winds blow down) and $w_t\sim U(4.5, 5.5)$ (winds blow up), respectively.
The parameters of COCA-Soft, COCA-Hard, and COCA-Best2Worlds for QVT control are in Table~\ref{tab:exp-para}.
\begin{table}[ht]
    \centering
    \begin{tabular}{c|c|c|c|c|c|c}
         Algorithm & $V$ & $\eta$ & $\alpha$ & $\epsilon$ & $\gamma$ & $H$\\
         \hline
         COCA-Soft & $\sqrt{T}$ & $T^{3/2}$ & $T$ & $1/\sqrt{T}$ & N/A & 7\\
         \hline
         COCA-Hard & $1$ & $T^{3/2}$ & $T^{2/3}$ & N/A & $T^{2/3}$ & 7\\
         \hline
         COCA-Best2Worlds & $\sqrt{t}$ & $t^{3/2}$ & $0.5t^{3/2}$ & $1/\sqrt{T}$ & $t^{3/2}$ & 7
    \end{tabular}
    \caption{Parameters of COCA-Soft, COCA-Hard, and COCA-Best2Worlds in QVT Control.}
    \label{tab:exp-para}
\end{table}

Figure~\ref{fig:exp-results-app-1} and Figure~\ref{fig:exp-results-app-2} shows the experiment results for QVT control with winds blowing down $w_t\sim U(-5.5,-4.5)$ and with winds blowing up $w_t\sim U(4.5, 5.5),$ respectively. 
These two figures show that COCA-Soft, COCA-Hard, and COCA-Best2Worlds achieve much better performance than the stable controller. 
Specifically, our algorithms have much smaller cumulative costs, near-zero static constraint violations, negative cumulative soft violations that decrease by time, and cumulative hard violations that remain unchanged small constant shortly after the initial stages. These results verify our algorithms are very adaptive to the adversarial environment and achieve minimal cumulative costs and best constraints satisfaction. Moreover, we observe COCA algorithms have almost identical performance on cumulative costs and soft violations, but COCA-Hard and COCA-Best2Worlds has a smaller hard violation than COCA-Soft. It justifies the penalty-based design is efficient in tackling the hard violation.

\begin{figure}[h]
    \centering
    \tabcolsep=0.02\linewidth
    \divide\tabcolsep by 8
    \begin{tabular}{ccccc}
        \includegraphics[width=0.194\linewidth]{Figures/setting1_negative_v2/cost.pdf} &
        \includegraphics[width=0.192\linewidth]{Figures/setting1_negative_v2/constr_viol_t_var_soft.pdf} &
        \includegraphics[width=0.19\linewidth]{Figures/setting1_negative_v2/constr_viol_t_var_hard.pdf} &
        \includegraphics[width=0.193\linewidth]{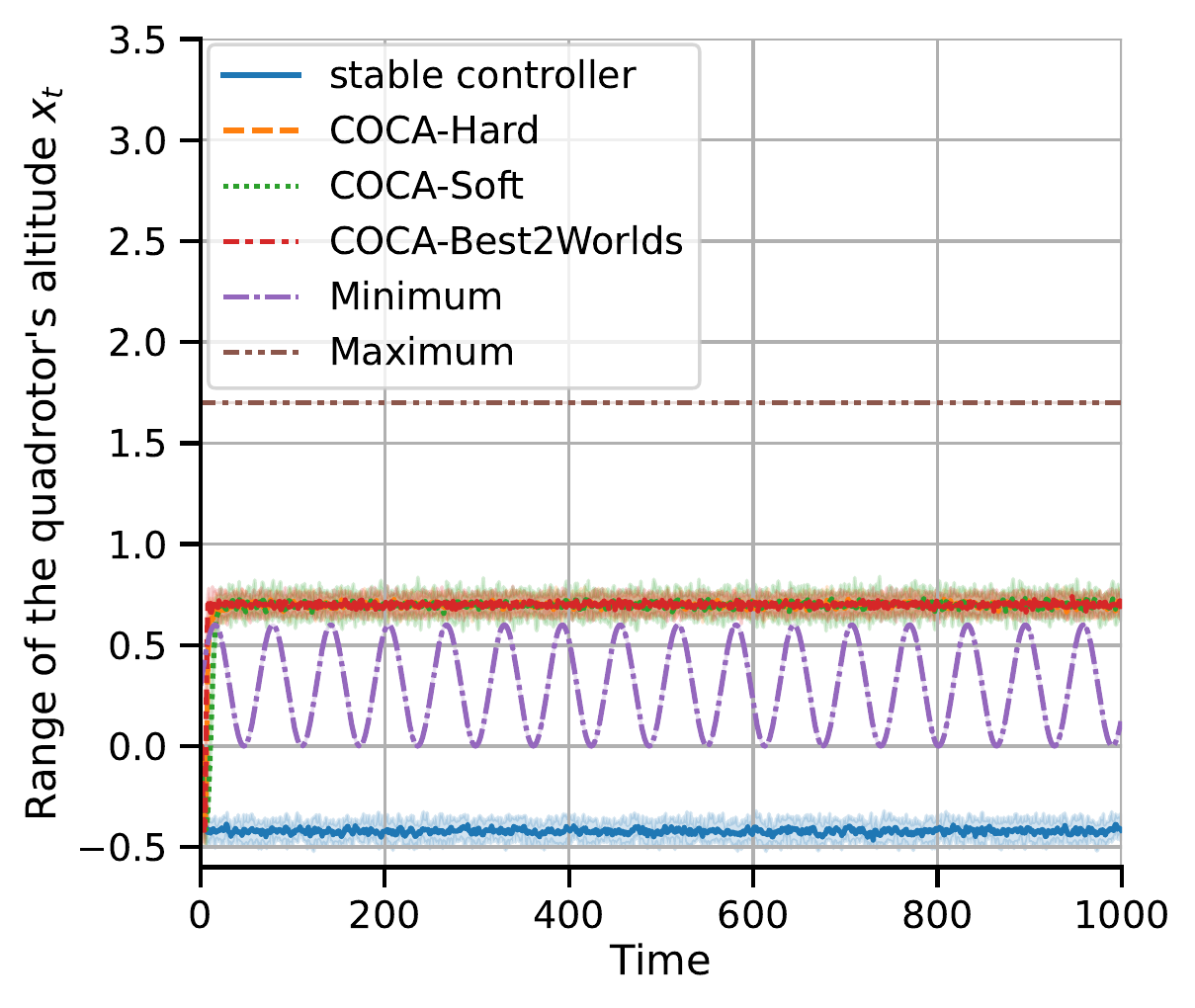} &
        \includegraphics[width=0.195\linewidth]{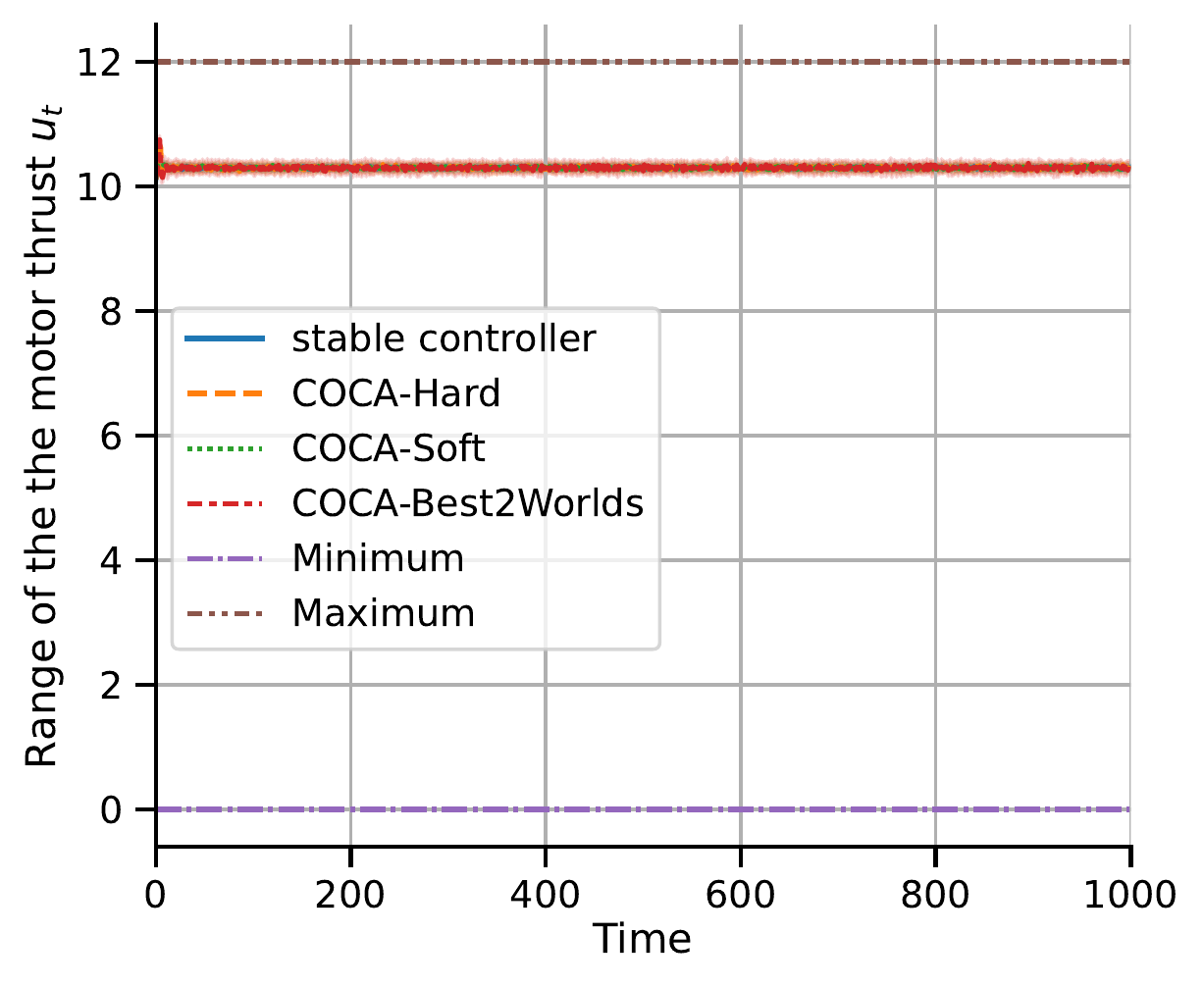}\\
        \small (a) Cumulative costs &
        \small (b) Soft violations &
        \small (c) Hard violations &
        \small (d) Range of $x_t$ &
        \small (e) Range of $u_t$
    \end{tabular}
    \caption{Experiment results for QVF control with winds blowing down $w_t\sim U(-5.5,-4.5).$
    The lines are plotted by averaging over 10 independent runs. The shaded areas in Figures (a)-(c) are 95\% confidence intervals and in Figure (d)-(e) are the full ranges of the data.
    }
    \label{fig:exp-results-app-1}
\end{figure}

\begin{figure}[h]
    \centering
    \tabcolsep=0.02\linewidth
    \divide\tabcolsep by 8
    \begin{tabular}{ccccc}
        \includegraphics[width=0.194\linewidth]{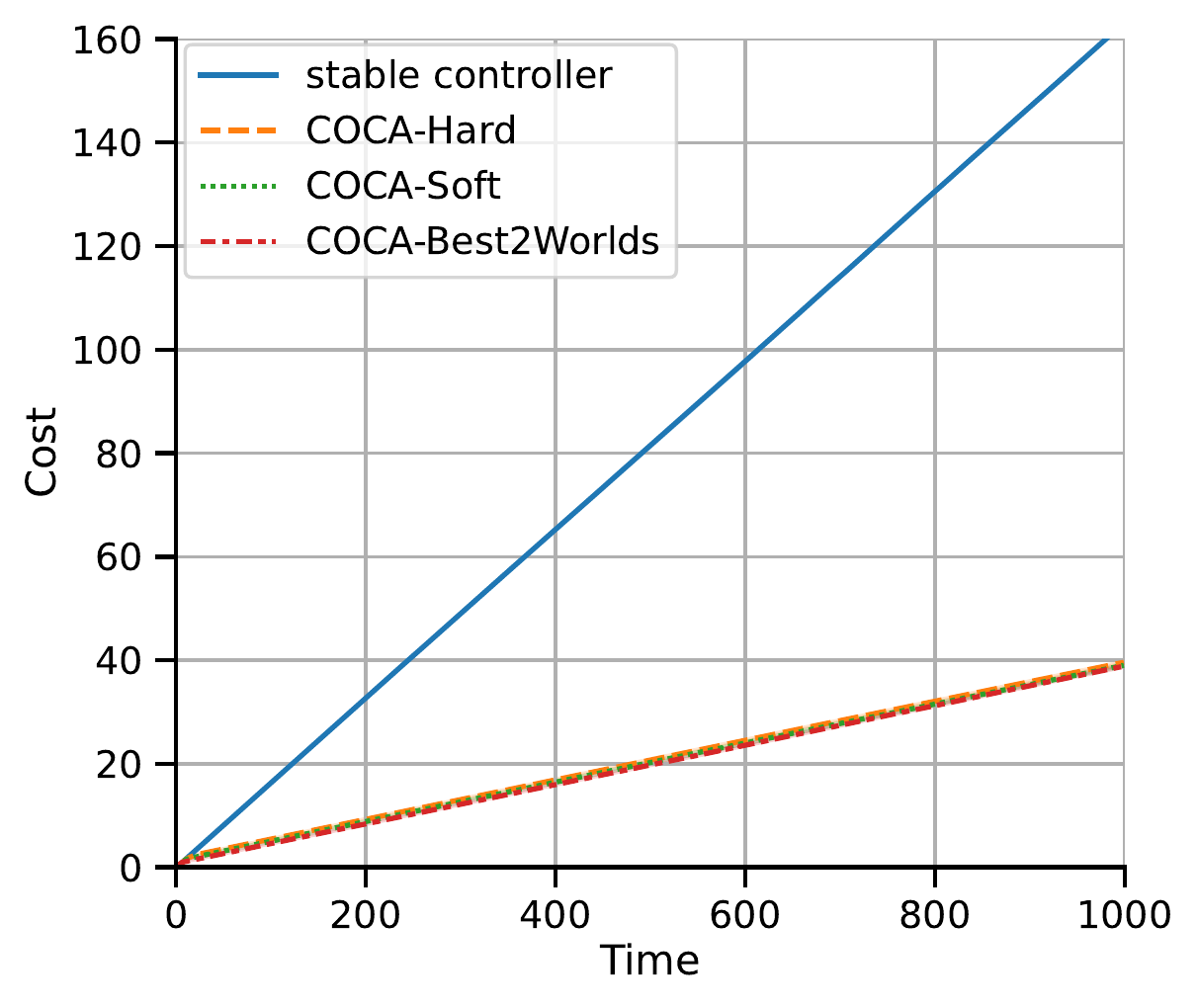} &
        \includegraphics[width=0.188\linewidth]{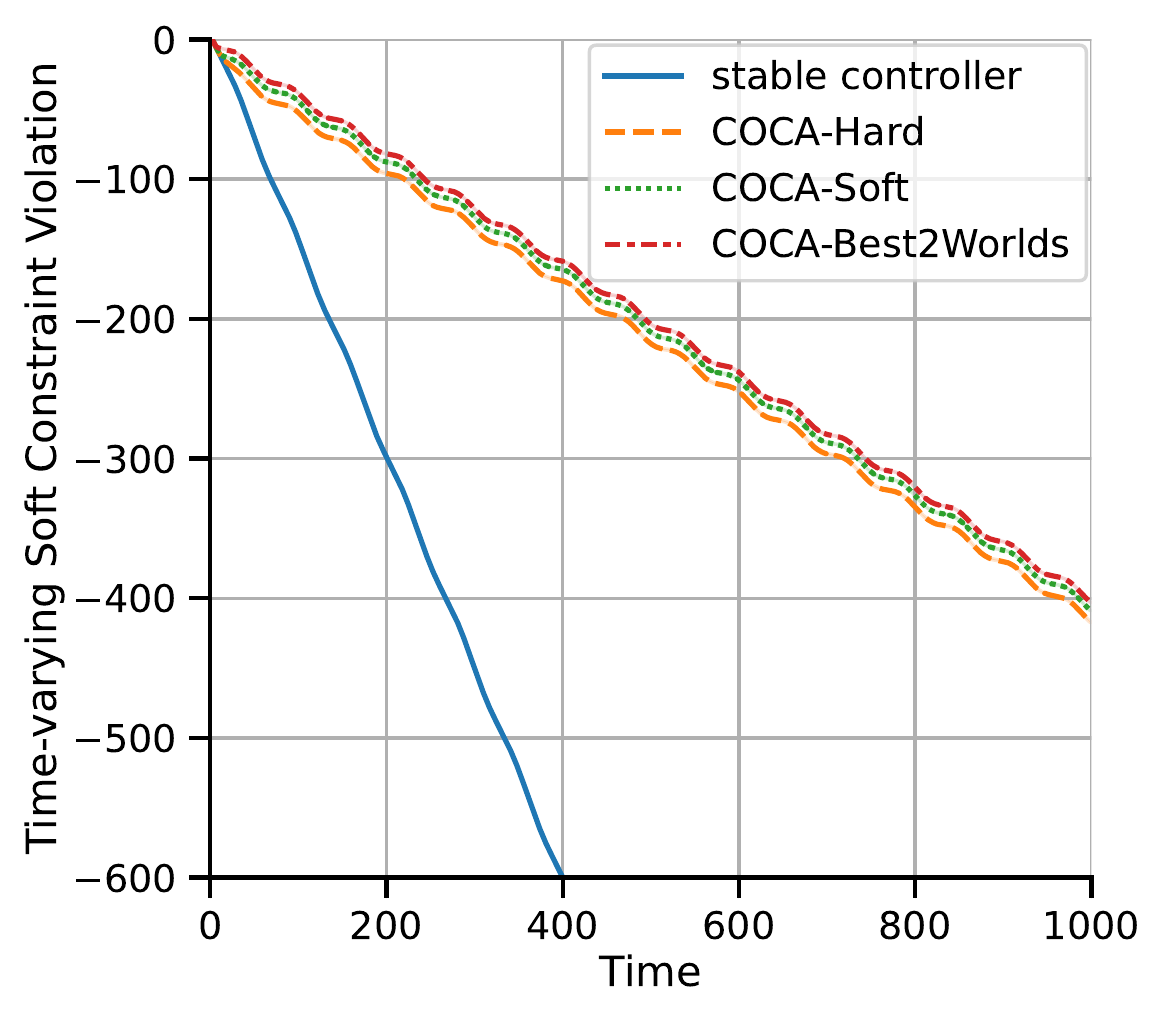} &
        \includegraphics[width=0.19\linewidth]{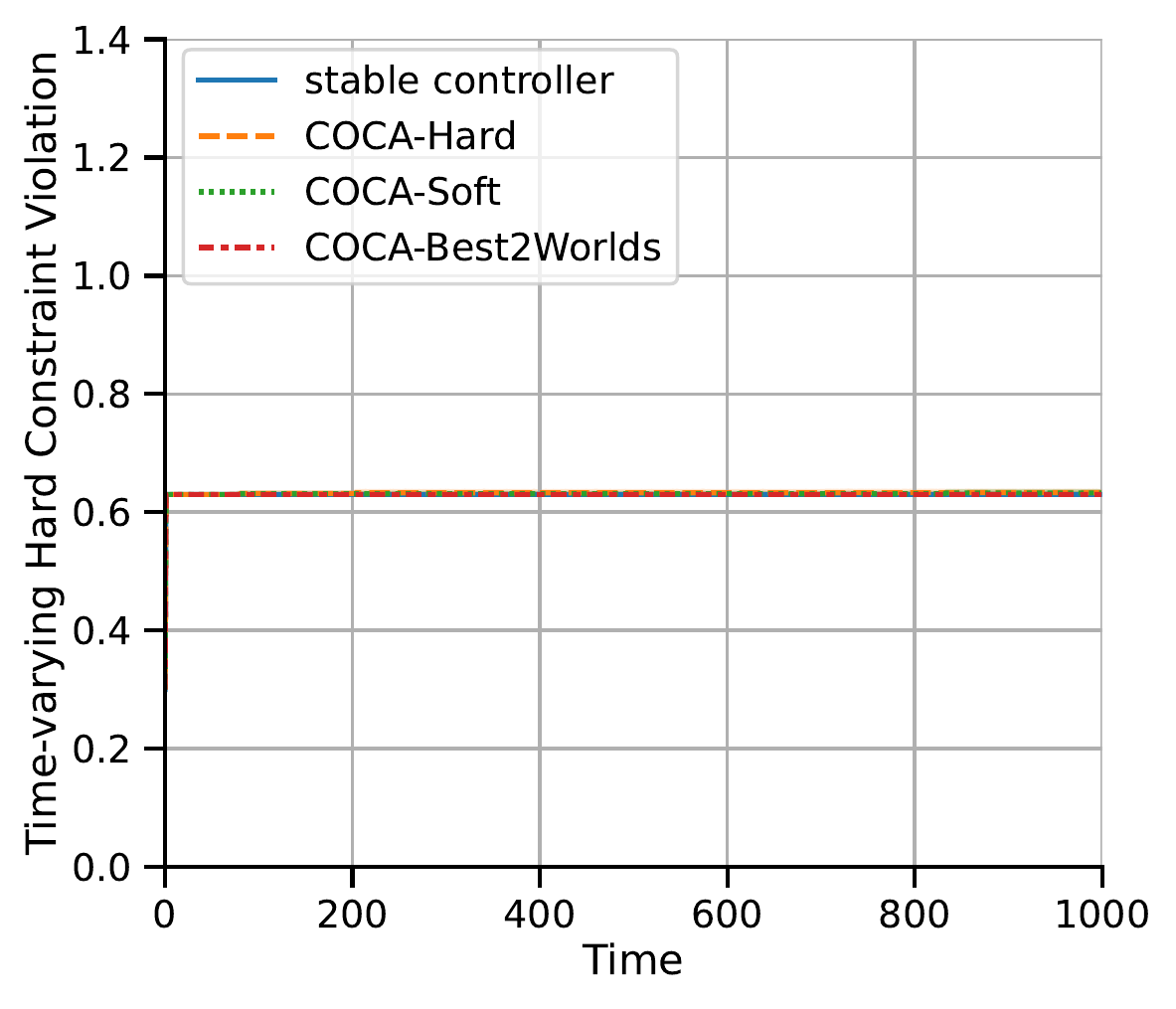} &
        \includegraphics[width=0.195\linewidth]{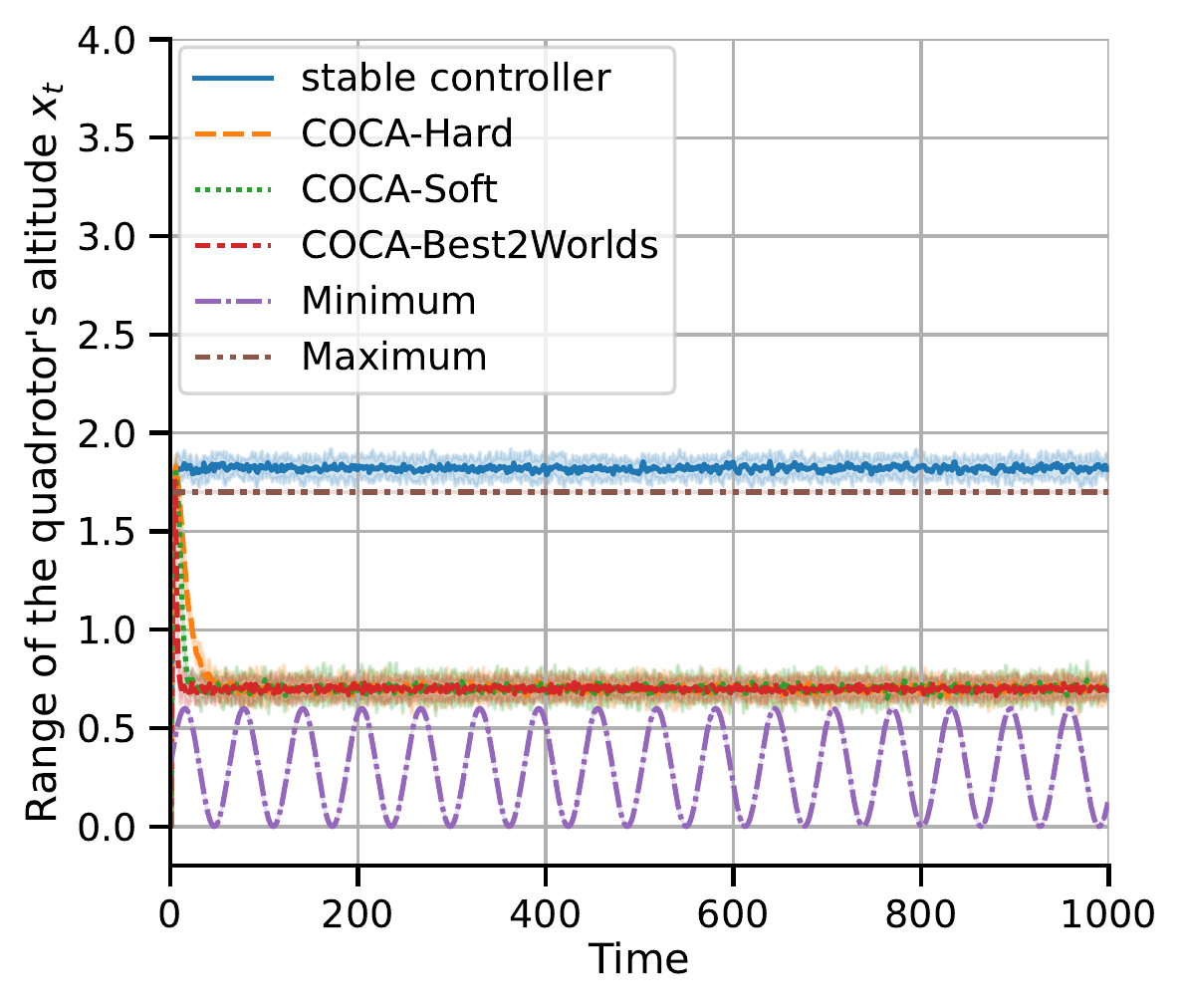} &
        \includegraphics[width=0.195\linewidth]{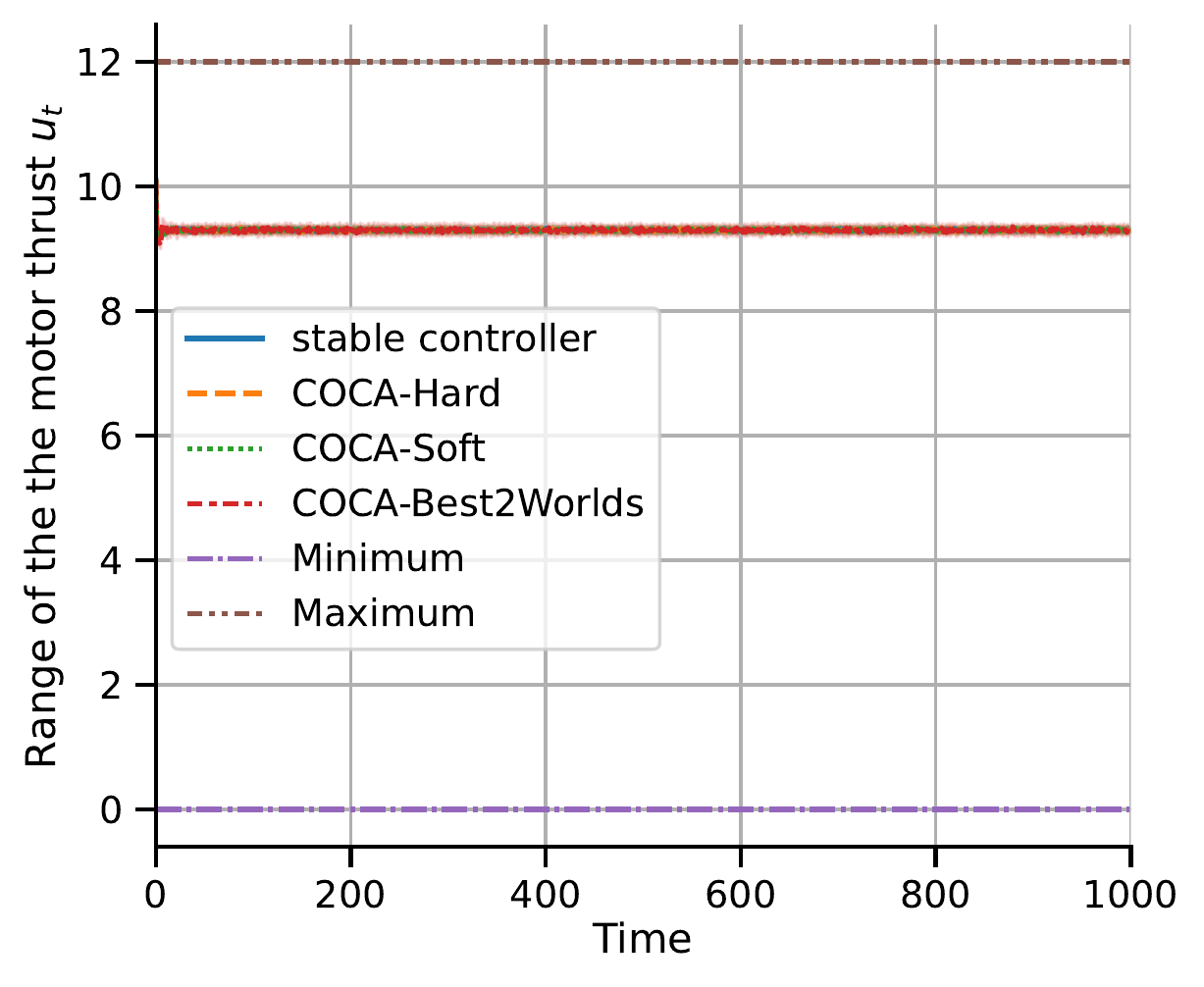}\\
        \small (a) Cumulative costs &
        \small (b) Soft violations &
        \small (c) Hard violations &
        \small (d) Range of $x_t$ &
        \small (e) Range of $u_t$
    \end{tabular}
    \caption{Experiment results for QVF control with winds blowing up $w_t\sim U(4.5, 5.5).$
    The lines are plotted by averaging over 10 independent runs. The shaded areas in Figures (a)-(c) are 95\% confidence intervals and in Figure (d)-(e) are the full ranges of the data.
    }
    \label{fig:exp-results-app-2}
\end{figure}

\newpage
\subsection{HVAC Control}
The system equation is $\dot{x}_t = \frac{\theta^o - x_t}{v\zeta} - \frac{u_t}{v} + \frac{w_t+\iota}{v},$ where $x_t$ is the room temperature, $u_t$ is the airflow rate of the HVAC system as the control input, $\theta^o$ is the outdoor temperature, $w_t$ is the random disturbance, $\iota$ represents the impact of the external heat sources, $v$ and $\zeta$ denotes the environmental parameters. Let $v=100$, $\zeta=6$, $\theta^o=30 ^\circ C$, and $\iota=1.5$. Let $w_t\sim U(-1.1,1.3)$ and we discretize the system with $\Delta_t = 60s.$ 
Similar to \cite{LiDasLi_21}, the state and input constraints are $22.5 \le x_t \le 25.5$ and $0.5\le u_t\le 4.5,$ respectively. We specify the time-varying cost functions $c_t=2(x_t-24)^2 + \chi_t(u_t-2.5)^2$ with $\chi_t \sim U(0.1, 4.0).$ The parameters of COCA for HVAC control are in Table~\ref{tab:exp-para-HVAC}.
\begin{table}[ht]
    \centering
    \begin{tabular}{c|c|c|c|c|c|c}
         Algorithm & $V$ & $\eta$ & $\alpha$ & $\epsilon$ & $\gamma$ & $H$\\
         \hline
         COCA & $0.1$ & $T^{3/2}$ & $T^{2/3}$ & N/A & $T^{2/3}$ & 7
    \end{tabular}
    \caption{Parameters of COCA in HVAC Control.}
    \label{tab:exp-para-HVAC}
\end{table}

We compare COCA with OGD-BZ algorithm (COCA-Soft, COCA-Hard, and COCA-Best2Worlds are exactly identical, called COCA, because there only exist static state and input constraints). Figure~\ref{fig:exp-results-app-3} (a)-(c) show the cumulative costs, the ranges of the room temperature $x_t$ and control input $u_t.$ We observe that COCA has a significantly better cumulative cost than OGD-BZ algorithm with a near-zero constraint violation. The results verify our approach is effective in designing less-conservative COCA algorithms.

\begin{figure}[h]
    \centering
    \tabcolsep=0.02\linewidth
    \divide\tabcolsep by 8
    \begin{tabular}{ccc}
        \includegraphics[width=0.32\linewidth]{Figures/setting2_biased/cost.pdf} &
        \includegraphics[width=0.32\linewidth]{Figures/setting2_biased/x_range.pdf} &
        \includegraphics[width=0.32\linewidth]{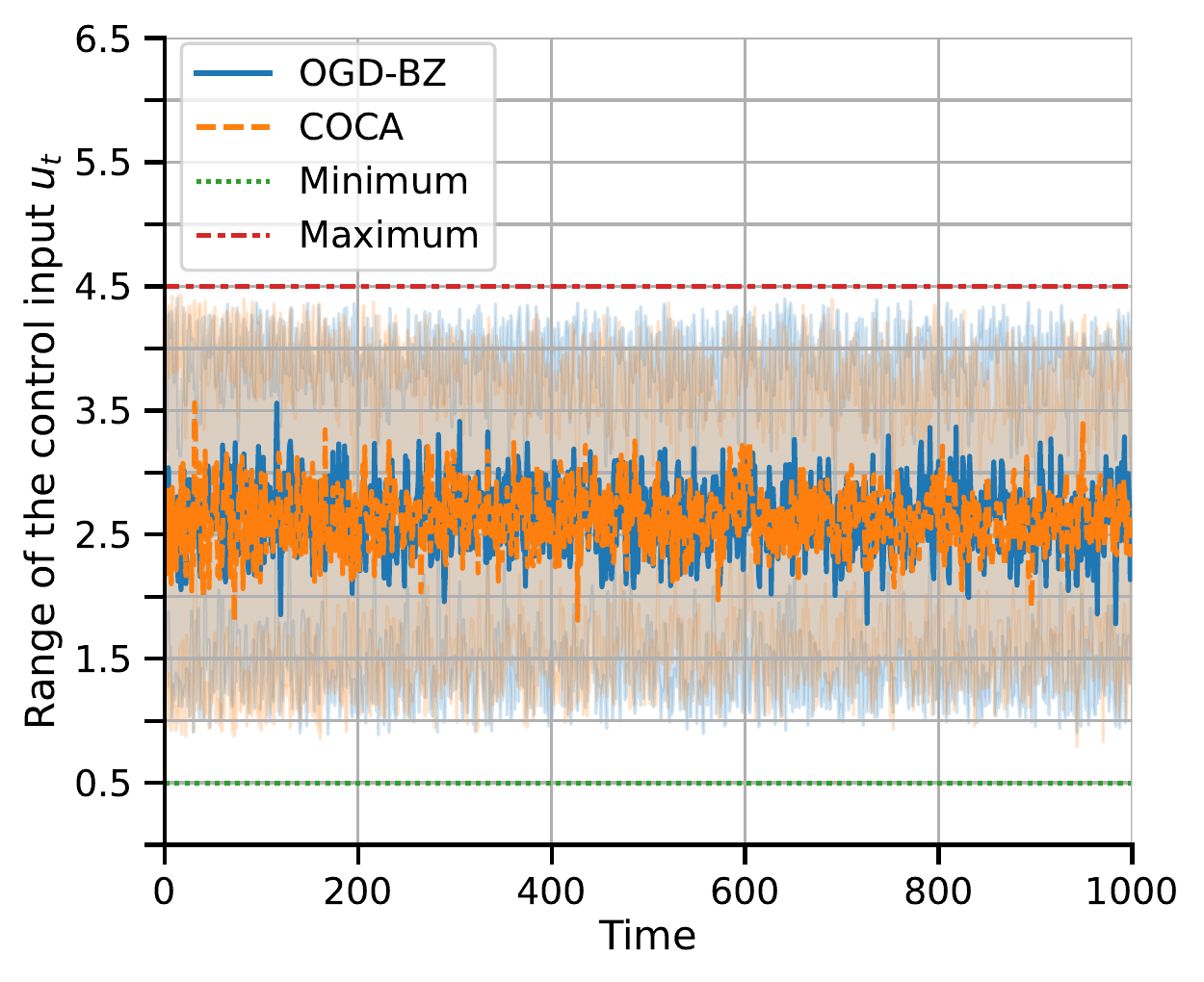}\\
        \small (a) Cumulative costs &
        \small (b) Range of $x_t$ &
        \small (c) Range of $u_t$
    \end{tabular}
    \caption{Experiment results for HVAC Control.
    The lines are plotted by averaging over 10 independent runs. The shaded areas in Figure (a) are 95\% confidence intervals and in Figures (b)-(c) are the full ranges of the data.
    }
    \label{fig:exp-results-app-3}
\end{figure}
\end{document}